\documentclass[11pt]{article}

\usepackage{amsthm}
\usepackage{amsmath}
\usepackage{amssymb}
\usepackage{graphicx}
\usepackage{mathtools}
\usepackage{enumerate}
\usepackage{enumitem}
\usepackage{footnote}
\usepackage{float}
\usepackage{xspace}
\usepackage{multirow}
\usepackage{nicefrac}
\usepackage{wrapfig}
\usepackage{framed}
\usepackage{url}
\usepackage[colorlinks=true, linkcolor=blue, citecolor=blue]{hyperref}
\usepackage{blkarray}
\PassOptionsToPackage{round}{natbib}

\usepackage{tikz}
\usetikzlibrary{arrows,chains,matrix,positioning,scopes}

\let\hat\widehat
\let\tilde\widetilde

\newtheorem{lemma}{Lemma}
\newtheorem{theorem}{Theorem}
\newtheorem{cor}{Corollary}

\newtheorem{assumption}{Assumption}

\DeclareMathOperator*{\minimize}{minimize}
\DeclareMathOperator*{\maximize}{maximize}
\DeclareMathOperator*{\subject}{subject~to}

\newcommand{\calA}{{\mathcal{A}}}

\newcommand{\calS}{{\mathcal{S}}}
\newcommand{\calI}{{\mathcal{I}}}

\newcommand{\calF}{{\mathcal{F}}}
\newcommand{\calV}{{\mathcal{V}}}

\newcommand{\one}{\boldsymbol{1}}

\DeclareMathOperator*{\argmin}{argmin}
\DeclareMathOperator*{\argmax}{argmax}

\newcommand{\norm}[1]{\left\|{#1}\right\|}
\newcommand{\inner}[2]{\left\langle #1,#2 \right\rangle}

\newcommand{\rbr}[1]{\left(#1\right)}
\newcommand{\sbr}[1]{\left[#1\right]}
\newcommand{\cbr}[1]{\left\{#1\right\}}

\newcommand{\abr}[1]{\left|#1\right|}

\DeclareFontFamily{OMX}{MnSymbolE}{}
\DeclareFontShape{OMX}{MnSymbolE}{m}{n}{
    <-6>  MnSymbolE5
   <6-7>  MnSymbolE6
   <7-8>  MnSymbolE7
   <8-9>  MnSymbolE8
   <9-10> MnSymbolE9
  <10-12> MnSymbolE10
  <12->   MnSymbolE12}{}
\DeclareSymbolFont{mnlargesymbols}{OMX}{MnSymbolE}{m}{n}
\SetSymbolFont{mnlargesymbols}{bold}{OMX}{MnSymbolE}{b}{n}
\DeclareMathDelimiter{\llangle}{\mathopen}{mnlargesymbols}{'164}{mnlargesymbols}{'164}
\DeclareMathDelimiter{\rrangle}{\mathclose}{mnlargesymbols}{'171}{mnlargesymbols}{'171}

\usepackage{prettyref}

\newcommand{\savehyperref}[2]{\texorpdfstring{\hyperref[#1]{#2}}{#2}}
\newrefformat{eq}{\savehyperref{#1}{\textup{(\ref*{#1})}}}
\newrefformat{eqn}{\savehyperref{#1}{Equation~\eqref{#1}}}
\newrefformat{lem}{\savehyperref{#1}{Lemma~\ref*{#1}}}
\newrefformat{def}{\savehyperref{#1}{Definition~\ref*{#1}}}
\newrefformat{line}{\savehyperref{#1}{line~\ref*{#1}}}
\newrefformat{thm}{\savehyperref{#1}{Theorem~\ref*{#1}}}
\newrefformat{corr}{\savehyperref{#1}{Corollary~\ref*{#1}}}
\newrefformat{cor}{\savehyperref{#1}{Corollary~\ref*{#1}}}
\newrefformat{sec}{\savehyperref{#1}{Section~\ref*{#1}}}
\newrefformat{app}{\savehyperref{#1}{Appendix~\ref*{#1}}}
\newrefformat{assum}{\savehyperref{#1}{Assumption~\ref*{#1}}}
\newrefformat{ex}{\savehyperref{#1}{Example~\ref*{#1}}}
\newrefformat{fig}{\savehyperref{#1}{Figure~\ref*{#1}}}
\newrefformat{alg}{\savehyperref{#1}{Algorithm~\ref*{#1}}}
\newrefformat{rem}{\savehyperref{#1}{Remark~\ref*{#1}}}
\newrefformat{conj}{\savehyperref{#1}{Conjecture~\ref*{#1}}}
\newrefformat{prop}{\savehyperref{#1}{Proposition~\ref*{#1}}}
\newrefformat{proto}{\savehyperref{#1}{Protocol~\ref*{#1}}}
\newrefformat{prob}{\savehyperref{#1}{Problem~\ref*{#1}}}
\newrefformat{claim}{\savehyperref{#1}{Claim~\ref*{#1}}}
\newrefformat{que}{\savehyperref{#1}{Question~\ref*{#1}}}
\usepackage{microtype}
\usepackage{graphicx}
\usepackage{subfigure}
\usepackage{booktabs} % for professional tables
\usepackage[colorinlistoftodos]{todonotes}
\usepackage{multicol}
\usepackage{float}
\newcommand{\DefinedAs}[0]{\mathrel{\mathop:}=}
\usepackage{algorithm}
\usepackage[noend]{algpseudocode}

\newtheorem{manualtheoreminner}{Theorem}
\newenvironment{manualtheorem}[1]{%
	\renewcommand\themanualtheoreminner{#1}%
	\manualtheoreminner
}{\endmanualtheoreminner}

\newtheorem{manuallemmainner}{Lemma}
\newenvironment{manuallemma}[1]{%
	\renewcommand\themanuallemmainner{#1}%
	\manuallemmainner
}{\endmanuallemmainner}

\title{\huge Provably Efficient Safe Exploration \\via Primal-Dual Policy Optimization}
\author{Dongsheng~Ding,\quad Xiaohan~Wei,\quad Zhuoran~Yang,\quad Zhaoran~Wang,\quad Mihailo~R.~Jovanovi\'{c}
	\thanks{D.\ Ding, and M.\ R.\ Jovanovi\'{c} are with the Ming Hsieh Department of Electrical and Computer Engineering, University of Southern California, Los Angeles, CA 90089.
		X.\ Wei is with Facebook Inc., Menlo Park, CA 94025. 
		Z.\ Yang is with the Department of Operations Research and Financial Engineering, Princeton University, Princeton, NJ 08544.
		Z.\ Wang is with the Department of Industrial Engineering and Management Sciences, Northwestern University, Evanston, IL 60208.
		{E-mails: dongshed@usc.edu, xiaohanw@usc.edu, zy6@princeton.edu, zhaoran.wang@northwestern.edu, mihailo@usc.edu.}
}
}

\begin{document}
\maketitle
\begin{abstract}%   <- trailing '%' for backward compatibility of .sty file
  	We study the Safe Reinforcement Learning (SRL) problem using the Constrained Markov Decision Process (CMDP) formulation in which an agent aims to maximize the expected total reward subject to a safety constraint on the expected total value of a utility function. We focus on an episodic setting with the function approximation where the Markov transition kernels have a linear structure but do not impose any additional assumptions on the sampling model. Designing SRL algorithms with provable computational and statistical efficiency is particularly challenging under this setting because of the need to incorporate both the safety constraint and the function approximation into the fundamental exploitation/exploration tradeoff. To this end, we present an \underline{O}ptimistic \underline{P}rimal-\underline{D}ual Proximal Policy \underline{OP}timization (OPDOP) algorithm where the value function is estimated by combining the least-squares policy evaluation and an additional bonus term for safe exploration. We prove that the proposed algorithm achieves an $\tilde{O}(d H^{2.5}\sqrt{T})$ regret and an $\tilde{O}(d H^{2.5}\sqrt{T})$ constraint violation, where $d$ is the dimension of the feature mapping, $H$ is the horizon of each episode, and $T$ is the total number of steps. These bounds hold when the reward/utility functions are fixed but the feedback after each episode is bandit. Our bounds depend on the capacity of the state-action space only through the dimension of the feature mapping and thus our results hold even when the number of states goes to infinity. To the best of our knowledge, we provide the first provably efficient online policy optimization algorithm for CMDP with safe exploration in the function approximation setting.
\end{abstract}

%\begin{keywords}
%	Safe Reinforcement Learning, Constrained Markov Decision Process, Primal-Dual Method, Proximal Policy Optimization, Upper Confidence Bound
%\end{keywords}

\section{Introduction}
\label{intro}

Reinforcement Learning (RL) studies how an agent learns to maximize its expected total reward by interacting with an unknown environment over time~\cite{sutton2018reinforcement}. Safe RL (SRL) involves extra restrictions arising from real-world problems~\cite{garcia2015comprehensive,amodei2016concrete,dulac2019challenges}. Examples include collision-avoidance in self-driving cars~\cite{fisac2018general}, switching cost limitations in medical applications~\cite{bai2019provably}, and legal and business restrictions in financial management~\cite{abe2010optimizing}. A standard environment model for SRL is the Constrained Markov Decision Process (CMDP)~\cite{altman1999constrained} that extends the classical MDP by adding an extra safety-related utility, and translates the safety demand into a constraint on the expected total utility~\cite{achiam2017constrained}. The presence of safety constraints makes the celebrated exploration-exploitation trade-off more challenging. 
%\ddmargin{is it a correct citation for exploration-exploitation trade-off? Maybe cite the bandit surveys by Bubeck and Cesa-Bianchi or Lattimore and Szepesvari instead?}

Many existing SRL algorithms for CMDPs are policy-based, and primal-dual type constrained optimization methods are commonly used to deal with policy constraints, e.g., Constrained Policy Gradient (CPG)~\cite{uchibe2007constrained}, Lagrangian-based Actor-Critic (AC)~\cite{borkar2005actor,bhatnagar2012online,chow2017risk,tessler2018reward,liang2018accelerated}, Primal-Dual Policy Optimization (PDPO)~\cite{paternain2019constrained,paternain2019safe,stooke2020responsive}, Constrained Policy Optimization (CPO)~\cite{achiam2017constrained,yang2019projection}, and Reward Constrained Policy Optimization (RCPO)~\cite{tessler2018reward}. These SRL algorithms either do not have a theoretical guarantee or can only be shown to converge asymptotically in the batch offline setting. It is imperative to study theoretical guarantees for SRL algorithms regarding computational/statistical efficiency.

%\rmargin{ I think that the term "provably efficient RL" is too broad to refer to a specific line of work.}
In this work, we look at a more challenging problem of finding a sequence of policies in response to streaming samples of reward functions, constraint functions, and transition. We seek to provide theoretical guarantees on the convergence rate of approaching the best-fixed policy in hindsight as well as feasibility region formed by constraints. Unlike scenarios where the safety constraint is known {\em a priori}~\cite{turchetta2016safe,berkenkamp2017safe,dalal2018safe,wachi2018safe,chow2018lyapunov,chow2019lyapunov,wachi2020safe}, we have to explore the unknown environment and adapt the policy to the safety constraint with the minimal violation~\cite{amodei2016concrete,garcia2015comprehensive}. We refer this task as {\em safe exploration}. Recent policy-based SRL algorithms for CMDPs, e.g., CPO~\cite{achiam2017constrained,yang2019projection} and PDPO~\cite{paternain2019constrained,paternain2019safe}, seek a single safe policy via the constrained policy optimization whose sample efficiency guarantees over streaming or time-varying data are largely unknown. In this paper, we aim to answer a theoretical question:

\noindent\textbf{Can we design a provably sample efficient online policy optimization algorithm for CMDP?}

\noindent\textbf{Contribution.} We propose a provably efficient SRL algorithm for CMDP with an unknown transition model in the linear episodic setting -- the \underline{O}ptimistic \underline{P}rimal-\underline{D}ual Proximal Policy \underline{OP}timization (OPDOP) algorithm -- where the value function is estimated by combining the least-squares policy evaluation and an additional bonus term for safe exploration. Theoretically, we prove that the proposed algorithm achieves an $\tilde{O}(d H^{2.5}\sqrt{T})$ regret and the same $\tilde{O}(d H^{2.5}\sqrt{T})$ constraint violation, where $d$ is the dimension of the feature mapping, $H$ is the horizon of each episode, and $T$ is the total number of steps. We establish these bounds in the setting where the reward/utility functions are fixed but the feedback after each episode is bandit. Our bounds depend on the capacity of the state space only through the dimension of the feature mapping and thus hold even when the number of states goes to infinity. To the best of our knowledge, our result is the first provably efficient online policy optimization for CMDP with safe exploration in the function approximation setting.
%\ddmargin{list our contributions: problem solving and non-triviality}
%\rmargin{However, my concern is in terms of the novelty in this paper - since in effect, this paper essentially builds from the "Provably Efficient Exploration in Policy Optimization" paper (Qi Cai et al., 2019) - the authors in this paper propose a similar regret analysis as in the OPPO algorithm, with the only difference being that the optimization is achieved under constraints. This paper essentially builds from the OPPO paper, where similar UCB exploration is performed without constraints, and similar regret analysis is performed.}

%\subsection{Related Work}
%\label{rework}

%\noindent\textbf{Provably Efficient RL}: The exploration-exploitation trade-off~\cite{cesa2006prediction} is crucial for RL algorithms to have good sample efficiency. 

\noindent\textbf{Related Work.} Our work is related to a line of provably efficient RL algorithms based on linear function approximation where the exploration is achieved by adding an Upper Confidence Bound (UCB) bonus~\cite{yang2019sample,yang2019reinforcement,jin2019provably}. 
%In our current result, we exploit the linear structure of the transition model and the value functions to develop an efficient SRL algorithm with safe exploration. 
%It serves as a one-step towards a better understanding of practical SRL algorithms for CMDPs with function approximation~\cite{achiam2017constrained,le2019batch}. 
A closely-related recent work is Proximal Policy Optimization (PPO)~\cite{schulman2017proximal}. As is shown in~\cite{liu2019neural,cai2019provably}, PPO converges to the optimal policy sublinearly and an optimistic variant of PPO is sample efficient with UCB exploration in the linear setting. However, such results only hold for unconstrained RL problems. Our work seeks to extend an optimistic variant of PPO for CMDP with UCB exploration and establish theoretical efficiency guarantees. For the large CMDP with unknown transition models, there is a line of works that relates to the policy optimization under constraints, e.g., CPG~\cite{uchibe2007constrained}, CPO~\cite{achiam2017constrained,yang2019projection}, RCPO~\cite{tessler2018reward}, and IPPO~\cite{liu2019ipo}. However, their theoretical guarantees are still lacking. By contrast, our work is supported by theoretical efficiency guarantees.
%, with limited theoretical performance guarantees.
% in terms of the regret and the constraint violation. 

Recent results on learning CMDPs with unknown transitions and rewards are closely related. As we know, most of them are model-based and only apply for finite state-action spaces. The works~\cite{singh2020learning,efroni2020exploration} independently leverage upper confidence bound (UCB) on fixed reward, utility, and transition probability to propose sample-efficient algorithms for tabular CMDPs;~\cite{singh2020learning} obtains an ${O}(\sqrt{|\mathcal A|T^{1.5}\log T})$ regret and constraint violation via linear program in the average cost setup in time $T$;~\cite{efroni2020exploration} obtains an ${O}(|\mathcal S|\sqrt{|\mathcal S||\mathcal A|H^4T})$ regret and constraint violation in the episodic setting via linear program and primal-dual policy optimization, where $\mathcal S$ is a state space and $\mathcal A$ is an action space. In~\cite{qiu2020upper}, the authors study an adversarial stochastic shortest path problem under constraints and unknown transitions with ${O}(|\mathcal S|\sqrt{|\mathcal A| H^2T})$ regret and constraint violation. The work \cite{bai2020model} extends Q-learning with optimism for finite state-action CMDPs with peak constraints. The work~\cite{brantley2020constrained} proposes UCB-based convex planning for episodic tabular CMDPs in dealing with convex or hard constraints. In contrast, our proposed algorithm can potentially be applied to infinite state space scenarios with sublinear regret and constraint violation bounds only depending on the implicit dimension as opposed to the true dimension of the state space. 

%Some recent model-free RL algorithms for MDPs with full-information or bandit feedback, e.g.,~\cite{neu2012adversarial,rosenberg2019onlinea,jin2019provably,cai2019provably}, are also related. Compared with this line of work, our algorithm is a model-free safe RL algorithm for CMDPs with full-information or bandit feedback. We exploit the least-squares policy evaluation~\cite{bradtke1996linear,boyan1999least,lazaric2010finite,lagoudakis2003least} to deal with two types of feedback and incorporate the UCB bonus for exploration. We establish sublinear bounds on both the regret and the constraint violation, which guarantee the efficiency and safety of the exploration.
\section{Problem Setup}
\label{prelim}

We consider an episodic Markov decision process (MDP) -- MDP$(\calS,\calA,H,\mathbb{P},r)$ -- where $\calS$ is a state space, $\calA$ is an action space, $H$ is a fixed length of each episode, $\mathbb{P} = \{\mathbb{P}_h\}_{1}^H$ is a collection of transition probability measures, and $r=\{r_h\}_{h\,=\,1}^{H}$ is a collection of reward functions. We assume that $\calS$ is a measurable space with possibly infinite number of elements. Moreover, for each step $h\in[H]$, $\mathbb{P}_h(\, \cdot \, \vert\, x,a)$ is a transition kernel over next state if action $a$ is taken for state $x$ and $r_h$: $\calS\times\calA\to[0,1]$ is a reward function. The constrained MDP -- $\text{CMDP}(\calS,\calA,H,\mathbb{P},r,g)$ -- also contains utility functions $g=\{g_h\}_{h\,=\,1}^H$ where $g_h$: $\calS\times\calA\to[0,1]$ is the utility function. We assume that reward/utility are deterministic. Our analysis readily generalizes to the setting where reward and utility are random. 

Let the policy space $\Delta(\calA\,\vert\, \calS, H)$ be $\{\{\pi_h(\,\cdot\,\vert\, \cdot\,)\}_{h\,=\,1}^H$: $\pi_h(\,\cdot\,\vert\, x) \in\Delta(\calA),~ \forall x\in\calS\text{ and }h\in[H]\}$, where $\Delta(\calA)$ denotes a probability simplex over the action space.
%A policy of an agent is a probability distribution $\pi\in \Delta (\calA \, \vert \, \calS,H)$ where $\pi_h (\, \cdot \, \vert \,x_h)$: $\calS \to \calA$ is the action that the agent takes at state $x_h$ and step $h$ in an episode.
Let $\pi^k\in\Delta (\calA \, \vert \, \calS,H)$ be a policy taken by the agent at episode $k$, where $\pi_h^k (\, \cdot \, \vert \,x_h^k)$: $\calS \to \calA$ is the action that the agent takes at state $x_h^k$. For simplicity, we assume the initial state $x_1^k$ to be fixed as $x_1$ in different episodes; it readily generalizes to be random according to some distribution. The agent interacts with the environment in the $k$th episode as follows. At the beginning, the agent determines a policy $\pi^k$. Then, at each step $h\in[H]$, the agent observes the state $x_h^k\in\calS$, determines an action $a_h^k$ following the policy $\pi_h^k( \, \cdot \, \vert\, x_h^k)$, and receives a reward $r_h (x_h^k,a_h^k)$ together with an utility $g_h(x_h^k,a_h^k)$. Meanwhile, the MDP evolves into next state $x_{h+1}^k$ drawing from the probability measure $\mathbb{P}_h (\, \cdot \, \vert\,  x_h^k,a_h^k) $. The episode terminates at state $x_{H+1}^k$; when this happens, no control action is taken and both reward and utility functions are equal to zero. In this paper, we focus a challenging bandit setting where the agent only observes the values of reward/utility functions, $r_h(x_h^k,a_h^k)$, $g_h(x_h^k,a_h^k)$, at visited state-action pair $(x_h^k,a_h^k)$; see examples in~\cite{schell2016data,el2018controlled,paternain2019safe}. We assume that reward/utility functions are fixed over episodes.
%This scenario includes examples in the work~\cite{schell2016data,el2018controlled,paternain2019safe} where only bandit observations are available.

Given a policy $\pi\in\Delta(\calA \, \vert \, \calS, H)$, the value function $V_{r,h}^{\pi}$ associated with the reward function $r$ at each step $h$ are the expected values of total rewards,
\[
V_{r,h}^{\pi}(x) \;=\; \mathbb{E}_\pi \sbr{\,\sum_{i \, = \, h}^{H}r_{i}(x_i,a_i ) \,\big\vert\, x_h =x\,}
%\;\text{ and }\; 
%V_{g,h}^{\pi}(x) \;=\; \mathbb{E}_\pi \sbr{\,\sum_{i \, = \, h}^{H}g_{i}(x_i,a_i ) \,\big\vert\, x_h =x\,}
\]
for all $x\in\calS$, $h\in[H]$, where the expectation $\mathbb{E}_\pi$ is taken over the random state-action sequence $\{(x_h,a_h)\}_{h\,=\,i}^H$; the action $a_h$ follows the policy $\pi_h( \, \cdot \, \vert x_h)$ at the state $x_h$ and the next state $x_{h+1}$ follows the transition dynamics $\mathbb{P}_h( \, \cdot \, \vert x_h,a_h)$. Thus, the action-value function $Q_{r,h}^{\pi}(x,a)$: $\calS\times\calA\to\mathbb{R}$ associated with the reward function $r$ is the expected value of total rewards when the agent starts from state-action pair $(x,a)$ at step $h$ and follows policy $\pi$,
\[
Q_{r,h}^{\pi}(x,a) \;=\; \mathbb{E}_\pi \sbr{\,\sum_{i\,=\,h}^{H}r_i(x_i,a_i ) \,\big\vert\,  x_h =x,a_h=a\,}
\] 
for all $(x,a)\in\calS\times\calA$ and $h\in[H]$. Similarly, we define the value function $V_{g,h}^{\pi}$: $\calS\to\mathbb{R}$ and the action-value function $Q_{g,h}^{\pi}(x,a)$: $\calS\times\calA\to\mathbb{R}$ associated with the utility function $g$.
% reads $Q_{g,h}^{\pi}(x,a) = \mathbb{E}_\pi [\,\sum_{i\,=\,h}^{H}g_i(x_i, a_i ) \,\big\vert\,  x_h =x,a_h=a\,]$. 
Denote symbol $\diamond=r$ or $g$. For brevity, we take the shorthand $\mathbb{P}_h V_{\diamond,h+1}^{\pi}(x,a) \DefinedAs \mathbb{E}_{x'\sim\mathbb{P}_h(\,\cdot\,\vert\, x,a)}V_{\diamond,h+1}^{\pi}(x')$. The Bellman equations associated with a policy $\pi$ are given by
\begin{equation}\label{eq.bellman}
\begin{array}{c}
Q_{\diamond,h}^{\pi}(x,a) \;=\; \big(\diamond_h\,+\,\mathbb{P}_h V_{\diamond,h+1}^{\pi}\big)(x,a)
%\;\;
%\text{and} 
%\;\;
\end{array}
\end{equation}
where $V_{\diamond,h}^{\pi}(x) = \big\langle{Q_{\diamond,h}^{\pi}\rbr{x,\,\cdot\,}},{\pi_h(\,\cdot\,\vert\, x)}\big\rangle_\calA$, for all $\rbr{x,a}\in\calS\times\calA$. Fix $x\,\in\,\calS$, the inner product of a function $f$: $\calS\times\calA\to\mathbb{R}$ with $\pi(\,\cdot\,\vert\, x)\in\Delta(\calA)$ is $\inner{f(x,\,\cdot\,)}{\pi(\,\cdot\,\vert\, x)}_\calA=\sum_{a\,\in\,\calA}\langle {f(x,a)},{\pi(a\,\vert\, x)}\rangle$.

%We express the Lagrangian function for the constrained problem~\eqref{eq.hindsight} as a finite sum $\mathcal{L}(\pi,Y) \;=\;\sum_{k=1}^{K} \mathcal{L}^k(\pi,Y)$ where $ \mathcal{L}^k(\pi,Y) = V_{r,1}^{\pi,k}(x_1) - Y (b-V_{g,1}^{\pi,k}(x_1))$ and $Y\geq 0$ is the Lagrangian multiplier or the dual variable. The solution to the problem~\eqref{eq.hindsight} can be obtained by solving the following minimax optimization problem,
%\[
%\minimize_{Y\geq 0} \;\maximize_{\pi\in\Delta(\calA\vert \calS, H)}\;\; \mathcal{L}(\pi,Y)
%\]
%which is a chanllenging nonconvex-concave problem. 

\textbf{Learning Performance}. The learning agent aims to find a solution of a constrained problem in which the objective function is the expected total rewards and the constraint is on the expected total utilities,
\begin{equation}\label{eq.hindsight}
\begin{array}{c}
\!\!\!\!
\maximize\limits_{\pi \, \in \, \Delta(\calA\,\vert\, \calS, H)}
\;
V_{r,1}^{\pi}(x_1)
%\\[0.3cm]
\;\;
\subject 
\;\;
V_{g,1}^{\pi}(x_1) \;\geq\; b
\end{array}
\end{equation}
where we take $b\in(0,H]$ to avoid triality. It is readily generalized to the problem with multiple constraints. Let $\pi^\star\in \Delta(\calA\,\vert\, \calS, H)$ be a solution to problem~\eqref{eq.hindsight}. Since the policy $\pi^\star$ is computed from knowing all reward and utility functions, it is commonly referred to as an optimal policy in-hindsight. 

The associated Lagrangian of problem~\eqref{eq.hindsight} is given by $\mathcal{L}(\pi,Y) \DefinedAs V_{r,1}^{\pi}(x_1)+ Y(V_{g,1}^{\pi}(x_1)-b)$, where $\pi$ is the primal variable and $Y\geq 0$ is the dual variable. We can cast problem~\eqref{eq.hindsight} into a saddle-point problem: $\maximize_{\pi\,\in\, \Delta(\calA\,\vert\, \calS, H)} \minimize_{Y\geq0} \mathcal{L}(\pi,Y) $, which is convex in $Y$ and is non-concave in $\pi$ due to the non-concavity of value functions in policy space~\cite{agarwal2019optimality}. To address the non-concavity, we will exploit constrained RL problem structure to propose a new variant of Lagrange multipliers method for constrained RL problems in Section~\ref{algorithm}, which warrants a new line of primal-dual mirror descent type analysis in sequel. It is different from unconstrained RL analysis, e.g.,~\cite{agarwal2019optimality,cai2019provably}.

Another key feature of constrained RL is the safe exploration under constraints~\cite{garcia2015comprehensive}. Without any constraint information~\emph{a priori}, it is infeasible for each policy to satisfy the constraint since utility information in constraints is only revealed after a policy is decided. Instead, we allow each policy to violate the constraint in each episode and minimize regret while minimizing total constraint violations for safe exploration over $K$ episodes; aslo see~\cite{efroni2020exploration}. We define the regret as the difference between the total reward value of policy $\pi^\star$ in hindsight and that of the agent's policy $\pi^k$ over $K$ episodes, and the constraint violation as a difference between the offset $Kb$ and the total utility value of the agent's policy $\pi^k$ over $K$ episodes,
\begin{equation}\label{eq.regret}
\begin{array}{rcl}
\text{\normalfont Regret}(K) &\!\!=\!\! &\displaystyle\sum_{k\,=\,1}^{K} \big(V_{r,1}^{\pi^\star}(x_1) -V_{r,1}^{\pi^k}(x_1) \big)
\\[0.2cm]
\text{\normalfont Violation}(K) &\!\!=\!\! &\displaystyle \left[\sum_{k\,=\,1}^{K} \big( b-V_{g,1}^{\pi^k}(x_1)\big)\right]_{+}.
\end{array}
\end{equation}
In this paper, we design algorithms, taking bandit feedback of the reward/utility functions, with both regret and constraint violation being sublinear in the total number of steps $T\DefinedAs HK$. Put differently, the algorithm should ensure that given $\epsilon>0$, if $T= {O}(1/\epsilon^2)$, then both $\text{\normalfont Regret}(K) = {O}(\epsilon)$ and $\text{\normalfont Violation}(K) = {O}(\epsilon)$ hold with high probability.

To achieve sublinear regret and constraint violation, we assume standard Slater condition for problem~\eqref{eq.hindsight} and recall strong duality; see the proof in~\cite{paternain2019constrained,paternain2019safe}, and boundedness of the optimal dual variable. The dual function is $\mathcal{D}(Y) \DefinedAs \maximize_{\pi}\mathcal{L}(\pi,Y)$ and the optimal dual variable $Y^\star\DefinedAs\argmin_{Y\,\geq\,0}\mathcal{D}(Y)$. 
\begin{assumption}[Slater Condition]
	\label{as.slater}
	There exists $\gamma>0$ and $\bar{\pi}\in\Delta(\calA\,\vert\, \calS,H)$ such that $V_{g,1}^{\bar{\pi}}(x_1) \geq b+\gamma$.
\end{assumption}

\begin{lemma}[Strong Duality, Boundedness of $Y^\star$]
	\label{lem.sd-b}
	Let Assumption~\ref{as.slater} hold. Then $V_{r,1}^{\pi^\star}(x_1) = \mathcal{D}(Y^\star)$. Moreover, $0\leq Y^\star\leq (V_{r,1}^{\pi^\star}(x_1) -V_{r,1}^{\bar{\pi}}(x_1) )/ \gamma$.
\end{lemma}

The Slater condition is mild in practice and commonly adopted in previous works; e.g.,~\cite{paternain2019constrained,efroni2020exploration,qiu2020upper,ding2020natural}. The implied properties for problem~\eqref{eq.hindsight} by Slater condition will be useful in our algorithm design and analysis.

\textbf{Linear Function Approximation}. We focus on CMDPs with linear transition kernels in feature maps. 
\begin{assumption}~\label{as.linearMDP}
	The $\text{\normalfont CMDP}(\calS,\calA,H,\mathbb{P},r,g)$ is a linear MDP with a kernal feature map $\psi$: $\calS\times\calA\times\calS\to\mathbb{R}^{d_1}$ and a value feature map $\varphi$: $\calS\times\calA\to\mathbb{R}^{d_2}$, if for any $h\in[H]$, there exists a vector $\theta_h\in\mathbb{R}^{d_1}$ with $\norm{\theta_h}_2\leq \sqrt{d_1}$ such that for any $(x,a,x') \in \calS\times\calA\times\calS$,
	\[
	\mathbb{P}_h\rbr{x'\,\vert\, x,a} \,=\, \langle\psi\rbr{x,a,x'},\theta_h\rangle;
	\]
	there exists a feature map $\varphi$: $\calS\times\calA\to\mathbb{R}^{d_2}$ and vectors $\theta_{r,h},\theta_{g,h}\in\mathbb{R}^{d_2}$ such that for any $(x,a) \in\calS\times\calA$,
	\[
	r_h(x,a) \,=\, \langle\varphi(x,a), \theta_{r,h}\rangle \text{ and } \, g_h(x,a) \,=\, \langle\varphi(x,a), \theta_{g,h}\rangle
	\]
	where $\max(\norm{\theta_{r,h}}_2,\norm{\theta_{g,h}}_2)\leq \sqrt{d_2}$.
	Moreover, we assume that for any function $V$: $\calS\to[0,H]$, $\norm{\int_{\calS} \psi(x, a, x') V(x') dx'}_2 \leq\sqrt{d_1}\, H$ for all $(x,a)\in\calS\times\calA$, and $\max(d_1,d_2) \leq d$.
\end{assumption}

Assumption~\ref{as.linearMDP} adapts the definition of linear kernel MDP~\cite{ayoub2020model,zhou2020provably,cai2019provably} for CMDPs. Linear kernel MDP examples include tabular MDPs~\cite{zhou2020provably}, feature embedded transition models~\cite{yang2019reinforcement}, and linear combinations of base models~\cite{modi2020sample}. We can construct related examples of CMDPs with linear structure by adding adding proper constraints. For usefulness of linear structure, see discussions in~\cite{du2019good,van2019comments,lattimore2019learning}. For more general transition dynamics, see factored linear model~\cite{pires2016policy}.

Although our definition in Assumption~\ref{as.linearMDP} and linear MDPs~\cite{yang2019sample,jin2019provably} all include tabular MDPs as special cases, they define transition dynamics using different feature maps. They are not comparable since one cannot be implied by the other~\cite{zhou2020provably}. The design of provably efficient RL algorithms beyond these linear type CMDPs remains open.

\section{Proposed Algorithm}
\label{algorithm}

In Algorithm~\ref{alg:OPDOP}, we present a new variant of Proximal Policy Optimization~\cite{schulman2017proximal} -- an $\underline{\text{O}}$ptimistic $\underline{\text{P}}$rimal-$\underline{\text{D}}$ual Proximal Policy $\underline{\text{OP}}$timization (OPDOP) algorithm. We effectuate the optimism through the Upper-Confidence Bounds (UCB) and address the constraints using a new variant of Lagrange multipliers method combined with the performance difference lemma; see Lemma~\ref{lem.PDL} in Appendix~\ref{app.sec.support} or works~\cite{schulman2015trust,schulman2017proximal,cai2019provably}. In each episode, our algorithm consists of three main stages. The first stage (lines 4--8) is \textbf{Policy Improvement}: we receive a new policy $\pi^k$ by improving previous $\pi^{k-1}$ via a mirror descent type optimization; The second stage (line 9) is \textbf{Dual Update}: we update dual variable $Y^k$ based on the constraint violation induced by previous policy $\pi^k$; The third stage (line 10) is \textbf{Policy Evaluation}: we optimistically evaluate newly obtained policy via the least-squares policy evaluation with an additional UCB bonus term for exploration.

\begin{algorithm}[H]
	\caption{$\underline{\text{O}}$ptimistic $\underline{\text{P}}$rimal-$\underline{\text{D}}$ual Proximal Policy $\underline{\text{OP}}$timization (OPDOP) }
	\label{alg:OPDOP}
	\begin{algorithmic}[1]
		\State \textbf{Initialization}: Let $\{Q_{r,h}^0,Q_{g,h}^0\}_{h\,=\,1}^H$ be zero functions, $\{\pi_h^0\}_{h\,\in\,[H]}$ be uniform distributions on $\mathcal{A}$, $V_{g,1}^0$ be $b$, $Y^0$ be $0$, $\chi$ be $2H/\gamma$, $\alpha,\eta>0,\theta\in(0,1]$.
		\For{episode $k=1,\ldots,K+1$} 
		\State Set the initial state $x_1^k = x_1$. 
		\For{step $h=1,2,\ldots,H$ } 
		\State Mix the policy
		\[
		\displaystyle\tilde{\pi}_h^{k-1} (\,\cdot\,\vert\, \cdot\,) \,\leftarrow\, (1-\theta) \,\pi_h^{k-1}(\,\cdot\,\vert\, \cdot\,)\,+\, \theta\,\text{Unif}(\calA).
		\]
		\State Update the policy
		\begin{equation}\label{eq.closed}
		\!\!\!\! \!\!\!\! \!\!
		\pi_h^k(\,\cdot\,\vert\, \cdot\,) 
		\,\propto\,
		\tilde{\pi}_h^{k-1}(\,\cdot\,\vert\, \cdot\,) \,\mathrm{e}^{\rbr{\alpha\,\big( Q_{r,h}^{k-1}\,+\, Y^{k-1}Q_{g,h}^{k-1}\big) (\,\cdot,\,\cdot\,)}}\!\!.
		\end{equation}
		%		\[
		%		\displaystyle\argmax_{\pi\in\Delta(\calA)} \;\big\langle{(\psi Q_{r,h}^{k-1}\,+\,Y^{k-1}Q_{g,h}^{k-1})(x_h,\,\cdot\,)},{\pi(\,\cdot\,\vert\, x_h)}\big\rangle \, -\, \displaystyle\frac{1}{\alpha} \,D\big(\pi_h(\,\cdot\,\vert\, x_h) \,\vert\,  \tilde{\pi}_h^{k-1}(\,\cdot\,\vert\, x_h)\big).
		%		\]
		\State Take an action $a_h^k\,\sim\,\pi_h^k(\,\cdot\,\vert\, x_h^k\,)$ and 
		recieve reward and utility $r_h(x_h^k,a_h^k),\, g_h(x_h^k,a_h^k)$.
		%		\[
		%		\text{ Full-information:}\; r_h^k(\,\cdot,\,\,\cdot\,),\, g_h^k(\,\cdot,\,\cdot\,).
		%		\;\,\text{ or }\;\,
		%		\text{ Bandit:}\; r_h(x_h^k,a_h^k),\, g_h(x_h^k,a_h^k).
		%		\]
		\State Observe the next state $x_{h+1}^k$.
		\EndFor
		\State Update the dual variable $Y^{k} $ by  
		\begin{equation}\label{eq.dual}
		\!\!\!\! \!\!\!\! \!\!
		Y^{k}\,\leftarrow\,
		\text{Proj}_{[\,0,\, \chi\, ]} \big({\, Y^{k-1} \,+\,\eta\, (b -  V_{g,1}^{k-1}(x_1))\,}\big).
		\end{equation}
		\State  Estimate the action-value or value functions $\{Q_{r,h}^k(\,\cdot\, ,\,\cdot\,), Q_{g,h}^k(\,\cdot\, ,\,\cdot\,), V_{g,h}^k(\,\cdot\,)\}_{h\,=\,1}^H$ via
		\[
		\text{LSTD}\!\rbr{\{x_h^\tau,a_h^\tau,r_h(x_h^\tau,a_h^\tau),g_h(x_h^\tau,a_h^\tau)\}_{h,\tau\,=\,1}^{H,k}}.
		\]
		\EndFor
	\end{algorithmic}
\end{algorithm}

\noindent\textbf{Policy Improvement}. 
In the $k$-th episode, a natural attempt of obtaining a policy $\pi^k$ is to solve a Lagrangian-based policy optimization problem,
\[
\displaystyle\maximize_{\pi\,\in\,\Delta(\calA\vert \calS, H)} \mathcal{L}(\pi,\!Y^{k-1}) \,\DefinedAs\, V_{r,1}^{\pi}(x_1) - Y^{k-1}\!(b-V_{g,1}^{\pi}(x_1))
\]
where  $\mathcal{L}(\pi,Y)$ is the Lagrangian and the dual variable $Y^{k-1}\geq 0$ is from the last episode. We will show later that $Y^{k-1}$ can be updated efficiently.
%
%$\mathcal{L}^{k-1}(\pi,Y)$,
%where $Y^{k-1}\geq 0$ is some Lagrange multiplier obtained from dual updates,
%\begin{equation}\label{eq:nat_dual}
%Y^{k} \;=\; \max \big(\, Y^{k-1}+ (b-V_g^{\pi,k-1}(\pi^k)),\,0\,\big).
%\end{equation}
This type update is used by several works, e.g. \cite{le2019batch,paternain2019constrained,paternain2019safe,tessler2018reward}. The issue with them is that it relies on an oracle solver involving either Q-learning~\cite{ernst2005tree}, PPO~\cite{schulman2017proximal}, or TRPO~\cite{schulman2015trust} to deliver a near-optimal policy. Thus, the overall algorithmic complexity is expensive and it is not suitable for the online use. In contrast, this work utilizes RL problem structure and show that only an easily-computable proximal step is sufficient by establishing efficiency guarantees without using any oracle solvers. 

Denote symbol $\diamond = r\text{ or } g$.
Via the performance difference lemma, we can expand value functions $V_{\diamond,1}^{\pi}(x_1)$ at the previously known policy $\pi^{k-1}$,
\[
\begin{array}{rcl}
&& \!\!\!\!\!\!\!\!\!\!\!\!\!\!  V_{\diamond,1}^{\pi}(x_1) \;=\; V_{\diamond,1}^{\pi^{k-1}}(x_1^k)  
\,+\,
\mathbb{E}_{\pi^{k-1}}\sbr{\,\displaystyle\sum_{h\,=\,1}^{H} \big\langle{Q_{\diamond,h}^{\pi}(x_h,\cdot\,)},{(\pi_h - \pi_h^{k-1})(\,\cdot\,\vert\,  x_h)}\big\rangle\,}
\end{array}
\]
where $\mathbb{E}_{\pi^{k-1}}$ is taken over the random state-action sequence $\{(x_h,a_h)\}_{1}^H$.
%we treat the unknown reaward and utility value functions via the performance difference lemma : one can approximate value functions $V_{\diamond,1}^{\pi,k-1}(x_1)$, where $\diamond= r$ or $g$, at the previously known policy $\pi^{k-1}$ as follows,
%where the \textit{linear term} is specified by $\mathbb{E}_{\pi^{k-1}} \sbr{\sum_{h=1}^{H} \langle{Q_{\diamond,h}^{\pi^{k-1},k-1}(x_h,\cdot)},{(\pi_h - \pi_h^{k-1})(\cdot\vert  x_h)}\rangle\,\vert\, x_1}$.
Thus, we introduce an approximation of $V_{\diamond,1}^{\pi}(x_1)$ for any state-action sequence $\{(x_h,a_h)\}_{1}^H$ induced by $\pi$,
\[
\begin{array}{rcl}
&& \!\!\!\!\!\!\!\!\!\!\!\!\!\!\!\!  L_\diamond^{k-1}(\pi)\;=\; V_{\diamond,1}^{k-1}(x_1) 
\,+\,
\displaystyle\sum_{h\,=\,1}^{H} \big\langle{Q_{\diamond,h}^{k-1}(x_h,\cdot\,)},{(\pi_h - \pi_h^{k-1})(\,\cdot\,\vert\,  x_h)}\big\rangle
\end{array}
\]
where $V_{\diamond,h}^{k-1}$ and $Q_{\diamond,h}^{k-1}$ can be estimated from an optimistic policy evaluation that will be discussed later. With this notion, in each episode, instead of solving a Lagrangian-based policy optimization, we perform a simple policy update in online mirror descent fashion,
\[
\begin{array}{rcl}
&& \!\!\!\! \!\!\!\! \displaystyle\maximize_{\pi\,\in\,\Delta(\calA\vert \calS, H)} \;  L_r^{k-1}(\pi)
\,-\,
Y^{k-1} \big(b-L_g^{k-1}(\pi)\big) 
%\\[0.2cm]
%&& \;\;\;\; \;\;\;\; \;\;\;\; \;\;\;\; \;\;\;\; \,-\,
\displaystyle\frac{1}{\alpha}\sum_{h\,=\,1}^{H} D\big(\pi_h(\,\cdot\,\vert\, x_h) \,\vert\, \tilde{\pi}_h^{k-1}(\,\cdot\,\vert\, x_h)\big)
\end{array}
\]
where $\tilde{\pi}_h^{k-1} (\,\cdot\,\vert\, x_h) = \rbr{1-\theta} \pi_h^{k-1}(\,\cdot\,\vert\, x_h)+\theta\,\text{Unif}(\calA)$
is a mixed policy of the previous one and the uniform distribution $\text{Unif}(\calA)$ with $\theta\in(0,1]$. The constant $\alpha>0$ is a trade-off parameter, $D(\pi\,\vert\,\tilde{\pi}^{k-1})$ is the KL divergence between $\pi$ and $\tilde{\pi}^{k-1}$ in which $\pi$ is absolutely continuous in $\tilde{\pi}^{k-1}$. The policy mixing step ensures such absolute continuity and implies uniformly bounded KL divergence (see Lemma~\ref{lem.mix} in Appendix~\ref{app.sec.support}).
%We estimate the action-value functions $Q_{\ell,h}^{\pi^{k-1},k-1}$ from the policy evaluation and denote it as $Q_{\ell,h}^{k-1}$ for $\ell=r$ or $g$. 
Ignoring other $\pi$-irrelevant terms, we update $\pi^k$ in terms of previous policy $\pi^{k-1}$ by
\[
\begin{array}{rcl}
&& \!\!\!\! \!\!\!\! \!\!\!\! \!\!\!\!
\displaystyle\argmax_{\pi\in\Delta(\calA\vert \calS, H)} \sum_{h\,=\,1}^{H}\! \big\langle{( Q_{r,h}^{k-1}+Y^{k-1}Q_{g,h}^{k-1})(x_h,\,\cdot\,)},{\pi_h (\,\cdot\,\vert\,  x_h) }\big\rangle
%\\[0.2cm]
%&& \;\;\;\; \;\;\;\; \;\;\;\; \;\;\;\; \;\;\;\;  \,-\,
\displaystyle\frac{1}{\alpha}\!\sum_{h\,=\,1}^{H}\! D\big(\pi_h(\,\cdot\,\vert\, x_h) \,\vert\, \tilde{\pi}_h^{k-1}(\,\cdot\,\vert\, x_h)\big).
\end{array}
\]
Since the above update is separable over states $\{x_h\}_1^H$, we can update the policy $\pi^k$ as line~6 in Algorithm~\ref{alg:OPDOP}, a closed solution for any step $h\in[H]$.
%\begin{equation}\label{eq.closed}
%\pi_h^k(\,\cdot\,\vert\, x_h) 
%\;\propto\;
%\tilde{\pi}_h^{k-1}(\,\cdot\,\vert\, x_h) \,e^{\alpha\,\big(\psi Q_{r,h}^{k-1}\,+\, Y^{k-1}Q_{g,h}^{k-1}\big) (x_h,\,\cdot\,)}.
%\end{equation}
If we set $Y^{k-1}=0$ and $\theta=0$, the above update reduces to one step in an optimistic variant of PPO~\cite{cai2019provably}. The idea of KL-divergence regularization in policy optimization has been widely used in many unconstrained scenarios, e.g., NPG~\cite{kakade2002natural}, PPO~\cite{schulman2017proximal}, TRPO~\cite{schulman2015trust}, and their neural variants~\cite{wang2019neural,liu2019neural}. Our method is distinct from them in that it is based on the performance difference lemma and  the optimistically estimated value functions.

\noindent\textbf{Dual Update}. 
To infer the constraint violation for the dual update, we esimate $V_{g,1}^{k-1}(x_1)$ for $V_{g,1}^{\pi^k}(x_1)$ via an optimistic policy evaluation that will be discussed later. We update the Lagrange multiplier $Y\geq 0$ by moving $Y^{k}$ to the direction of minimizing the Lagrangian $\mathcal{L}(\pi,Y)$ over $Y$ in line~9 in Algorithm~\ref{alg:OPDOP},
%\[
%Y^{k} \;=\; \text{Proj}_{[\,0,\, \chi\, ]} \big({\, Y^{k-1} \,+\,\eta\,(b \,-\,  V_{g,1}^{k-1}(x_1))\,}\big).
%\]
where $\eta>0$ is a stepsize and $\text{Proj}_{[\,0,\, \chi\,]}$ is a projection onto $[0,\chi]$ with an upper bound $\chi$ on $Y^{k}$. By Lemma~\ref{lem.sd-b}, we choose $\chi={2H}/{\gamma}\geq 2Y^\star$ so that projection interval $[\,0,\, \chi\,]$ includes the optimal dual variable $Y^\star$; also see works~\cite{efroni2020exploration,nedic2009subgradient}.

The dual update works as a trade-off between the reward maximization and the constraint violation reduction. If the current policy $\pi^k$ satisfies the approximated constraint, i.e., $b-L_g^{k-1}(\pi^k)\leq 0$, we put less weight on the action-value function associated with the utility and maximize the reward; otherwise, we sacrifice the reward a bit to satisfy the constraint. The dual update has a similar use in dealing with constraints in CMDP, e.g., Lagrangian-based AC~\cite{chow2017risk,liang2018accelerated}, and online constrained optimization problems~\cite{yu2017online,wei2019online,yuan2018online}. In contrast, we handle the dual update via the optimistic policy evaluation that yields an arguably efficient approximation of constraint violation. 

\begin{algorithm}[H]
	\caption{Least-Squares Temporal Difference (LSTD) with UCB exploration}
	\label{alg:LSVI}
	%	\setstretch{1.3}
	\begin{algorithmic}[1]
		\State \textbf{Input}: $\{x_h^\tau,a_h^\tau,r_h(x_h^\tau,a_h^\tau), g_h(x_h^\tau,a_h^\tau)\}_{h,\tau\,=\,1}^{H,k}$.
		\State \textbf{Initialization}: Set $\{V_{r,H+1}^k,V_{g,H+1}^k\}$ be zero functions and $\lambda =1,\beta=O(\sqrt{dH^2\log\rbr{{dT}/{p}}})$.
		\For{step $h=H,H-1,\cdots,1$} \Comment{$\diamond = r$, $g$}
		\State $\!\!\!\!\!\!\!\!\Lambda_{\diamond,h}^k \,\leftarrow\, \displaystyle\sum_{\tau\,=\,1}^{k-1} \phi_{\diamond,h}^\tau(x_h^\tau,a_h^\tau)\phi_{\diamond,h}^\tau(x_h^\tau,a_h^\tau)^\top+ \lambda I$.
		%		\State $\Lambda_{g,h}^k \,\leftarrow\, \displaystyle\sum_{\tau\,=\,1}^{k-1} \phi_{g,h}^\tau(x_h^\tau,a_h^\tau)\phi_{g,h}^\tau(x_h^\tau,a_h^\tau)^\top+ \lambda I$.
		\State $\!\!\!\!\!\!\!\!w_{\diamond,h}^k \,\leftarrow\, (\Lambda_{\diamond,h}^k)^{-1}\displaystyle \sum_{\tau\,=\,1}^{k-1}\phi_{\diamond,h}^\tau(x_h^\tau,a_h^\tau) V_{\diamond,h+1}^\tau (x_{h+1}^\tau)$.
		%		\State $w_{g,h}^k \,\leftarrow\, (\Lambda_{g,h}^k)^{-1}\displaystyle \sum_{\tau\,=\,1}^{k-1}\phi_{g,h}^\tau(x_h^\tau,a_h^\tau) \big(g_h(x_h^\tau,a_h^\tau)+V_{g,h+1}^\tau (x_{h+1}^\tau)\big)$.
		\State $\!\!\!\!\!\!\!\!\phi_{\diamond,h}^k(\,\cdot\,,\,\cdot\,) \,\leftarrow\, \int_{\calS} \psi(\,\cdot\,,\,\cdot\,,x'\, ) V_{\diamond,h+1}^k(x')dx'$.
		\State $\!\!\!\!\!\!\!\!\Gamma_{\diamond,h}^k(\,\cdot\,,\,\cdot\,) \,\leftarrow\, \beta (\phi_{\diamond,h}^k(\,\cdot\,,\,\cdot\,)^\top(\Lambda_{\diamond,h}^k)^{-1} \phi_{\diamond,h}^k(\,\cdot\,,\,\cdot\,) )^{1/2}$.
		%		\State $\phi_{g,h}^k(\,\cdot\,,\,\cdot\,) \,\leftarrow\, \int_{\calS} \psi(\,\cdot\,,\,\cdot\,,x'\, ) V_{g,h+1}^k(x')dx'$.
		\State $\!\!\!\!\!\!\!\!\Lambda_{h}^k \,\leftarrow\, \displaystyle\sum_{\tau\,=\,1}^{k-1} \varphi(x_h^\tau,a_h^\tau)\varphi(x_h^\tau,a_h^\tau)^\top+ \lambda I$.
		\State $\!\!\!\!\!\!\!\! u_{\diamond,h}^k \,\leftarrow\, (\Lambda_{h}^k)^{-1}\displaystyle \sum_{\tau\,=\,1}^{k-1}\varphi(x_h^\tau,a_h^\tau) \diamond_{h} (x_{h}^\tau,a_h^\tau)$.
		%		\State $\Gamma_{g,h}^k(\,\cdot\,,\,\cdot\,) \,\leftarrow\, \beta (\phi_{g,h}^k(\,\cdot\,,\,\cdot\,)^\top(\Lambda_{g,h}^k)^{-1} \phi_{g,h}^k(\,\cdot\,,\,\cdot\,) )^{1/2}$.
		\State $\!\!\!\!\!\!\!\! \Gamma_{h}^k(\,\cdot\,,\,\cdot\,) \,\leftarrow\, \beta (\varphi(\,\cdot\,,\,\cdot\,)^\top(\Lambda_{h}^k)^{-1} \varphi(\,\cdot\,,\,\cdot\,) )^{1/2}$.
		\State $\!\!\!\! \!\!\!\! Q_{\diamond,h}^k(\,\cdot\,,\cdot\,) \,\leftarrow\, \min\big(\,\varphi(\,\cdot\,,\,\cdot\,)^\top u_{\diamond,h}^k\,+\,\phi_{\diamond,h}^k(\,\cdot\,,\,\cdot\,)^\top w_{\diamond,h}^k\,+\,(\Gamma_{h}^k +\Gamma_{\diamond,h}^k)(\,\cdot\,,\,\cdot\,),\, H-h+1\,\big)^+ $.
		%		\State $Q_{g,h}^k(\,\cdot\,,\cdot\,) \,\leftarrow\, \min\big(\,\phi_{g,h}^k(\,\cdot\,,\,\cdot\,)^\top w_{g,h}^k+\Gamma_{g,h}^k(\,\cdot\,,\,\cdot\,),\, H-h+1\,\big)^+ $.
		\State $\!\!\!\!\!\!\!\!V_{\diamond,h}^k(\,\cdot\,) \,\leftarrow\, \big\langle{Q_{\diamond,h}^k(\,\cdot\,,\,\cdot\,)},{\pi_h^k(\,\cdot\,\vert\, \cdot\,)}\big\rangle_\calA$.
		%		\State $V_{g,h}^k(\,\cdot\,) \,\leftarrow\, \big\langle{Q_{g,h}^k(\,\cdot\,,\,\cdot\,)},{\pi_h^k(\,\cdot\,\vert \,\cdot\,)}\big\rangle_\calA$.
		\EndFor
		\State \textbf{Return}: $\{Q_{\diamond,h}^k(\,\cdot\,,\,\cdot\,),V_{\diamond,h}^k(\,\cdot\,,\,\cdot\,) \}_{h\,=\,1}^H$.
	\end{algorithmic}
\end{algorithm}

\noindent\textbf{Policy Evaluation}. The last stage of the $k$th episode takes the Least-Squares Temporal Difference (LSTD)~\cite{bradtke1996linear,boyan1999least,lazaric2010finite,lagoudakis2003least} to evaluate the policy $\pi^k$ based on previous $k-1$ historical trajectories. 
For each step $h\in[H]$, instead of $\mathbb{P}_hV_{r,h+1}^{\pi^k}$ in the Bellman equations~\eqref{eq.bellman}, we estimate $\mathbb{P}_hV_{r,h+1}^{k}$ by $(\phi_{r,h}^k)^\top w_{r,h}^k$ where $w_{r,h}^k $ is updated by the minimizer of the regularized least-squares problem over $w$,
\begin{equation}\label{eq.lsQ}
%w_{r,h}^k 
%\;\leftarrow\; 
%\argmin_{w\in\mathbb{R}^d} \; 
\sum_{\tau\,=\,1}^{k-1} \big(V_{r,h+1}^\tau(x_{h+1}^\tau)\,-\,{\phi_{r,h}^\tau(x_h^\tau,a_h^\tau)}^\top{w}\big)^2 \,+\, \lambda\,\|w\|_2^2
\end{equation}
where $\phi_{r,h}^\tau(\,\cdot\,,\,\cdot\,) \DefinedAs \int_{\calS} \psi(\,\cdot\,,\,\cdot\,,x'\, ) V_{r,h+1}^\tau(x')dx'$, $V_{r,h+1}^\tau(\cdot) = \langle{Q_{r,h+1}^\tau(\,\cdot\,,\cdot\,)},{\pi_{h+1}^\tau(\,\cdot\,\vert \,\cdot\,)}\rangle_\calA$ for $h\in[H-1]$ and $V_{H+1}^\tau = 0$, and $\lambda>0$ is the regularization parameter. Similarly, we estimate $\mathbb{P}_hV_{g,h+1}^{k}$ by $(\phi_{g,h}^k)^\top w_{g,h}^k$. We display the least-squares solution in line 4--6 of Algorithm~\ref{alg:LSVI} where symbol $\diamond = r\text{ or } g$. We also estimate $r_h(\cdot,\cdot)$ by $\varphi^\top u_{r,h}^k$ where $u_{r,h}^k$ is updated by the minimizer of another regularized least-squares problem,
% over $u$,
\begin{equation}\label{eq.lsR}
%w_{r,h}^k 
%\;\leftarrow\;
%\argmin_{w\in\mathbb{R}^d} 
%\; 
\sum_{\tau\,=\,1}^{k-1} \big(r_h(x_h^\tau,a_h^\tau)\,-\,{\varphi(x_h^\tau,a_h^\tau)}^\top{u}\big)^2 \,+\, \lambda\,\|u\|_2^2
\end{equation}
where $\lambda>0$ is the regularization parameter. Similarly, we estimate $g_h(\cdot,\cdot)$ by $\varphi^\top u_{g,h}^k$. The least-squares solutions lead to line 8--9 of Algorithm~\ref{alg:LSVI}. 

After obtaining estimates of $\mathbb{P}_hV_{\diamond,h+1}^{k}$ and $\diamond_h(\cdot,\cdot)$ for $\diamond = r\text{ or } g$, we update the estimated action-value function $\{Q_{\diamond,h}^k\}_{h\,=\,1}^H$ iteratively in line 11 of Algorithm~\ref{alg:LSVI}, 
%\[
%Q_{\diamond,h}^k(\,\cdot\,,\cdot\,) \,\leftarrow\, \min\big(\,\varphi(\,\cdot\,,\,\cdot\,)^\top u_{\diamond,h}^k+\phi_{\diamond,h}^k(\,\cdot\,,\,\cdot\,)^\top w_{\diamond,h}^k+(\Gamma_{h}^k +\Gamma_{\diamond,h}^k)(\,\cdot\,,\,\cdot\,),\, H-h+1\,\big)^+ 
%\]
where $\varphi^\top u_{\diamond,h}^k$ is an estimate of $\diamond_h$ and $(\phi_{\diamond,h}^k)^\top w_{\diamond,h}^k$ is an estimate of $\mathbb{P}_h V_{\diamond,h+1}^k$; we add UCB bonus terms $\Gamma_{h}^k(\,\cdot\,,\,\cdot\,),\Gamma_{\diamond,h}^k(\,\cdot\,,\,\cdot\,)$: $\calS\times\calA\to\mathbb{R}^+ $ so that 
\[
\varphi^\top u_{\diamond,h}^k + \Gamma_{h}^k \;\text{ and }\; (\phi_{\diamond,h}^k)^\top w_{\diamond,h}^k+\Gamma_{\diamond,h}^k
\]
all become upper confidence bounds for exploration. We take $\Gamma_{h}^k =\beta (\varphi^\top (\Lambda_{h}^k)^{-1}\varphi)^{1/2}$ and $\Gamma_{\diamond,h}^k =\beta ((\phi_{\diamond,h}^k)^\top (\Lambda_{\diamond,h}^k)^{-1}\phi_{\diamond,h}^k)^{1/2}$ and leave the parameter $\beta>0$ to be tuned later. Moreover, the bounded reward/utility $\diamond_h\in[0,1]$ implies $Q_{\diamond,h}^k\in[0,H-h+1]$. 

We remark the computational efficiency of Algorithm~\ref{alg:OPDOP}. For the time complexity, since line 6 is a scalar update, they need $O(d|\calA|T)$ time. A dominating calculation is from lines 5/9 in Algorithm~\ref{alg:LSVI}. If we use the Sherman–Morrison formula for computing $(\Lambda_h^k)^{-1}$, it takes $O(d^2T)$ time. Another important calculation is the integration from line 6 in Algorithm~\ref{alg:LSVI}. We can either compute it analytically if it is tractable or approximate it via the Monte Carlo integration efficiently~\cite{zhou2020provably} that assumes $O(dT)$ time. Therefore, the time complexity is $O(d^2|\calA|T)$ in total. For the space complexity, we don't need to store policy since it is recursively calculated via~\eqref{eq.closed}. By updating $Y^k$, $\Lambda_{h}^k$, $\Lambda_{\diamond,h}^k$, $w_{\diamond,h}^k$, $u_{\diamond,h}^k$, and $\diamond_h(x_h^k,a_h^k)$ recursively, it takes $O((d^2+|\calA|)H)$ space.

\section{Regret and Constraint Violation Analysis}
\label{regret}

We now prove that the regret and the constraint violation for Algorithm~\ref{alg:OPDOP} are sublinear in $T\DefinedAs KH$; $T$ is the total number of steps taken by the algorithm; $K$ is the total number of episodes; $H$ is the episode horizon. We recall that $|\calA|$ is the cardinality of action space $\calA$ and $d$ is the dimension of the feature map.

\begin{theorem}[Linear Kernal MDP: Regret and Constraint Violation]\label{thm.main-full}
	Let Assumptions~\ref{as.slater}~and~\ref{as.linearMDP} hold. 
	Fix $p\in\rbr{0,1}$. 
	We set
	$\alpha={\sqrt{\log\abr{\calA}}}/(H^2K)$, $\beta=C_1\sqrt{dH^2\log\rbr{{dT}/{p}}}$, $\eta=1/\sqrt{K}$, $\theta = {1}/K$, and $\lambda=1$ in Algorithm~\ref{alg:OPDOP}, where $C_1$ is an absolute constant. Suppose $\log\abr{\calA}=O\rbr{d^2\log^2\rbr{{dT}/{p}}}$. Then, with probability $1-p$, the regret and the constraint violation in~\eqref{eq.regret} satisfy 
	\[
	\begin{array}{rcl}
	\text{\normalfont Regret}(K) &\!\!\leq\!\!& \displaystyle C\, d H^{2.5}\sqrt{T} \,\log\rbr{\tfrac{dT}{p}}
	\\[0.2cm]
	\text{\normalfont Violation}(K) &\!\!\leq\!\!& \displaystyle C'\,d H^{2.5}\sqrt{T} \,{\log\rbr{\tfrac{dT}{p}}}
	\end{array}
	\]
	where $C$ and $C'$ are absolute constants. 
\end{theorem}
%\begin{proof}
%	In Section~\ref{proof}, we provide a proof sketch that supports complete proof in Appendix~\ref{ap.main1}. 
%\end{proof}

The above result establishes that OPDOP enjoys an $\tilde{O}(d H^{2.5}\sqrt{T} )$ regret and an $\tilde{O}(d H^{2.5}\sqrt{T} )$ constraint violation if we set algorithm parameters $\{\alpha,\beta,\eta,\theta,\lambda\}$ properly. Our results have the optimal dependence on the total number of steps $T$ up to some logarithmic factors. The $d$ dependence occurs due to the uniform concentration for controlling the fluctuations in the least-squares policy evaluation. This matches the existing bounds in the linear MDP setting without any constraints~\cite{cai2019provably,ayoub2020model,zhou2020provably}. Our bounds differ from them only by an extra $H$, which is a price introduced by the uniform bound on the constraint violation. It is noticed that we don't observe reward/utility functions, our algorithm works for  bandit feedback after each episode. 

We remark the tabular setting for Algorithm~\ref{alg:OPDOP}; see Appendix~\ref{ap.CMDP-t} for details. The tabular CMDP is a special case of Assumption~\ref{as.linearMDP} by taking canonical bases as feature mappings; see them in Section~\ref{sec.tabular}. The feature map has dimension $d=|\calS|^2|\calA|$ and thus Theorem~\ref{thm.main-full} automatically provides $O(|\calS|^2|\calA|H^{2.5}\sqrt{T})$ regret and constraint violation for the tabular CMDPs. The $d=|\calS|^2|\calA|$ dependence relies on the least-squares policy evaluation and it can be improved via other optimistic policy evolution methods if we are limited to the tabular case. We provide such results in Section~\ref{sec.tabular}.
% achieving $O(|\calS||\calA|^{0.5}H^{2.5}\sqrt{T})$ regret and constraint violation.

\subsection{Proof Outline of Theorem~\ref{thm.main-full}}
\label{proof}

We sketch the proof for Theorem~\ref{thm.main-full}. We state key lemmas and delay their full versions and complete proofs to Appendix~\ref{ap.main1}. In what follows, we fix $p\in(0,1)$ and use the shorthand w.p. for with probability.

%The proof consists of the regret analysis and the constraint violation analysis. 

\noindent\textbf{Regret Analysis}. We take a regret decomposition,
\[
\text{Regret}(K) \;=\; {\normalfont\text{(R.I)} } \,+\, {\normalfont\text{(R.II)} }
\]
where ${\normalfont\text{(R.I)} }  = \sum_{k\,=\,1}^{K}\big(V_{r,1}^{\pi^\star}(x_1) - V_{r,1}^{k}(x_1)\big)$ in which $\pi^\star$ is an optimal policy in hindsight, and ${\normalfont\text{(R.II)} } = \sum_{k\,=\,1}^{K}\big(V_{r,1}^{k}(x_1) -V_{r,1}^{\pi^k}(x_1)\big)$ in which $V_{r,1}^k(x_1)$ is estimated via our optimistic policy evaluation given by Algorithm~\ref{alg:LSVI}. As we use $V_{r,h+1}^k$ to estimate $V_{r,h+1}^{\pi^k}$, it leads a model prediction error in the Bellman equations, $\iota_{r,h}^k \DefinedAs r_h^k+\mathbb{P}_h V_{r,h+1}^k-Q_{r,h}^k$; similarly define $\iota_{g,h}^k$. In Appendix~\ref{subsec.ucb-bandit}, the UCB optimism of $\iota_{\diamond,h}^k$ with $\diamond=r$ or $g$, shows that or any $(k,h)\in[K]\times[H]$ and $(x,a)\in\calS\times\calA$, w.p. $1-p/2$, we have
\[
- 2 (\Gamma_h^k+\Gamma_{\diamond,h}^k)(x,a) \;\leq\; \iota_{\diamond,h}^k(x,a) \;\leq\; 0.
\]

By assumptions of Theorem~\ref{thm.main-full}, the policy improvement~\eqref{eq.closed} yields Lemma~\ref{lem.1}, depicting total differences of estimates $ V_{r,1}^{k}(x_1)$, $ V_{g,1}^{k}(x_1)$ to the optimal ones. 
\begin{manuallemma}{1}[Policy Improvement: Primal-Dual Mirror Descent Step]
	\label{lem.1}
	Let assumptions of Theorem~\ref{thm.main-full} hold. Then, 
	$\sum_{k\,=\,1}^{K}
	\big(V_{r,1}^{\pi^\star}(x_1) - V_{r,1}^{k}(x_1)\big) \,+\,\sum_{k\,=\,1}^{K} Y^{k} \big(V_{g,1}^{\pi^\star} (x_1)- V_{g,1}^{k}(x_1)\big)
	\leq O(H^{2.5}\sqrt{T\log|\calA|})+\sum_{k\,=\,1}^{K} \sum_{h\,=\,1}^H \mathbb{E}_{\pi^\star}[\iota_{r,h}^{k}(x_h,a_h)+Y^{k}\iota_{g,h}^{k}(x_h,a_h)].$
\end{manuallemma}
Lemma~\ref{lem.1} displays a coupling between the regret and the constraint violation. This coupling also finds use in online convex optimization~\cite{mahdavi2012trading,yuan2018online,wei2019online,koppel2019projected} and CMDP problems~\cite{efroni2020exploration}. 
The proof of Lemma~\ref{lem.1} takes a primal-dual mirror descent type analysis of line~6 of Algorithm~\ref{alg:OPDOP} by utilizing the performance difference lemma. Our primal-dual mirror descent analysis enable standard ones in~\cite{schulman2017proximal,liu2019neural,cai2019provably} to deal with constraints. 

Via the dual update~\eqref{eq.dual}, we show the second total differences $-\sum_{k\,=\,1}^{K} Y^{k} \big(V_{g,1}^{\pi^\star} (x_1)- V_{g,1}^{k}(x_1)\big)$ scales $O(\sqrt{K})$. Together with a decomposition of $\normalfont\text{(R.II)}$, $\normalfont\text{(R.II)} =-\sum_{k\,=\,1}^{K}\sum_{h\,=\,1}^H\iota_{r,h}^k(x_h^k,a_h^k)+ M_{r,H,2}^K$, where $M_{r,H,2}^K$ is a martingale, we have Lemma~\ref{lem.2}.
\begin{manuallemma}{2}\label{lem.2}
	Let assumptions of Theorem~\ref{thm.main-full} hold. Then, 
	${\normalfont\text{Regret}}(K) \leq O(H^{2.5}\sqrt{T \log|\calA|})+ \sum_{k\,=\,1}^{K}\sum_{h\,=\,1}^H\big({ \mathbb{E}_{\pi^\star}\![\iota_{r,h}^k(x_h,a_h) ]\!-\!\iota_{r,h}^k(x_h^k,a_h^k)}\big)+ M_{r,H,2}^K.$
\end{manuallemma}
Finally, we note that $M_{r,H,2}^K$ is a martingale that scales as $O(H\sqrt{T})$ via the Azuma-Hoeffding inequality. For the model prediction error, we use the UCB optimism and apply the elliptical potential lemma.
\begin{manuallemma}{3}\label{lem.3}
	Let assumptions of Theorem~\ref{thm.main-full} hold. Then, 
	$\sum_{k\,=\,1}^{K}\!\sum_{h\,=\,1}^H\! \big(\mathbb{E}_{\pi^\star}\![\iota_{r,h}^k(x_h,a_h) ]-\iota_{r,h}^k(x_h^k,a_h^k)\big)
	\leq
	O(dH^{1.5}\!\sqrt{T\log\rbr{K}\log\rbr{{dT}/{p}}})$, $\normalfont\text{w.p.}\, 1-p/2$.
\end{manuallemma}

\begin{manuallemma}{4}\label{lem.4}
	Let assumptions of Theorem~\ref{thm.main-full} hold. Then, 
	$\big\lvert{M_{r,H,2}^K}\big\rvert \leq4 H\sqrt{T \log\rbr{{4}/{p}}}$, $\normalfont\text{w.p.}\, 1-p/2$.
\end{manuallemma}

Applying probability bounds from Lemmas~\ref{lem.3} and~\ref{lem.4} to Lemma~\ref{lem.2} yields our regret bound. 

%\rmargin{the proof sketch in Section 4 basically provides a subsample of formulas from the appendix without much explanation. Here, I would have appreciated at least a short sentence explaining why one should expect the bound on (R.I) claimed in line 422 to hold, etc. }

\noindent\textbf{Constraint Violation Analysis}. We take a violation decomposition,
\[
{\text{Violation}}(K) 
\;=\;
\displaystyle\sum_{k\,=\,1}^{K}\big(b-V_{g,1}^{k}(x_1)\big)\,+\,\normalfont\text{(V.II)} 
\]
where $\normalfont\text{(V.II)} = \sum_{k\,=\,1}^{K}\big({V_{g,1}^{k}(x_1)-V_{g,1}^{\pi^k}(x_1)}\big)$. We still begin with the policy improvement~\eqref{eq.closed} to have Lemma~\ref{lem.5}, a refined result from Lemma~\ref{lem.1}.
\begin{manuallemma}{5}[Policy Improvement: Refined Primal-Dual Mirror Descent Step]
	\label{lem.5}
	Let assumptions of Theorem~\ref{thm.main-full} hold. Then, 
	$\sum_{k\,=\,1}^{K}
	\big(V_{r,1}^{\pi^\star}(x_1) - V_{r,1}^{k}(x_1)\big) +Y \sum_{k\,=\,1}^{K}  \big(b- V_{g,1}^{k}(x_1)\big)
	\leq
	O(H^{2.5} \sqrt{T \log|\calA|})$ for any $ Y\in[0,\chi].$
\end{manuallemma}
Lemma~\ref{lem.5} separates the dual update $Y^k$ in the second total differences in Lemma~\ref{lem.1}. We prove Lemma~\ref{lem.5} by combining Lemma~\ref{lem.1} with the UCB optimism and a change of variable of $Y^k$ from the dual update~\eqref{eq.dual}. 

Similar to $\normalfont\text{(R.II)}$, we also have $\normalfont\text{(V.II)}  =-\sum_{k\,=\,1}^{K}\sum_{h\,=\,1}^H\iota_{g,h}^k(x_h^k,a_h^k)+ M_{g,H,2}^K$, where $M_{g,H,2}^K$ is a martingale. By adding $\normalfont\text{(V.II)}$ to the inequality in Lemma~\ref{lem.5} with multiplier $Y\geq 0$, and also adding $\normalfont\text{(R.II)}$ to it, then taking $Y=0$ if $\sum_{k\,=\,1}^{K}  \big(b- V_{g,1}^{\pi^k}(x_1)\big)\leq 0$; otherwise $Y=\chi$, w.p. $1-p$, we have, 
\[
 \big(V_{r,1}^{\pi^\star}(x_1) - V_{r,1}^{\pi'}(x_1) \big)+\chi \big[ b- V_{g,1}^{\pi'}(x_1)\big]_{+}
\; \leq\; O(\, dH^{2.5} \sqrt{T} \log\rbr{{dT}/{p}} /{K}\,)
\]
where $  V_{r,1}^{\pi'}(x_1) = \frac{1}{K}\sum_{k\,=\,1}^{K} V_{r,1}^{\pi^k}(x_1)$ and $  V_{g,1}^{\pi'}(x_1) = \frac{1}{K}\sum_{k\,=\,1}^{K} V_{g,1}^{\pi^k}(x_1)$ for a policy $\pi' $. Here, we bound $\Gamma_{h}^k+\Gamma_{\diamond,h}^k$ and $M_{\diamond,H,2}^K$ as done in Lemmas~\ref{lem.3} and~\ref{lem.4}.

Last, by strong duality, applying the constraint violation bound from Lemma~\ref{thm.violationgeneral} in Appendix~\ref{app.sec.opt} leads to
$\big[ b-V_{g,1}^{\pi'}(x_1)\big]_{+}\leq O( dH^{2.5} \sqrt{T} \log\rbr{{dT}/{p}} /(\chi K))$. This gives our desired violation bound.

\vspace*{-0.1cm}
\section{Further Results on Tabular Case}\label{sec.tabular}
\vspace*{-0.1cm}

The tabular $\text{\normalfont CMDP}(\calS,\calA,H,\mathbb{P},r,g)$ is a special case of Assumption~\ref{as.linearMDP} with $|\calS|<\infty$ and $|\calA|<\infty$. Let $d_1 = |\calS|^2|\calA|$ and $d_2=|\calS||\calA|$. We take the following feature maps $\psi(x,a,x') \in \mathbb{R}^{d_1}$, $\varphi(x,a)\in \mathbb{R}^{d_2}$, and parameter vectors,
\begin{equation}\label{eq.tabfea}
\begin{array}{rcl}
&& \!\!\!\! \!\!\!\! \!\!\!\! \!\!\!\! \psi(x,a,x') = \mathbf{e}_{(x,a,x')},\, \theta_h=\mathbb{P}_h(\,\cdot\,,\cdot\,,\cdot\,)
\\[0.2cm]
&& \!\!\!\! \!\!\!\! \!\!\!\! \!\!\!\! \varphi(x,a) = \mathbf{e}_{(x,a)},\, \theta_{r,h}=r_h(\,\cdot\,,\cdot\,) ,\,\theta_{g,h}=g_h(\,\cdot\,,\cdot\,)
\end{array}
\end{equation}
where $\mathbf{e}_{(x,a,x')}$ is a canonical basis of $\mathbb{R}^{d_1}$ associated with $(x,a,x')$ and $\theta_h = \mathbb{P}_h(\,\cdot\,,\cdot\,,\cdot\,)$ reads that for any $(x,a,x')\in\calS\times\calA\times\calS$, the $(x,a,x')$th entry of $\theta_h$ is $\mathbb{P}(x'\,|\,x,a)$; similarly we define $\mathbf{e}_{(x,a)}$, $\theta_{r,h}$, and $\theta_{g,h}$. We can verify that $\norm{\theta_h}\leq \sqrt{d_1}$, $\norm{\theta_{r,h}}\leq \sqrt{d_2}$, $\norm{\theta_{g,h}}\leq \sqrt{d_2}$, and for any $V$: $\calS\to[0,H]$ and any $(x,a)\in\calS\times\calA$, we have 
$\norm{ \sum_{x'\,\in\,\calS} \psi(x,a,x') V(x') }\leq\sqrt{|\calS|} H \leq \sqrt{d_1} H$.
Therefore, we take $d \DefinedAs\max\rbr{d_1,d_2} = |\calS|^2|\calA|$ in Assumption~\eqref{as.linearMDP} for the tabular case.

The proof of Theorem~\ref{thm.main-full} is generic, since it is ready to achieve sublinear regret and constraint violation bounds as long as the policy evaluation is sample-efficient, e.g., the UCB design of `optimism in the face of uncertainty.' In what follows, we introduce another efficient policy evaluation for line~10 of Algorithm~\ref{alg:OPDOP} in the tabular case. Let us first introduce some notation. For any $(h,k)\in[H]\times[K]$, any $(x,a,x')\in\calS\times\calA\times\calS$, and any $(x,a)\in\calS\times\calA$, we define two visitation counters $n_h^k(x,a,x')$ and $n_h^k(x,a)$ at step $h$ in episode $k$,
\begin{equation}\label{eq.counters}
\begin{array}{rcl}
\!\!\!\! \!\! n_h^k(x,a,x') &\!\!=\!\!& \displaystyle\sum_{\tau\,=\,1}^{k-1} \one\{(x,a,x') = (x_h^\tau,a_h^\tau,a_{h+1}^\tau) \}
\\[0.2cm]
\!\!\!\! \!\! n_h^k(x,a) &\!\!=\!\!& \displaystyle \sum_{\tau\,=\,1}^{k-1} \one\{(x,a) = (x_h^\tau,a_h^\tau) \}.
\end{array}
\end{equation}
This allows us to estimate reward function $r$, utility function $g$, and transition kernel $\mathbb{P}_h$ for episode $k$ by 
\begin{equation}\label{eq.rg}
\begin{array}{rcl}
\!\!\!\! \!\! \hat{r}_h^k(x,a) &\!\!=\!\!&\displaystyle \sum_{\tau\,=\,1}^{k-1} \frac{ \one\{ (x,a) = (x_h^\tau,a_h^\tau) \} r_h(x_h^\tau,a_h^\tau)}{n_h^k(x,a) +\lambda} 
\\[0.2cm]
\!\!\!\! \!\! \hat{g}_h^k(x,a) &\!\!=\!\!&\displaystyle \sum_{\tau\,=\,1}^{k-1} \frac{\one\{ (x,a) = (x_h^\tau,a_h^\tau) \} g_h(x_h^\tau,a_h^\tau)}{n_h^k(x,a) +\lambda}
\end{array}
\end{equation}
\vspace*{-0.2cm}
\begin{equation}\label{eq.P}
\hat{\mathbb{P}}_h^k (x'\,|\,x,a) 
\;=\;
\frac{n_h^k(x,a,x')}{n_h^k(x,a) + \lambda}
\end{equation}
for all $(x,a,x')\in\calS\times\calA\times\calS$, $(x,a)\in\calS\times\calA$
where $\lambda>0$ is the regularization parameter. Moreover, we introduce the bonus term $\Gamma_{h}^k$: $\calS\times\calA\to\mathbb{R}$,
$
\Gamma_h^k (x,a) \;=\; \beta \rbr{n_h^k(x,a)+\lambda}^{-1/2}
$
which adapts the counter-based bonus terms in the literature~\cite{azar2017minimax,jin2018q}, where $\beta>0$ is to be determined later. 

Using the estimated transition kernels $\{\hat{\mathbb{P}}_h^k\}_{h\,=\,1}^H$, the estimated reward/utility functions $\{\hat{r}_h^k,\hat{g}_h^k\}_{h\,=\,1}^H$, and the bonus terms $\{\Gamma_h^k\}_{h\,=\,1}^H$, we now can estimate the action-value function via line~7 of Algorithm~\ref{alg:tbandit}
for any $(x,a)\in\calS\times\calA$, where $\diamond = r$ or $g$. Thus, $V_{\diamond,h}^k(x) = \langle{Q_{\diamond,h}^k(x, \cdot)}, {\pi_h^k(\cdot\,|\,x)}\rangle_\calA$. We summarize the above procedure in Algorithm~\ref{alg:tbandit}. Using already estimated $\{Q_{r,h}^k(\cdot, \cdot),Q_{g,h}^k(\cdot, \cdot)\}_{h\,=\,1}^H$, we execute the policy improvement and the dual update in Algorithm~\ref{alg:OPDOP}.

Similar to Theorem~\ref{thm.main-full}, we provide theoretical guarantees in Theorem~\ref{thm.tabular}; see Appendix~\ref{app.further} for the proof. Theorem~\ref{thm.tabular} improves $(|\calS|,|\calA|)$ dependence in Theorem~\ref{thm.main-full} for the tabular case and also improves $|\calS|$ dependence in~\cite{efroni2020exploration}. It is worthy mentioning  our Algorithm~\ref{alg:OPDOP} is generic in handling an infinite state space.

\begin{theorem}[Tabular Case: Regret and Constraint Violation]\label{thm.tabular}
	Let Assumption~\ref{as.slater} hold and let Assumption~\ref{as.linearMDP} hold with feature maps~\eqref{eq.tabfea}.
	Fix $p\in\rbr{0,1}$. 
	In Algorithm~\ref{alg:OPDOP}, we set
	$\alpha={\sqrt{\log\abr{\calA}}}/(H^2K)$, $\beta = C_1 H\sqrt{|\calS| \log(|\calS||\calA|T/p)}$, $\eta=1/\sqrt{K}$, $\theta = {1}/K$, and $\lambda=1$ where $C_1$ is an absolute constant. Then, with probability $1-p$, the regret and the constraint violation in~\eqref{eq.regret} satisfy 
	\[
	\begin{array}{rcl}
	\text{\normalfont Regret}(K) &\!\!\leq\!\!&  \displaystyle C |\calS|\sqrt{|\calA| H^{5} T}\log\!\rbr{\!\tfrac{|\calS||\calA|T}{p}\!}
	\\[0.2cm]
	\text{\normalfont Violation}(K) &\!\!\leq\!\!&  \displaystyle C'|\calS|\sqrt{|\calA| H^{5}T} \log\!\rbr{\!\tfrac{|\calS||\calA|T}{p}\!}
	\end{array}
	\]
	where $C$ and $C'$ are absolute constants. 
\end{theorem}

\begin{algorithm}[H]
	\caption{Optimistic Policy Evaluation (OPE)}
	\label{alg:tbandit}
	%	\setstretch{1.3}
	\begin{algorithmic}[1]
		\State \textbf{Input}: $\{x_h^\tau,a_h^\tau,r_h(x_h^\tau,a_h^\tau), g_h(x_h^\tau,a_h^\tau)\}_{h,\tau\,=\,1}^{H,k}$.
		\State \textbf{Initialization}: Set $\{V_{r,H+1}^k,V_{g,H+1}^k\}$ be zero functions, and $\lambda=1$, $\beta= C_1 H\sqrt{|\calS| \log(|\calS||\calA|T/p)}$.
		\For{step $h=H,H-1,\cdots,1$} \Comment{$\diamond = r$, $g$}
		\State \!\!\!\! Compute counters $n_h^k(x,a,x')$ and $n_h^k(x,a)$ via~\eqref{eq.counters} for all $(x,a,x')\in\calS\times\calA\times\calS$ and $(x,a)\in\calS\times\calA$.
		\State  Estimate reward/utility functions $\hat r_h^k$, $\hat g_h^k$ via~\eqref{eq.rg} for all $(x,a)\in\calS\times\calA$.
		\State \!\!\!\! Estimate transition $\hat{\mathbb{P}}_h^k$ via~\eqref{eq.P} for all $(x,a,x')\in\calS\times\calA\times\calS$, and take bonus $\Gamma_h^k=\beta \rbr{n_h^k(x,a)+\lambda}^{-1/2}$ for all $(x,a)\in\calS\times\calA$.
		\State $Q_{\diamond,h}^k(\,\cdot\,,\,\cdot\,) \,\leftarrow\, \min\big(\,\hat\diamond_h^k(\,\cdot\,,\,\cdot\,) +\sum_{x'\,\in\,\calS} \hat{\mathbb{P}}_h(x'\,|\,\cdot,\cdot) V_{\diamond,h+1}^k(x')+ 2\Gamma_{h}^k(\,\cdot\,,\,\cdot\,),\, H-h+1\,\big)^+ $.
		%		\STATE $Q_{g,h}^k(\,\cdot\,,\,\cdot\,) \,\leftarrow\, \min\big(\,g_h^k(\,\cdot\,,\,\cdot\,) +\phi_{g,h}^k(\,\cdot\,,\,\cdot\,)^\top w_{g,h}^k+ \Gamma_{g,h}^k(\,\cdot\,,\,\cdot\,),\, H-h+1\,\big)^+ $.
		\State $V_{\diamond,h}^k(\,\cdot\,) \,\leftarrow\, \big\langle {Q_{\diamond,h}^k(\,\cdot\,,\,\cdot\,)},{\pi_h^k(\,\cdot\,\vert\, \cdot\,)}\big\rangle_\calA$.
		%		\STATE $V_{g,h}^k(\,\cdot\,) \,\leftarrow\, \big\langle{Q_{g,h}^k(\,\cdot\,,\,\cdot\,)},{\pi_h^k(\,\cdot\,\vert\, \cdot\,)}\big\rangle_\calA$.
		\EndFor
		\State \textbf{Return}: $\{Q_{r,h}^k(\,\cdot\,,\,\cdot\,),Q_{g,h}^k(\,\cdot\,,\,\cdot\,)\}_{h\,=\,1}^H$.
	\end{algorithmic}
\end{algorithm}

\section{Concluding Remarks}
\label{conclusion}

In this paper, we have developed a provably efficient safe reinforcement learning algorithm in the linear MDP setting. The algorithm extends the proximal policy optimization to the constrained MDP and incorporates the UCB exploration. We prove that the proposed algorithm obtains an $\tilde{O}(\sqrt{T})$ regret and an $\tilde{O}(\sqrt{T})$ constraint violation under mild regularity conditions where $T$ is the total number of steps taken by the algorithm. Moreover, our algorithm works in settings where reward/utility functions are given by bandit feedback. To the best of our knowledge, our algorithm is the first provably efficient online policy optimization algorithm for CMDP with safe exploration in the function approximation setting. We hope that our work provides a step towards a principled way to~design efficient safe reinforcement learning algorithms. 

% Acknowledgements should go at the end, before appendices and references

%\acks{We would like to acknowledge support for this project
%from the National Science Foundation ... }


\newpage
\bibliography{dd-bib}
\bibliographystyle{abbrv} % 

\newpage
\appendix

\section{Preliminaries}

Our analysis begins with decomposition of the regret given in~\eqref{eq.regret}. 
\begin{equation}\label{eq.regretde1}
{\text{Regret}}(K) \;=\; \underbrace{\sum_{k\,=\,1}^{K}\big(V_{r,1}^{\pi^\star}(x_1) - V_{r,1}^{k}(x_1)\big)}_{\normalfont\text{(R.I)} } \,+\, \underbrace{\sum_{k\,=\,1}^{K}\big(V_{r,1}^{k}(x_1) -V_{r,1}^{\pi^k}(x_1)\big)}_{\normalfont\text{(R.II)} }
\end{equation}
where we add and subtract the value $V_{r,1}^{k}(x_1)$ estimated from an optimistic policy evaluation by Algorithm~\ref{alg:LSVI}; the policy $\pi^\star$ in hindsight is the best policy in hindsight for problem~\eqref{eq.hindsight}. To bound the total regret~\eqref{eq.regretde1}, we would like to analyze~$\normalfont\text{(R.I)}$ and~$\normalfont\text{(R.II)}$ separately. 

First, we define the model prediction error for the reward as 
\begin{equation}\label{eq.mper}
\iota_{r,h}^k \;\DefinedAs\; r_h \,+\, \mathbb{P}_h V_{r,h+1}^k \,-\, Q_{r,h}^k
\end{equation}
for all $(k,h)\in[K]\times[H]$, which describes the prediction error in the Bellman equations~\eqref{eq.bellman} using $V_{r,h+1}^{k}$ instead of $V_{r,h+1}^{\pi^k}$. With this notation, we expand~$\normalfont\text{(R.I)}$ into 
\begin{equation}\label{eq.RI}
%\normalfont\text{(R.I)} \;=\; 
\sum_{k\,=\,1}^{K}\sum_{h\,=\,1}^{H} \mathbb{E}_{\pi^\star} \Big[ \big\langle{Q_{r,h}^k(x_h,\,\cdot\,)},{\pi_h^\star( \,\cdot\,\vert\,  x_h) - \pi_h^k(\,\cdot \,\vert \, x_h)}\big\rangle \Big]
\,+\,
\sum_{k\,=\,1}^{K}\sum_{h\,=\,1}^H \mathbb{E}_{\pi^\star}\Big[\iota_{r,h}^k(x_h,a_h) \Big]
\end{equation}
where the first double sum is linear in terms of the policy difference and the second one describes the total model prediction error. The above expansion is based on the standard performance difference lemma (see Lemma~\ref{lem.PDL}) and we provide a proof in~Section~\ref{subsec.RCI} for readers' convenience. Meanwhile, if we define the model prediction error for the utility as 
\begin{equation}\label{eq.mpeg}
\iota_{g,h}^k \;\DefinedAs\; g_h \,+\, \mathbb{P}_h V_{g,h+1}^k \,-\, Q_{g,h}^k
\end{equation}
then, similarly, we can expand $\sum_{k\,=\,1}^{K}\big(V_{g,1}^{\pi^\star}(x_1) - V_{g,1}^{k}(x_1)\big)$ into 
\begin{equation}\label{eq.gRI}
\sum_{k\,=\,1}^{K}\sum_{h\,=\,1}^{H} \mathbb{E}_{\pi^\star} \sbr{ \big\langle{Q_{g,h}^k(x_h,\,\cdot\,)},{\pi_h^\star(\,\cdot\,\vert\,  x_h) - \pi_h^k(\,\cdot\,\vert\, x_h)}\big\rangle } 
\,+\,
\sum_{k\,=\,1}^{K}\sum_{h\,=\,1}^H \mathbb{E}_{\pi^\star}\Big[\iota_{g,h}^k(x_h,a_h) \Big].
\end{equation}

To analyze the constraint violation, we also introduce an useful decomposition, 
\begin{equation}\label{eq.violation1}
{\text{Violation}}(K) 
\;=\;
\displaystyle\sum_{k\,=\,1}^{K}\big(b-V_{g,1}^{k}(x_1)\big)\,+\,\underbrace{\sum_{k\,=\,1}^{K}\big({V_{g,1}^{k}(x_1)-V_{g,1}^{\pi^k}(x_1)}\big)}_{\normalfont\text{(V.II)} }
\end{equation}
which the inserted value $V_{g,1}^k(x_1)$ is estimated from an optimistic policy evaluation by Algorithm~\ref{alg:LSVI}.

For notational simplicity, we introduce the underlying probability structure as follows. For any $(k,h) \in [K]\times[H]$, we define $\calF_{h,1}^k$ as a $\sigma$-algebra generated by state-action sequences, reward and utility functions,
\[
\{\rbr{x_{i}^\tau,a_{i}^\tau}\}_{(\tau,i)\,\in\,[k-1]\times[H]}
%\,\bigcup\,
%\{r^\tau,g^\tau\}_{\tau\,\in\,[k]}
\,\bigcup\,
\{(x_i^k,a_i^k)\}_{i\,\in\,[h]}.
\]
Similarly, we define $\calF_{h,2}^k$ as an $\sigma$-algebra generated by
\[
\{\rbr{x_{i}^\tau,a_{i}^\tau}\}_{(\tau,i)\,\in\,[k-1]\times[H]} 
%\,\bigcup\, \{r^\tau,g^\tau\}_{\tau\,\in\,[k]}
\,\bigcup\,
\{(x_i^k,a_i^k)\}_{i\,\in\,[h]}
\,\bigcup\,
\{x_{h+1}^k\}.
\]
Here, $x_{H+1}^k$ is a null state for any $k\in[K]$. A filtration is a sequence of $\sigma$-algebras $\{\calF_{h,m}^k\}_{(k,h,m)\,\in\,[K]\times[H]\times[2]}$ in terms of time index
\begin{equation}\label{eq.t}
t(k,h,m) \;\DefinedAs\; 2(k-1)H \,+\, 2(h-1) \,+\, m
\end{equation}
which holds that $\calF_{h,m}^k\subset\calF_{h',m'}^{k'}$ for any $t\leq t'$. 
The estimated reward/utility value functions, $V_{r,h}^k,V_{g,h}^k$, and the associated $Q$-functions, $Q_{r,h}^k,Q_{g,h}^k$ are $\calF_{1,1}^k$-measurable since they are obtained from previous $k-1$ historical trajectories. With these notations, we can expand~$\normalfont\text{(R.II)}$ in~\eqref{eq.regretde1} into 
\begin{equation}\label{eq.RII}
\normalfont\text{(R.II)}\; = \; -\sum_{k\,=\,1}^{K}\sum_{h\,=\,1}^H\iota_{r,h}^k(x_h^k,a_h^k)
\,+\,
M_{r,H,2}^K
\end{equation}
where $\{M_{r,h,m}^k\}_{(k,h,m)\in[K]\times[H]\times[2]}$ is a martingale adapted to the filtration $\{\calF_{h,m}^k\}_{(k,h,m)\in[K]\times[H]\times[2]}$ in terms of time index $t$.
Similarly, we have it for~$\normalfont\text{(V.II)}$,
\begin{equation}\label{eq.CII}
\normalfont\text{(V.II)}\; = \; -\sum_{k\,=\,1}^{K}\sum_{h\,=\,1}^H\iota_{g,h}^k(x_h^k,a_h^k)
\,+\,
M_{g,H,2}^K
\end{equation}
where $\{M_{g,h,m}^k\}_{(k,h,m)\in[K]\times[H]\times[2]}$ is a martingale adapted to the filtration $\{\calF_{h,m}^k\}_{(k,h,m)\in[K]\times[H]\times[2]}$ in terms of time index $t$. We prove~\eqref{eq.RII} in~Section~\ref{subsec.RCII} for completeness (also see~\cite[Lemma~4.2]{cai2019provably});~\eqref{eq.CII} is similar.

We recall two UCB bonus terms $\Gamma_{\diamond,h}^k \DefinedAs\beta ((\phi_{\diamond,h}^k)^\top (\Lambda_{\diamond,h}^k)^{-1}\phi_{\diamond,h}^k)^{1/2}$ and $\Gamma_h^k \DefinedAs \beta ((\varphi)^\top (\Lambda_h^k)^{-1}\varphi)^{1/2}$ in the action-value function estimation of Algorithm~\ref{alg:LSVI}. By the UCB argument, if we set $\lambda=1$ and $\beta=C_1 \sqrt{dH^2\log(dT/p)}$ where $C_1$ is an absolute constant, then for any $(k,h)\in[K]\times[H]$ and $(x,a)\in\calS\times\calA$, we have
\begin{equation}\label{eq.ucb-bandit}
-2(\Gamma_{h}^k+\Gamma_{\diamond,h}^k)(x,a) \;\leq\;{\iota_{\diamond,h}^k\rbr{x,a}} \;\leq\; 0
\end{equation}
with probability $1-p/2$ where the symbol $\diamond=r$ or $g$. We prove~\eqref{eq.ucb-bandit} in Section~\ref{subsec.ucb-bandit} for completeness.

In what follows we delve into the analysis of the regret and the constraint violation.

\section{Proof of Regret and Constraint Violation}
\label{ap.main1}

The goal is to prove that the regret and the constraint violation for Algorithm~\ref{alg:OPDOP} are sublinear in the total number of steps: $T\DefinedAs KH$, taken by the algorithm, where $K$ is the total number of episodes and $H$ is the horizon length. We recall that $|\calA|$ is the size of action space $\calA$ and $d$ is the feature map's dimension. We repeat Theorem~\ref{thm.main-full} here for readers' convenience.

\begin{manualtheorem}{1}[Linear Kernal MDP: Regret and Constraint Violation]
	\label{thm.full}
	Let Assumptions~\ref{as.slater} and~\ref{as.linearMDP} hold. 
	Fix $p\in\rbr{0,1}$. 
	We set
	$\alpha={\sqrt{\log\abr{\calA}}}/(H^2K)$, $\beta=C_1 \sqrt{dH^2\log\rbr{{dT}/{p}}}$, $\eta=1/\sqrt{K}$, $\theta = 1/K$, and $\lambda=1$ in Algorithm~\ref{alg:OPDOP} with the full-information setting, where $C_1$ is an absolute constant. Suppose $\log\abr{\calA}=O\rbr{d^2\log^2\rbr{{dT}/{p}}}$. Then, the regret and the constraint violation in~\eqref{eq.regret} satisfy 
	\[
	\text{\normalfont Regret}(K) \;\leq\; \displaystyle C\, dH^{2.5}\sqrt{T} \,\log\rbr{\frac{dT}{p}}
	\;\,\text{ and }\;\,
	\text{\normalfont Violation}(K) \;\leq\; \displaystyle C'\,dH^{2.5}\sqrt{T} \,{\log\rbr{\frac{dT}{p}}}
	\]
	%	\[
	%	\begin{array}{rcl}
	%	\frac{1}{K}\sum_{k=1}^{K}V_{g,1}^{\pi^k,k}\rbr{x_1} 
	%	&\!\!\geq\!\!& 
	%	b-C'\frac{H^4 d^3{\log^2\rbr{\frac{dT}{p}}}}{\sqrt{K}}.
	%	\end{array}
	%	\]
	with probability $1-p$ where $C$ and $C'$ are absolute constants. 
\end{manualtheorem}

We divide the proof into two parts for the regret bound and the constraint violation, respectively, in Section~\ref{subsec.rf} and Section~\ref{subsec.cf}. 

\subsection{Proof of Regret Bound}\label{subsec.rf}

Our analysis begins with a primal-dual mirror descent type analysis for the policy update in line~6 of Algorithm~\ref{alg:OPDOP}. In Lemma~\ref{lem.PDMD}, we present a key upper bound on the total differences of estimated values $ V_{r,1}^{k}(x_1)$ and $ V_{g,1}^{k}(x_1)$ given by Algorithm~\ref{alg:LSVI} to the optimal ones.

\begin{lemma}[Policy Improvement: Primal-Dual Mirror Descent Step]
	\label{lem.PDMD}
	Let Assumption~\ref{as.slater} and Assumption~\ref{as.linearMDP} hold. In Algorithm~\ref{alg:OPDOP}, if we set $\alpha={\sqrt{\log|\calA|}}/(H^2\sqrt{K})$ and $\theta = 1/K$, then
	\begin{equation}\label{eq.optcons}
	\begin{array}{rcl}
	&&\!\!\!\!\!\!\!\!\!\!  \displaystyle\sum_{k\,=\,1}^{K}
	\big(V_{r,1}^{\pi^\star}(x_1) - V_{r,1}^{k}(x_1)\big) \,+\,\sum_{k\,=\,1}^{K} Y^{k} \big(V_{g,1}^{\pi^\star} (x_1)- V_{g,1}^{k}(x_1)\big)
	\\[0.2cm]
	&\!\!\leq\!\!&
	\displaystyle C_2 H^{2.5}\sqrt{T\log|\calA|}
	\,+\, \sum_{k\,=\,1}^{K} \sum_{h\,=\,1}^H \mathbb{E}_{\pi^\star}\Big[\iota_{r,h}^{k}(x_h,a_h)\Big]\,+\, Y^{k} \sum_{k\,=\,1}^{K}\sum_{h\,=\,1}^H \mathbb{E}_{\pi^\star}\Big[\iota_{g,h}^{k}(x_h,a_h)\Big].
	\end{array}
	\end{equation}
	where $C_2$ is an absolute constant and $T=HK$.
\end{lemma}
\begin{proof}
	%	We first note that the left-hand side of~\eqref{eq.optcons} can be decomposed via equations~\eqref{eq.RI} and~\eqref{eq.gRI}.
	We recall that line~6 of Algorithm~\ref{alg:OPDOP} follows a solution $\pi^k$ to the following subproblem,
	\begin{equation}\label{eq.sub}
	\maximize_{\pi\,\in\,\Delta\rbr{\calA\vert \calS, H}} \; \sum_{h\,=\,1}^{H} \big\langle{ Q_{r,h}^{k-1}+Y^{k-1}Q_{g,h}^{k-1}},{\pi_h }\big\rangle\,-\, \frac{1}{\alpha}\sum_{h\,=\,1}^{H} D\big({\pi_h\, \vert \,\tilde{\pi}_h^{k-1}}\big)
	\end{equation}
	where we use the shorthand $\big\langle{Q_{r,h}^{k-1}+Y^{k-1}Q_{g,h}^{k-1}},{\pi_h}\big\rangle$ for $\big\langle{(Q_{r,h}^{k-1}+Y^{k-1}Q_{g,h}^{k-1})(x_h,\,\cdot\,)},{\pi_h(\,\cdot\,\vert\, x_h)}\big\rangle$ and the shorthand $D(\pi_h \,\vert\, \tilde\pi_h^{k-1})$ for $D(\pi_h(\,\cdot\,\vert\, x_h) \,\vert\, \tilde\pi_h^{k-1}(\,\cdot\,\vert\, x_h))$ if dependence on the state-action sequence $\{x_h,a_h\}_{h\,=\,1}^H$ is clear from context. We note that~\eqref{eq.sub} is in form of a mirror descent subproblem in Lemma~\ref{lem.pushback}. We can apply the pushback property with $x^\star=\pi_h^{k},y = \tilde{\pi}_h^{k-1}$ and $z=\pi_h^\star$,
	\[
	\begin{array}{rcl}
	&&\!\!\!\!\!\!\!\!\!\!  \displaystyle\sum_{h\,=\,1}^{H} \big\langle{ Q_{r,h}^{k-1}+Y^{k-1}Q_{g,h}^{k-1}},{\pi_h^{k}}\big\rangle
	\,-\,
	\frac{1}{\alpha} \sum_{h\,=\,1}^{H} D\big({\pi_h^{k}\,\vert\, \tilde{\pi}_h^{k-1}}\big)
	\\[0.2cm]
	&\!\!\geq\!\!&
	\displaystyle\sum_{h\,=\,1}^{H} \big\langle{ Q_{r,h}^{k-1}+Y^{k-1}Q_{g,h}^{k-1}},{\pi_h^\star}\big\rangle
	%	\\[0.2cm]
	%	&& 
	\,-\, 
	\dfrac{1}{\alpha} \displaystyle\sum_{h\,=\,1}^{H} D\big({\pi_h^\star\, \vert\, \tilde{\pi}_h^{k-1}}\big)
	\,+\,
	\frac{1}{\alpha} \sum_{h\,=\,1}^{H} D\big({\pi_h^\star \,\vert\, \pi_h^{k}}\big).
	\end{array}
	\]
	Equivalently, we write the above inequality as follows,
	\begin{equation}\label{eq.md}
	\begin{array}{rcl}
	&&\!\!\!\!\!\!\!\!\!\!  \displaystyle
	\sum_{h\,=\,1}^{H} \big\langle{ Q_{r,h}^{k-1}},{\pi_h^\star-\pi_h^{k-1}}\big\rangle
	\,+\, Y^{k-1}\sum_{h\,=\,1}^{H} \big\langle{Q_{g,h}^{k-1}},{\pi_h^\star-\pi_h^{k-1}}\big\rangle
	\\[0.2cm]
	&\!\!\leq\!\!&
	\displaystyle
	\sum_{h\,=\,1}^{H} \big\langle{ Q_{r,h}^{k-1}+Y^{k-1}Q_{g,h}^{k-1}},{\pi_h^{k}-\pi_h^{k-1}}\big\rangle
	\,-\,
	\frac{1}{\alpha} \sum_{h\,=\,1}^{H} D\big({\pi_h^{k}\,\vert\, \tilde{\pi}_h^{k-1}}\big)
	\\[0.2cm]
	&&
	\,+\, 
	\dfrac{1}{\alpha} \displaystyle\sum_{h\,=\,1}^{H} D\big({\pi_h^\star\, \vert\, \tilde{\pi}_h^{k-1}}\big)
	\,-\,
	\frac{1}{\alpha} \sum_{h\,=\,1}^{H} D\big({\pi_h^\star \,\vert\, \pi_h^{k}}\big).
	\end{array}
	\end{equation}
	By taking expectation $\mathbb{E}_{\pi^\star}$ on both sides of~\eqref{eq.md} over the random state-action sequence $\{(x_h,a_h)\}_{1}^H$ starting from $x_1$, and applying decompositions~\eqref{eq.RI} and~\eqref{eq.gRI}, we have
	\begin{equation}\label{eq.mdexp}
	\begin{array}{rcl}
	&&\!\!\!\!\!\!\!\!\!\!  \displaystyle
	\big(V_{r,1}^{\pi^\star}(x_1) - V_{r,1}^{k-1}(x_1)\big) \,+\,Y^{k-1} \big(V_{g,1}^{\pi^\star}(x_1) - V_{g,1}^{k-1}(x_1) \big)
	\\[0.2cm]
	&\!\!\leq\!\!&
	\displaystyle \sum_{h\,=\,1}^{H} \mathbb{E}_{\pi^\star} \Big[\big\langle{ Q_{r,h}^{k-1}+Y^{k-1}Q_{g,h}^{k-1}},{\pi_h^{k}-\pi_h^{k-1}}\big\rangle
	\Big] \,-\,
	\frac{1}{\alpha} \sum_{h\,=\,1}^{H} \mathbb{E}_{\pi^\star}\big[D\big({\pi_h^{k}\,\vert\, \tilde{\pi}_h^{k-1}}\big)\big]
	\\[0.2cm]
	&& 
	\,+\, 
	\dfrac{1}{\alpha} \displaystyle\sum_{h\,=\,1}^{H}  \mathbb{E}_{\pi^\star}\big[ D\big({\pi_h^\star\, \vert\, \tilde{\pi}_h^{k-1}}\big)
	- D\big({\pi_h^\star \,\vert\, \pi_h^{k}}\big)\big]
	\\[0.2cm]
	&& \displaystyle \,+\,\sum_{h\,=\,1}^H \mathbb{E}_{\pi^\star}\Big[\iota_{r,h}^{k-1}(x_h,a_h)\Big]\,+\,  Y^{k-1}\sum_{h\,=\,1}^H \mathbb{E}_{\pi^\star}\Big[\iota_{g,h}^{k-1}(x_h,a_h)\Big]
	\end{array}
	\end{equation}
	
	The rest is to bound the right-hand side of the above inequality~\eqref{eq.mdexp}.
	By the H\"older's inequality and the Pinsker's inequality, we first have 
	\[
	\begin{array}{rcl}
	&&\!\!\!\!\!\!\!\! \displaystyle\sum_{h\,=\,1}^H  \big\langle{Q_{r,h}^{k-1}+Y^{k-1}Q_{g,h}^{k-1}},{\pi_h^k-\pi_h^{k-1}}\big\rangle
	\,-\,
	\frac{1}{\alpha}\sum_{h\,=\,1}^H D\big({\pi_h^k\,\vert \,\tilde{\pi}_h^{k-1}}\big)
	\\[0.2cm]
	& \!\!=\!\! &  \displaystyle\sum_{h\,=\,1}^H  \big\langle{Q_{r,h}^{k-1}+Y^{k-1}Q_{g,h}^{k-1}},{\pi_h^k-\tilde{\pi}_h^{k-1}}\big\rangle
	\,-\,
	\frac{1}{\alpha}\sum_{h\,=\,1}^H D\big({\pi_h^k\,\vert\, \tilde{\pi}_h^{k-1}}\big) 
	\\[0.2cm]
	&&\displaystyle\,+\, 
	\sum_{h\,=\,1}^H  \big\langle{Q_{r,h}^{k-1}+Y^{k-1}Q_{g,h}^{k-1}},{\tilde{\pi}_h^{k-1}-\pi_h^{k-1}}\big\rangle
	\\[0.2cm]
	& \!\!\leq\!\! &
	\displaystyle\sum_{h\,=\,1}^H \rbr{ \big\|Q_{r,h}^{k-1}+Y^{k-1}Q_{g,h}^{k-1}\big\|_\infty\big\|\pi_h^k-\tilde{\pi}_h^{k-1}\big\|_1-\frac{1}{2\alpha} \big\|\pi_h^k-\tilde{\pi}_h^{k-1}\big\|_1^2} 
	\\[0.2cm]
	&& \displaystyle\,+\,
	\sum_{h\,=\,1}^H \big\|{Q_{r,h}^{k-1}}+Y^{k-1}Q_{g,h}^{k-1}\big\|_\infty\big\|\tilde{\pi}_h^{k-1}-\pi_h^{k-1}\big\|_1.
	\end{array}
	\]
	Then, using the square completion,
	\[
	\begin{array}{rcl}
	&&\displaystyle \big\|Q_{r,h}^{k-1}+Y^{k-1}Q_{g,h}^{k-1}\big\|_\infty\big\|\pi_h^k-\tilde{\pi}_h^{k-1}\big\|_1
	\,-\,
	\frac{1}{2\alpha} \big\|\pi_h^k-\tilde{\pi}_h^{k-1}\big\|_1^2 
	\\[0.2cm]
	&&\displaystyle\;=\;
	\,-\,
	\frac{1}{2\alpha}\rbr{\alpha\big\|Q_{r,h}^{k-1}+Y^{k-1}Q_{g,h}^{k-1}\big\|_\infty-\big\|\pi_h^k-\tilde{\pi}_h^{k-1}\big\|_1}^2 
	\,+\,
	\frac{\alpha}{2}\big\|Q_{r,h}^{k-1}+Y^{k-1}Q_{g,h}^{k-1}\big\|_\infty^2
	\\[0.2cm]
	&& \displaystyle \;\leq\; \frac{\alpha}{2}\big\|Q_{r,h}^{k-1}+Y^{k-1}Q_{g,h}^{k-1}\big\|_\infty^2
	\end{array}
	\]
	where we dropoff the first quadratic term for the inequality,
	and $\big\|\tilde{\pi}_h^{k-1}-\pi_h^{k-1}\big\|_1\leq \theta$, we have 
	\begin{equation}\label{eq.regd}
	\begin{array}{rcl}
	&&\!\!\!\!\!\!\!\! \displaystyle\sum_{h\,=\,1}^H \big\langle{Q_{r,h}^{k-1}+Y^{k-1}Q_{g,h}^{k-1}},{\pi_h^k-\pi_h^{k-1}}\big\rangle
	\,-\,
	\frac{1}{\alpha}\sum_{h\,=\,1}^H D\big({\pi_h^k\,\vert\, \tilde{\pi}_h^{k-1}}\big)
	\\[0.2cm]
	&\!\!\leq\!\!&
	\displaystyle\frac{\alpha}{2}\sum_{h\,=\,1}^H \big\|{Q_{r,h}^{k-1}+Y^{k-1}Q_{g,h}^{k-1}}\big\|_\infty^2 
	\,+\,
	\theta\sum_{h\,=\,1}^H \big\|{Q_{r,h}^{k-1}+Y^{k-1}Q_{g,h}^{k-1}}\big\|_\infty
	\\[0.2cm]
	&\!\!\leq\!\!&
	\displaystyle\frac{\alpha (1+\chi)^2H^3}{2}
	\,+\,
	\theta\rbr{1+\chi} H^2
	\end{array}
	\end{equation}
	where the last inequality is due to $\big\| Q_{r,h}^{k-1}\big\|_\infty\leq H$, a result from line~12 in Algorithm~\ref{alg:LSVI}, and $0\leq Y^{k-1}\leq \chi$. Taking the same expectation $\mathbb{E}_{\pi^\star}$ as previously on both sides of~\eqref{eq.regd} and substituting it into the left-hand side of~\eqref{eq.mdexp} yield,
	\begin{equation}\label{eq.mdexp1}
	\begin{array}{rcl}
	&&\!\!\!\!\!\!\!\!\!\!  \displaystyle
	\big(V_{r,1}^{\pi^\star}(x_1) - V_{r,1}^{k-1}(x_1)\big) \,+\,Y^{k-1} \big(V_{g,1}^{\pi^\star}(x_1) - V_{g,1}^{k-1}(x_1)\big)
	\\[0.2cm]
	&\!\!\leq\!\!&
	\displaystyle \frac{\alpha (1+\chi)^2H^3}{2}
	\,+\,
	\theta\rbr{1+\chi} H^2
	%	\sum_{h\,=\,1}^{H} \mathbb{E}_{\pi^\star} \big[\big\langle{ Q_{r,h}^{k-1}+Y^{k-1}Q_{g,h}^{k-1}},{\pi_h^{k}-\pi_h^{k-1}}\big\rangle
	%	\big] \,-\,
	%	\frac{1}{\alpha} \sum_{h\,=\,1}^{H} \mathbb{E}_{\pi^\star}\big[D\big({\pi_h^{k}\,\vert\, \tilde{\pi}_h^{k-1}}\big)\big]
	%	\\[0.2cm]
	%	&& 
	\,+\, 
	\dfrac{1}{\alpha} \displaystyle\sum_{h\,=\,1}^{H}  \mathbb{E}_{\pi^\star}\big[ D\big({\pi_h^\star\, \vert\, \tilde{\pi}_h^{k-1}}\big)
	- D\big({\pi_h^\star \,\vert\, \pi_h^{k}}\big)\big]
	\\[0.2cm]
	&& \displaystyle \,+\, \sum_{h\,=\,1}^H \mathbb{E}_{\pi^\star}\Big[\iota_{r,h}^{k-1}(x_h,a_h)\Big]\,+\,  Y^{k-1}\sum_{h\,=\,1}^H \mathbb{E}_{\pi^\star}\Big[\iota_{g,h}^{k-1}(x_h,a_h)\Big]
	\\[0.2cm]
	&\!\!\leq\!\!&
	\displaystyle \frac{\alpha (1+\chi)^2H^3}{2}
	\,+\,
	\theta\rbr{1+\chi} H^2
	\,+\,\frac{\theta H \log|\calA|}{\alpha}
	%	\sum_{h\,=\,1}^{H} \mathbb{E}_{\pi^\star} \big[\big\langle{ Q_{r,h}^{k-1}+Y^{k-1}Q_{g,h}^{k-1}},{\pi_h^{k}-\pi_h^{k-1}}\big\rangle
	%	\big] \,-\,
	%	\frac{1}{\alpha} \sum_{h\,=\,1}^{H} \mathbb{E}_{\pi^\star}\big[D\big({\pi_h^{k}\,\vert\, \tilde{\pi}_h^{k-1}}\big)\big]
	%	\\[0.2cm]
	%	&& 
	\,+\, 
	\dfrac{1}{\alpha} \displaystyle\sum_{h\,=\,1}^{H}  \mathbb{E}_{\pi^\star}\big[ D\big({\pi_h^\star\, \vert\, {\pi}_h^{k-1}}\big)
	- D\big({\pi_h^\star \,\vert\, \pi_h^{k}}\big)\big]
	\\[0.2cm]
	&& \displaystyle \,+\, \sum_{h\,=\,1}^H \mathbb{E}_{\pi^\star}\Big[\iota_{r,h}^{k-1}(x_h,a_h)\Big]\,+\,  Y^{k-1}\sum_{h\,=\,1}^H \mathbb{E}_{\pi^\star}\Big[\iota_{g,h}^{k-1}(x_h,a_h)\Big].
	\end{array}
	\end{equation}
	where in the second inequality 
	we note the fact that $D(\pi_h^\star\, \vert\, \tilde{\pi}_h^{k-1})-D(\pi_h^\star\, \vert\, \pi_h^{k-1})\leq \theta\log |\calA|$ from Lemma~\ref{lem.mix}.
	
	We note that $Y^0$ is initialized to be zero. 
	By taking a telescoping sum of both sides of~\eqref{eq.mdexp1} from $k=1$ to $k=K+1$ and shifting the index $k$ by one, we have
	\begin{equation}\label{eq.mdexp2}
	\begin{array}{rcl}
	&&\!\!\!\!\!\!\!\!\!\!  \displaystyle\sum_{k\,=\,1}^{K}
	\big(V_{r,1}^{\pi^\star}(x_1) - V_{r,1}^{k}(x_1)\big) \,+\,\sum_{k\,=\,1}^{K} Y^{k} \big(V_{g,1}^{\pi^\star} (x_1)- V_{g,1}^{k}(x_1)\big)
	\\[0.2cm]
	&\!\!\leq\!\!&
	\displaystyle \frac{\alpha (1+\chi)^2H^3(K+1)}{2}
	\,+\,
	\theta\rbr{1+\chi} H^2(K+1)
	\,+\,\frac{\theta H(K+1) \log|\calA|}{\alpha}
	%	\sum_{h\,=\,1}^{H} \mathbb{E}_{\pi^\star} \big[\big\langle{ Q_{r,h}^{k-1}+Y^{k-1}Q_{g,h}^{k-1}},{\pi_h^{k}-\pi_h^{k-1}}\big\rangle
	%	\big] \,-\,
	%	\frac{1}{\alpha} \sum_{h\,=\,1}^{H} \mathbb{E}_{\pi^\star}\big[D\big({\pi_h^{k}\,\vert\, \tilde{\pi}_h^{k-1}}\big)\big]
	%	\\[0.2cm]
	%	&& 
	\,+\, 
	\dfrac{H\log|\calA|}{\alpha} 
	\\[0.2cm]
	&& \displaystyle \,+\, \sum_{k\,=\,1}^{K} \sum_{h\,=\,1}^H \mathbb{E}_{\pi^\star}\Big[\iota_{r,h}^{k}(x_h,a_h)\Big]\,+\, Y^{k} \sum_{k\,=\,1}^{K}\sum_{h\,=\,1}^H \mathbb{E}_{\pi^\star}\Big[\iota_{g,h}^{k}(x_h,a_h)\Big].
	\end{array}
	\end{equation}
	where we ignore $-\alpha^{-1} \sum_{h\,=\,1}^{H} \mathbb{E}_{\pi^\star} [D(\pi_h^\star \,\vert\, \pi_h^{K+1})]$ and utilize
	\[
	D\rbr{\pi_h^\star\, \vert\, \pi_h^{0}} \;=\; \sum_{a\,\in\,\calA}^{} \pi_h^\star(a\,\vert\, x_h) \log\rbr{|\calA|\,\pi_h^\star (a\,\vert\, x_h)} \;\leq\; \log|\calA|
	\]
	where $\pi_h^{0}$ is uniform over $\calA$ and we ignore $\sum_{a\,\in\,\calA}^{} \pi_h^\star(a\,\vert\, x_h) \log\rbr{\pi_h^\star (a\,\vert\, x_h)}$ that is nonpositive.
	
	Finally, we take $\chi \DefinedAs H/\gamma$ and $\alpha$, $\theta$ in the lemma to complete the proof.
\end{proof}

By the dual update of Algorithm~\ref{alg:OPDOP}, we can simplify the result in Lemma~\ref{lem.PDMD} and return back to the regret~\eqref{eq.regretde1}.
\begin{lemma}
	\label{lem.reger}
	Let Assumption~\ref{as.slater} and Assumption~\ref{as.linearMDP} hold. In Algorithm~\ref{alg:OPDOP}, if we set $\alpha={\sqrt{\log|\calA|}}/(H^2\sqrt{K})$, $\eta=1/\sqrt{K}$, and $\theta = 1/K$, then
	\begin{equation}\label{eq.reg}
	{\normalfont\text{Regret}}(K) \;=\; C_3 H^{2.5}\sqrt{T \log|\calA|}\,+\, \sum_{k\,=\,1}^{K}\sum_{h\,=\,1}^H\rbr{ \mathbb{E}_{\pi^\star}\big[\iota_{r,h}^k(x_h,a_h)\big]-\iota_{r,h}^k(x_h^k,a_h^k)}\,+\, M_{r,H,2}^K
	\end{equation}
	where $C_3$ is an absolute constant.
\end{lemma}
\begin{proof}
	By the dual update in line~9 in Algorithm~\ref{alg:OPDOP}, we have
	\[
	\begin{array}{rcl}
	0 &\!\!\leq \!\!& \big(Y^{K+1}\big)^2
	\\[0.2cm]
	&\!\!=\!\!& \displaystyle \sum_{k=1}^{K+1} \Big(\big(Y^{k}\big)^2 -\big(Y^{k-1}  \big)^2\Big)
	\\[0.2cm]
	&\!\!=\!\!& \displaystyle \sum_{k=1}^{K+1} \rbr{
		\text{Proj}_{[\,0,\,\chi\,]} \big({Y^{k-1}+\eta({b-V_{g,1}^{k-1}(x_1)}) }\big) }^2- \big(Y^{k-1}\big)^2 
	\\[0.2cm]
	&\!\!\leq\!\!& \displaystyle \sum_{k=1}^{K+1} \big({Y^{k-1}+\eta ({b-V_{g,1}^{k-1}(x_1)}) }\big)^2- \big(Y^{k-1}\big)^2 
	\\[0.2cm]
	& \!\! \leq \!\! &  \displaystyle \sum_{k=1}^{K+1}  2 \eta Y^{k-1}\rbr{V_{g,1}^{\pi^\star}(x_1)-V_{g,1}^{k-1}(x_1)}+\eta^2\rbr{b-V_{g,1}^{k-1}(x_1)}^2.
	\end{array}
	\]
	where we use the feasibility of $\pi^\star$ in the last inequality. Since $Y^0=0$ and $|b-V_{g,1}^{k-1}(x_1)|\leq H$, the above inequality implies that
	\begin{equation}\label{eq.Yb}
	\displaystyle-\, \sum_{k=1}^{K}Y^{k}\rbr{V_{g,1}^{\pi^\star}(x_1)-V_{g,1}^{k}(x_1)} 
	\;\leq\; \sum_{k=1}^{K+1}  \frac{\eta}{2}\rbr{b-V_{g,1}^{k-1}(x_1)}^2
	\;\leq\; \frac{\eta H^2 (K+1)}{2}.
	\end{equation}
	By noting the UCB result~\eqref{eq.ucb-bandit} and $Y^k\geq 0$, the inequality~\eqref{eq.optcons} implies that 
	\[
	\displaystyle\sum_{k\,=\,1}^{K}
	\big(V_{r,1}^{\pi^\star}(x_1) - V_{r,1}^{k}(x_1)\big) \,+\,\sum_{k\,=\,1}^{K} Y^{k} \big(V_{g,1}^{\pi^\star} (x_1)- V_{g,1}^{k}(x_1)\big)
	\;\leq\;
	\displaystyle C_2 H^{2.5}\sqrt{T\log|\calA|}
	\,+\, \sum_{k\,=\,1}^{K} \sum_{h\,=\,1}^H \mathbb{E}_{\pi^\star}\Big[\iota_{r,h}^{k}(x_h,a_h)\Big].
	\]
	Then, we add~\eqref{eq.Yb} to the above inequality and take $\eta = 1/\sqrt{K}$,
	\begin{equation}\label{eq.regg}
	\begin{array}{rcl}
	\displaystyle\sum_{k\,=\,1}^{K}
	\big(V_{r,1}^{\pi^\star} (x_1)- V_{r,1}^{k}(x_1)\big) 
	&\!\!\leq\!\!&
	\displaystyle  C_3 H^{2.5}\sqrt{T\log|\calA|}
	\displaystyle \,+\, \sum_{k\,=\,1}^{K} \sum_{h\,=\,1}^H \mathbb{E}_{\pi^\star}\Big[\iota_{r,h}^{k}(x_h,a_h)\Big]
	\end{array}
	\end{equation}
	where $C_3$ is an absolute constant.
	Finally, we combine~\eqref{eq.RII} and~\eqref{eq.regg} to complete the proof.
\end{proof}

By Lemma~\ref{lem.reger}, the rest is to bound the last two terms in the right-hand side of~\eqref{eq.reg}. 
We next show two probability bounds for them in Lemma~\ref{lem.Lb} and Lemma~\ref{lem.M}, separately.

\begin{lemma}[Model Prediction Error Bound]
	\label{lem.Lb}
	Let Assumption~\ref{as.linearMDP} hold. Fix $p\in\rbr{0,1}$. If we set $\beta=C_1 \sqrt{dH^2\log\rbr{{dT}/{p}}}$ in Algorithm~\ref{alg:OPDOP}, then with probability $1-p/2$ it holds that
	\begin{equation}
	\label{eq.sumb}
	\sum_{k\,=\,1}^{K}\sum_{h\,=\,1}^H\rbr{\mathbb{E}_{\pi^\star}\big[\iota_{r,h}^k(x_h,a_h) \big]
		-\iota_{r,h}^k(x_h^k,a_h^k)}
	\;\leq\;
	4 C_1 \sqrt{2d^2 H^3 T \log\rbr{K+1}\log\rbr{\frac{dT}{p}}}
	\end{equation}
	where $C_1$ is an absolute constant and $T=HK$.
\end{lemma}

\begin{proof}
	By the UCB result~\eqref{eq.ucb-bandit}, with probability $1-p/2$ for any $(k,h)\in[K]\times[H]$ and $(x,a)\in\calS\times\calA$, we have
	\[
	-2(\Gamma_{h}^k+\Gamma_{r,h}^k)(x,a) \;\leq\;{\iota_{r,h}^k\rbr{x,a}} \;\leq\; 0.
	\]
	By the definition of $\iota_{r,h}^k(x,a)$, $|\iota_{r,h}^k(x,a)|\leq 2H$. Hence, it holds with probability $1-p/2$ that 
	\[
	-\iota_{r,h}^k(x,a) \;\leq\; 2 \min\!\rbr{H, (\Gamma_{h}^k+\Gamma_{r,h}^k)(x,a)  }
	\]
	for any $(k,h)\in[K]\times[H]$ and $(x,a)\in\calS\times\calA$. Therefore, we have
	\[
	\sum_{k\,=\,1}^{K}\sum_{h\,=\,1}^H \rbr{\mathbb{E}_{\pi^\star}\sbr{\iota_{r,h}^k(x_h,a_h) \,\vert\, x_1}  - \iota_{r,h}^k(x_h^k,a_h^k)}
	\; \leq \;2
	\sum_{k\,=\,1}^{K}\sum_{h\,=\,1}^H \min\rbr{H,(\Gamma_{h}^k+\Gamma_{r,h}^k)(x_h^k,a_h^k)}
	\]
	where $\Gamma_h^k(\cdot,\cdot) =\beta (\varphi(\cdot,\cdot)^\top (\Lambda_{h}^k)^{-1}\varphi(\cdot,\cdot))^{1/2}$ and $\Gamma_{r,h}^k(\cdot,\cdot) =\beta (\phi_{r,h}^k(\cdot,\cdot)^\top (\Lambda_{r,h}^k)^{-1}\phi_{r,h}^k(\cdot,\cdot))^{1/2}$. 
	Application of the Cauchy-Schwartz inequality shows that 
	\begin{equation}\label{eq.sch}
	\begin{array}{rcl}
	&&\!\!\!\! \!\!\!\! \displaystyle\sum_{k\,=\,1}^{K}\sum_{h\,=\,1}^H \min\rbr{H,(\Gamma_{h}^k+\Gamma_{r,h}^k)(x_h^k,a_h^k)} 
	\\[0.2cm]
	&\!\!\leq\!\!&\displaystyle
	\beta\sum_{k\,=\,1}^{K}\sum_{h\,=\,1}^H \min\rbr{H/\beta,\rbr{\varphi(x_h^k,a_h^k)^\top (\Lambda_{h}^k)^{-1}\varphi(x_h^k,a_h^k)}^{1/2}+\rbr{\phi_{r,h}^k(x_h^k,a_h^k)^\top (\Lambda_{r,h}^k)^{-1}\phi_{r,h}^k(x_h^k,a_h^k)}^{1/2}}
	\end{array}
	\end{equation}
	Since we take $\beta=C_1 \sqrt{dH^2\log\rbr{{dT}/{p}}}$ with $C_1>1$, we have $H/\beta\leq 1$. The rest is to apply 
	Lemma~\ref{lem.bdsums}. First, for any $h\in[H]$ it holds that 
	\[
	\sum_{k\,=\,1}^{K} \phi_{r,h}^k\big({x_h^k,a_h^k}\big)^\top\big({\Lambda_{r,h}^k}\big)^{-1}\phi_{r,h}^k\big({x_h^k,a_h^k}\big)
	\;\leq\;
	2\log\rbr{\frac{\det\big({\Lambda_{r,h}^{K+1}}\big)}{\det\big({\Lambda_{r,h}^1}\big)}}.
	\]
	Due to $\Vert{\phi_{r,h}^k}\Vert\leq \sqrt{d} H$ in Assumption~\ref{as.linearMDP} and $\Lambda_{r,h}^1=\lambda I$ in Algorithm~\ref{alg:LSVI}, it is clear that for any $h\in[H]$,
	\[
	\Lambda_{r,h}^{K+1} 
	\;=\; 
	\sum_{k\,=\,1}^{K} \phi_{r,h}^k\big({x_h^k,a_h^k}\big)\phi_{r,h}^k\big({x_h^k,a_h^k}\big)^\top\,+\,\lambda I
	\;\preceq\;
	(dH^2K+\lambda)I.
	\] 
	Therefore, we have
	\[
	\log\rbr{\frac{\det\big({\Lambda_{r,h}^{K+1}}\big)}{\det\big({\Lambda_{r,h}^1}\big)}} 
	\;\leq\; 
	\log\rbr{\frac{\det\big({(dH^2K+\lambda)I}\big)}{\det(\lambda I)}} 
	\;\leq\;
	d \log\rbr{\frac{dH^2K+\lambda}{\lambda}}.
	\]
	Therefore,
	\begin{equation}\label{eq.ub1}
	\sum_{k\,=\,1}^{K} \phi_{r,h}^k\big({x_h^k,a_h^k}\big)^\top\big({\Lambda_{r,h}^k}\big)^{-1}\phi_{r,h}^k\big({x_h^k,a_h^k}\big)
	\;\leq\;
	2d \log\rbr{\frac{dH^2K+\lambda}{\lambda}}.
	\end{equation}
	Similarly, we can show that
	\begin{equation}\label{eq.ub2}
	\sum_{k\,=\,1}^{K} \varphi\big({x_h^k,a_h^k}\big)^\top\big({\Lambda_{h}^k}\big)^{-1}\varphi\big({x_h^k,a_h^k}\big)
	\;\leq\;
	2d \log\rbr{\frac{dK+\lambda}{\lambda}}.
	\end{equation}
	
	Applying the above inequalities~\eqref{eq.ub1} and~\eqref{eq.ub2} to~\eqref{eq.sch} leads to
	\[
	\begin{array}{rcl}
	&&\!\!\!\! \!\!\!\! \displaystyle\sum_{k\,=\,1}^{K}\sum_{h\,=\,1}^H \min\rbr{H,(\Gamma_{h}^k+\Gamma_{r,h}^k)(x_h^k,a_h^k)} 
	\\[0.5cm]
	&\!\!\leq\!\!&\displaystyle
	\beta\sum_{h\,=\,1}^H \min\rbr{K, \sum_{k\,=\,1}^{K} \rbr{\varphi(x_h^k,a_h^k)^\top (\Lambda_{h}^k)^{-1}\varphi(x_h^k,a_h^k)}^{1/2}+\rbr{\phi_{r,h}^k(x_h^k,a_h^k)^\top (\Lambda_{r,h}^k)^{-1}\phi_{r,h}^k(x_h^k,a_h^k)}^{1/2}}
	\\[0.5cm]
	&\!\!\leq\!\!&\displaystyle
	\beta\sum_{h\,=\,1}^H \rbr{ \rbr{K\sum_{k\,=\,1}^{K}\varphi(x_h^k,a_h^k)^\top (\Lambda_{h}^k)^{-1}\varphi(x_h^k,a_h^k)}^{1/2}+\rbr{K\sum_{k\,=\,1}^{K}\phi_{r,h}^k(x_h^k,a_h^k)^\top (\Lambda_{r,h}^k)^{-1}\phi_{r,h}^k(x_h^k,a_h^k)}^{1/2}}
	\\[0.5cm]
	&\!\!\leq\!\!&\displaystyle
	\beta\sum_{h\,=\,1}^H \sqrt{K} \rbr{\rbr{2d \log\rbr{\frac{dK+\lambda}{\lambda}}}^{1/2}+\rbr{2d \log\rbr{\frac{dH^2K+\lambda}{\lambda}}}^{1/2}}
	\end{array}
	\]
	
	Finally, we set $\beta=C_1 \sqrt{dH^2\log\rbr{{dT}/{p}}}$ and $\lambda=1$ to obtain~\eqref{eq.sumb}.
\end{proof}

\begin{lemma}[Matingale Bound]
	\label{lem.M}
	Fix $p\in\rbr{0,1}$. In Algorithm~\ref{alg:OPDOP}, it holds with probability $1-p/2$ that
	\begin{equation}\label{eq.M}		
	\abr{M_{r,H,2}^K} \;\leq\; 4\sqrt{H^2T \log\rbr{\frac{4}{p}}}
	\end{equation}
	where $T= HK$.
\end{lemma}
\begin{proof}
	In the verification of~\eqref{eq.RII} (see Section~\ref{subsec.RCII}), we introduce the following matingale, 
	\[
	M_{r,H,2}^K \;=\; \sum_{k\,=\,1}^K\sum_{h\,=\,1}^H\big({D_{r,h,1}^k+D_{r,h,2}^k}\big)
	\]
	where  
	\[
	\begin{array}{rcl}
	D_{r,h,1}^k &\!\!=\!\!& \rbr{\calI_h^k \big(Q_{r,h}^k-Q_{r,h}^{\pi^k,k}\big) }(x_h^k) - \rbr{Q_{r,h}^k-Q_{r,h}^{\pi^k,k}}\big(x_h^k,a_h^k\big)
	\\[0.2cm]
	D_{r,h,2}^k &\!\!=\!\!& \rbr{\mathbb{P}_h V_{r,h+1}^k-\mathbb{P}_h V_{r,h+1}^{\pi^k,k}}\big(x_h^k,a_h^k\big)-\rbr{V_{r,h+1}^k- V_{r,h+1}^{\pi^k,k}}\big(x_{h+1}^k\big)
	\end{array}
	\]
	where $\rbr{\calI_{h}^k f}\rbr{x} \DefinedAs \inner{f(x,\cdot)}{\pi_h^k(\cdot\vert x)}$.
	
	Due to the truncation in line~11 of Algorithm~\ref{alg:LSVI}, we know that $Q_{r,h}^k,Q_{r,h}^{\pi^k},V_{r,h+1}^k,V_{r,h+1}^{\pi^k} \in[0,H]$. This shows that $|{D_{r,h,1}^k}|,|{D_{r,h,2}^k}|\leq 2H$ for all $(k,h)\in[K]\times[H]$. Application of the Azuma-Hoeffding inequality yields,
	\[
	P\big(\,\big|M_{r,H,2}^K\big| \geq s\,\big) \;\leq\;2\exp\rbr{\frac{-s^2}{16H^2 T}}.
	\]
	For $p\in(0,1)$, if we set $s=4H\sqrt{T\log\rbr{4/p}}$, then the inequality~\eqref{eq.M} holds with probability at least $1-p/2$.
	%	\[
	%	\abr{M_{r,K,H,2}} \;\leq\;\sqrt{16H^2T \log\rbr{\frac{4}{p}}}
	%	\]
	%	where $T=HK$.
\end{proof}

We now are ready to show the desired regret bound. Applying~\eqref{eq.sumb} and~\eqref{eq.M} to the right-hand side of the inequality~\eqref{eq.reg}, we have
\[
{\text{Regret}}(K) 
\;\leq\;
C_3 H^{2.5} \sqrt{T \log|\calA|}
\,+\,
2 C_1 \sqrt{2d^2 H^3 T \log\rbr{K+1}\log\rbr{\frac{dT}{p}}}
\,+\,
4\sqrt{ H^2T \log\rbr{\frac{4}{p}}}
\]
with probability $1-p$ where $C_1,C_3$ are absolute constants. Then, with probability $1-p$ it holds that
\[
{\text{Regret}}(K)  
\;\leq\;
C d H^{2.5}\sqrt{T} \log\rbr{\frac{dT}{p}}
\]
where $C$ is an absolute constant.

\subsection{Proof of Constraint Violation}\label{subsec.cf}

In Lemma~\ref{lem.PDMD}, we have provided an useful upper bound on the total differences that are weighted by the dual update $Y^k$. To extract the constraint violation, we first refine Lemma~\ref{lem.PDMD} as follows.

\begin{lemma}[[Policy Improvement: Refined Primal-Dual Mirror Descent Step]
	\label{lem.C1}
	Let Assumptions~\ref{as.slater} and~\ref{as.linearMDP} hold. In Algorithm~\ref{alg:OPDOP}, if we set $\alpha={\sqrt{\log|\calA|}}/(H^2\sqrt{K})$, $\theta = 1/K$, and $\eta=1/\sqrt{K}$, then Then, for any $Y\in[0,\chi]$,
	\begin{equation}\label{eq.mdexp3}
	\displaystyle\sum_{k\,=\,1}^{K}
	\big(V_{r,1}^{\pi^\star}(x_1) - V_{r,1}^{k}(x_1)\big) \,+\,Y \sum_{k\,=\,1}^{K}  \big(b- V_{g,1}^{k}(x_1)\big)
	\;\leq\;
	\displaystyle C_4 H^{2.5} \sqrt{T \log|\calA|}
	\end{equation}
	where $C_4$ is an absolute constant, $T=HK$, and $\chi \DefinedAs H/\gamma$.
\end{lemma}
\begin{proof}
	By the dual update in line~9 in Algorithm~\ref{alg:OPDOP}, for any $Y\in[0,\chi]$ we have
	\[
	\begin{array}{rcl}
	|Y^{k+1} -Y |^2 & \!\!= \!\!& \Big\lvert
	\text{Proj}_{[\,0,\,\chi\,]} \big({Y^{k}+\eta({b-V_{g,1}^{k}(x_1)}) }\big)-\text{Proj}_{[\,0,\,\chi\,]}  (Y) \Big\rvert^2
	\\[0.2cm]
	& \!\!\leq \!\!& \big\lvert Y^{k}+\eta({b-V_{g,1}^{k}(x_1)})-Y\big\rvert^2
	\\[0.2cm]
	& \!\!\leq \!\!& \big(Y^{k}-Y\big)^2 \,+\, 2\eta \big({b-V_{g,1}^{k}(x_1)}\big)  \big(Y^{k}-Y\big) \,+\, \eta^2 H^2
	%		\big\lvert Y^{k}+\eta({b-V_{g,1}^{k}(x_1^{k})})-Y\big\rvert^2
	\end{array}
	\]
	where we apply the non-expansiveness of projection in the first inequality and $|b-V_{g,1}^{k}(x_1)|\leq H$ for the last inequality. By summing the above inequality from $k=1$ to $k=K$, we have
	\[
	0\;\leq\;|Y^{K+1} -Y |^2 \;=\;|Y^{1} -Y |^2 \,+\, 2\eta \sum_{k\,=\,1}^{K}\big({b-V_{g,1}^{k}(x_1)}\big)  \big(Y^{k}-Y\big) \,+\, \eta^2 H^2 K
	\]
	which implies that
	\[
	\sum_{k\,=\,1}^{K}\big({b-V_{g,1}^{k}(x_1)}\big)  \big(Y-Y^{k}\big) \;\leq\; \frac{1}{2\eta}|Y^{1} -Y |^2 \,+\, \frac{\eta}{2} H^2 K.
	\]
	By adding the above inequality to~\eqref{eq.mdexp2} in Lemma~\ref{lem.PDMD} and noting that $V_{g,1}^{\pi^\star,k} (x_1)\geq b$ and the UCB result~\eqref{eq.ucb-bandit}, we have
	\[
	\begin{array}{rcl}
	&&\!\!\!\!\!\!\!\!\!\!  \displaystyle\sum_{k\,=\,1}^{K}
	\big(V_{r,1}^{\pi^\star}(x_1) - V_{r,1}^{k}(x_1)\big) \,+\,Y\sum_{k\,=\,1}^{K}  \big(b- V_{g,1}^{k}(x_1)\big)
	\\[0.2cm]
	&\!\!\leq\!\!&
	\displaystyle \frac{\alpha (1+\chi)^2H^3(K+1)}{2}
	\,+\,
	\theta\rbr{1+\chi} H^2(K+1)
	\,+\,\frac{\theta H(K+1) \log|\calA|}{\alpha}
	\,+\, 
	\dfrac{H\log|\calA|}{\alpha} 
	\\[0.2cm]
	&& \displaystyle \,+\,\frac{1}{2\eta}|Y^{1} -Y |^2 \,+\, \frac{\eta}{2} H^2 K. 
	\end{array}
	\]
	By taking $\chi = H/\gamma$, and $\alpha$, $\theta$, $\eta$ in the lemma, we complete the proof.
\end{proof}

According to Lemma~\ref{lem.C1}, we can multiply~\eqref{eq.CII} by $Y$ and add it, together with~\eqref{eq.RII}, to~\eqref{eq.mdexp3},
\begin{equation}\label{eq.mdexp4}
\begin{array}{rcl}
&&\!\!\!\!\!\!\!\!\!\!  \displaystyle\sum_{k\,=\,1}^{K}
\big(V_{r,1}^{\pi^\star}(x_1) - V_{r,1}^{\pi^k}(x_1)\big) \,+\,Y\sum_{k\,=\,1}^{K}  \big(b- V_{g,1}^{\pi^k}(x_1)\big)
\\[0.2cm]
&\!\!\leq\!\!&
\displaystyle C_4 H^{2.5} \sqrt{T \log|\calA|}
\,-\,\sum_{k\,=\,1}^{K}\sum_{h\,=\,1}^H\iota_{r,h}^k(x_h^k,a_h^k)\,-\,Y \sum_{k\,=\,1}^{K}\sum_{h\,=\,1}^H\iota_{g,h}^k(x_h^k,a_h^k)
\,+\,
M_{r,H,2}^K
\,+\,Y
M_{g,H,2}^K.
\end{array}
\end{equation}

We now are ready to show the desired constraint violation bound. We note that there exists a policy $\pi' $ such that $  V_{r,1}^{\pi'}(x_1) = \frac{1}{K}\sum_{k\,=\,1}^{K} V_{r,1}^{\pi^k}(x_1)$ and $  V_{g,1}^{\pi'}(x_1) = \frac{1}{K}\sum_{k\,=\,1}^{K} V_{g,1}^{\pi^k}(x_1)$. By the occupancy measure method~\cite{altman1999constrained}, $V_{r,1}^{\pi^k}(x_1)$ and $V_{g,1}^{\pi^k}(x_1)$ are linear in terms of an occupancy measure induced by policy $\pi^k$ and initial state $x_1$. Thus, an average of $K$ occupancy measures is still an occupancy measure that produces policy $\pi'$ with values $V_{r,1}^{\pi'}(x_1)$ and $V_{g,1}^{\pi'}(x_1)$. 
Particularly, we take $Y=0$ when $\sum_{k\,=\,1}^{K}  \big(b- V_{g,1}^{\pi^k}(x_1)\big)<0$; otherwise $Y =\chi$. 
Therefore, we have
\begin{equation}\label{eq.C0}
\begin{array}{rcl}
&&\!\!\!\!\!\!\!\!\!\!  \displaystyle
V_{r,1}^{\pi^\star}(x_1) - \frac{1}{K}\sum_{k\,=\,1}^{K}V_{r,1}^{\pi^k}(x_1) \,+\,\chi \left[ b-\frac{1}{K} \sum_{k\,=\,1}^{K}V_{g,1}^{\pi^k}(x_1)\right]_{+}
\\[0.2cm]
&&\!\!\!\!\!\!\!\!\!\!  \displaystyle
\;=\;V_{r,1}^{\pi^\star}(x_1) -V_{r,1}^{\pi'}(x_1) \,+\,\chi \left[ b-V_{g,1}^{\pi'}(x_1)\right]_{+}
\\[0.2cm]
&\!\!\leq\!\!&
\displaystyle \frac{C_4 H^{2.5} \sqrt{T \log|\calA|}}{K}
\,-\,\frac{1}{K}\sum_{k\,=\,1}^{K}\sum_{h\,=\,1}^H \iota_{r,h}^k(x_h^k,a_h^k)\,-\,\frac{\chi}{K}\sum_{k\,=\,1}^{K}\sum_{h\,=\,1}^H\iota_{g,h}^k(x_h^k,a_h^k)
\\[0.2cm]
&&\displaystyle
\,+\,
\frac{1}{K}M_{r,H,2}^K
\,+\,
\frac{\chi }{K}\big| M_{g,H,2}^K\big|
\\[0.2cm]
&\!\!\leq\!\!&
\displaystyle \frac{C_4 H^{2.5} \sqrt{T \log|\calA|}}{K}
\,+\,\frac{1}{K}\sum_{k\,=\,1}^{K}\sum_{h\,=\,1}^H (\Gamma_{h}^k+\Gamma_{r,h}^k)(x_h^k,a_h^k)\,+\,\frac{\chi}{K}\sum_{k\,=\,1}^{K}\sum_{h\,=\,1}^H(\Gamma_{h}^k+\Gamma_{g,h}^k)(x_h^k,a_h^k)
\\[0.2cm]
&&\displaystyle
\,+\,
\frac{1}{K}M_{r,H,2}^K
\,+\,
\frac{\chi }{K}\big| M_{g,H,2}^K\big|
\end{array}
\end{equation}
where we apply the UCB result~\eqref{eq.ucb-bandit} for the last inequality.

Finally, we recall two immediate results of Lemma~\ref{lem.Lb} and Lemma~\ref{lem.M}. 
Fix $p\in\rbr{0,1}$, the proof of Lemma~\ref{lem.Lb} also shows that with probability $1-p/2$,
\begin{equation}\label{eq.bb}
\sum_{k\,=\,1}^{K}\sum_{h\,=\,1}^H (\Gamma_{h}^k+\Gamma_{\diamond,h}^k)\big(x_h^k,a_h^k\big)
\;\leq\;
C_1 \sqrt{2d^2 H^3 T \log\rbr{K+1}\log\rbr{\frac{dT}{p}}}
\end{equation}
and the proof of Lemma~\ref{lem.M} shows that with probability $1-p/2$,
\[
\abr{M_{g,H,2}^K} \;\leq\; 4\sqrt{H^2T \log\rbr{\frac{4}{p}}}.
\]
If we take $\log|\calA| = O(d^2\log^2(dT/p))$,~\eqref{eq.C0} implies that with probability $1-p$ we have
\[
V_{r,1}^{\pi^\star}(x_1) -V_{r,1}^{\pi'}(x_1) \,+\,\chi \left[ b-V_{g,1}^{\pi'}(x_1)\right]_{+}
\;\leq\; C_5\,d H^{2.5} \sqrt{T}\, {\log\rbr{\frac{dT}{p}}}.
\]
where $C_5$ is an absolute constant.
Finally, by noting our choice of $\chi\geq 2 Y^\star$, we can apply Lemma~\ref{thm.violationgeneral} to conclude that 
\[
\text{Violation}(K) 
\;\leq\;
C^{'} d H^{2.5} \sqrt{T}\, {\log\rbr{\frac{dT}{p}}}.
\]
with probability $1-p$, where $C^{'}$ is an absolute constant.

\section{Further Results on Tabular Case}\label{ap.CMDP-t}

As mentioned, Assumption~\ref{as.linearMDP} includes a tabular $\text{\normalfont CMDP}(\calS,\calA,H,\mathbb{P},r,g)$ as a special case with $|\calS|<\infty$ and $|\calA|<\infty$. We take the following feature maps and parameter vectors,
\begin{subequations}\label{eq.tabular}
	\begin{equation}\label{eq.tabulara}
	d_1 \;=\; |\calS|^2|\calA|,\; \psi(x,a,x') \;=\; \mathbf{e}_{(x,a,x')}\,\in\,\mathbb{R}^{d_1},\; \theta_h \;=\; \mathbb{P}_h(\,\cdot\,,\cdot\,,\cdot\,)\,\in\,\mathbb{R}^{d_1}
	\end{equation}
	\begin{equation}\label{eq.tabularb}
	d_2 \;=\; |\calS||\calA|,\; \varphi(x,a) \;=\; \mathbf{e}_{(x,a)}\,\in\,\mathbb{R}^{d_2},\; \theta_{r,h}\;=\;r_h(\,\cdot\,,\cdot\,)\,\in\,\mathbb{R}^{d_2} ,\;\theta_{g,h} \;=\;g_h(\,\cdot\,,\cdot\,)\,\in\,\mathbb{R}^{d_2}.
	\end{equation}
\end{subequations}
where $\mathbf{e}_{(x,a,x')}$ is a canonical basis of $\mathbb{R}^{d_1}$ associated with $(x,a,x')$ and $\theta_h \;=\; \mathbb{P}_h(\,\cdot\,,\cdot\,,\cdot\,)$ reads that for any $(x,a,x')\in\calS\times\calA\times\calS$, the $(x,a,x')$th entry of $\theta_h$ is $\mathbb{P}(x'\,|\,x,a)$; similarly we define $\mathbf{e}_{(x,a)}$, $\theta_{r,h}$, and $\theta_{g,h}$. Thus, we can see that 
\[
\mathbb{P}_h\rbr{x'\,\vert\, x,a} \;=\; \langle\psi\rbr{x,a,x'},\theta_h\rangle, \; \text{ for any } (x,a,x') \,\in\, \calS\times\calA\times\calS
%	\;\;
%	r_h^k(x,a) \;=\; \langle{\phi\rbr{x,a}},{\theta_{r,h}^k}\rangle,
%	\;\; \text{and} \;\; 
%	g_h^k(x,a) \;=\; \langle{\phi\rbr{x,a}},{\theta_{g,h}^k}\rangle.
\]
\[
r_h(x,a) \;=\; \langle\varphi(x,a), \theta_{r,h}\rangle \; \text{ and } \; g_h(x,a) \;=\; \langle\varphi(x,a), \theta_{g,h}\rangle, \; \text{ for any } (x,a) \,\in\, \calS\times\calA.
\]
We can also verify that 
\[
\norm{\theta_h} \;=\; \rbr{\sum_{(x,a,x')} \abr{\mathbb{P}_h(x'\,|\,x,a)}^2}^{1/2}\;\leq\; \sqrt{|\calS|^2|\calA|} \;=\;\sqrt{d_1}
\]
\[
\norm{\theta_{r,h}} \;=\; \rbr{\sum_{(x,a)} \rbr{r_h(x,a)}^2 }^{1/2} \;\leq\;\sqrt{|\calS||\calA|} \;=\;\sqrt{d_2}
\]
\[
\norm{\theta_{g,h}} \;=\; \rbr{\sum_{(x,a)} \rbr{g_h(x,a)}^2 }^{1/2} \;\leq\;\sqrt{|\calS||\calA|} \;=\;\sqrt{d_2}
\]
and for any $V$: $\calS\to[0,H]$ and any $(x,a)\in\calS\times\calA$, we have 
\[
\norm{ \sum_{x'\,\in\,\calS} \psi(x,a,x') V(x') }\;=\;\rbr{ \sum_{x'\,\in\,\calS} \rbr{V(x')}^2}^{1/2}\;\leq\;\sqrt{|\calS|}\, H \;\leq\; \sqrt{d_1}\, H.
\]
Therefore, the tabular CMDP is a special case of Assumption~\eqref{as.linearMDP} with $d \DefinedAs\max\rbr{d_1,d_2} = |\calS|^2|\calA|$. 

%We remark the difference between the full-information setting and the bandit setting for the tabular CMDP. In the full-information setting, we observe full vectors $\theta_{r,h}^k$ and $\theta_{g,h}^k$ in~\eqref{eq.tabularb}. Even though, they can vary adversarially over episodes. However, in the bandit setting, we only observe one coordinate of $\theta_{r,h}^k$ and $\theta_{g,h}^k$ at a particular state-action pair. In this setting, we do not allow any adversarial changes in $\theta_{r,h}^k$ and $\theta_{g,h}^k$ and we may use their shorthands $\theta_{r,h}$ and $\theta_{g,h}$.

\subsection{Tabular Case of Algorithm~\ref{alg:OPDOP}}
We now detail Algorithm~\ref{alg:OPDOP} for the tabular case as follows. Our policy evaluation works with regression feature $\phi_{\diamond,h}^\tau$: $\calS\times\calA\to\mathbb{R}^{d_2}$,
\[
\phi_{\diamond,h}^\tau(x,a) \;=\; \sum_{x'} \psi(x,a,x') V_{\diamond,h+1}^\tau(x'),\; \text{ for any } (x,a)\in\calS\times\calA
\]
where $\diamond=r$ or $g$. Thus, for any $(\bar{x},\bar{a},\bar{x}')\in\calS\times\calA\times\calS$, the $(\bar{x},\bar{a},\bar{x}')$th entry of $\phi_{\diamond,h}^\tau(x,a)$ is given by 
\[
\sbr{\phi_{\diamond,h}^\tau(x,a)}_{(\bar{x},\bar{a},\bar{x}')} \;=\; \one\{ (x,a) = (\bar{x},\bar{a}) \} V_{\diamond,h+1}^\tau(\bar{x}')
\]
which shows that $\phi_{\diamond,h}^\tau(x,a)$ is a sparse vector with $|\calS|$ nonzero elements at $\{(x,a,x'),x'\in\calS\}$ and the $(x,a,x')$th entry of $\phi_{\diamond,h}^\tau(x,a)$ is $V_{\diamond,h+1}^\tau(x')$. For instance of $\diamond=r$, the regularized least-squares problem~\eqref{eq.lsQ} becomes
\[
\sum_{\tau\,=\,1}^{k-1} \Big({V_{r,h+1}^\tau(x_{h+1}^\tau)\,-\,\sum_{(x,a,x')} \one\{ (x,a) = (x_h^\tau,a_h^\tau) \}  V_{r,h+1}^\tau(x') [w]_{(x,a,x')}}\Big)^2 \,+\, \lambda\,\|w\|_2^2
\]
where $[w]_{(x,a,x')}$ is the $(x,a,x')$th entry of $w$, and the solution $w_{r,h}^k$ serves as an estimator of the transition kernel $\mathbb{P}_h(\cdot\,|\,\cdot,\cdot)$. 
%Thus, we can estimate the action-value function by adding a bonus term $\Gamma_{r,h}^k$: $\calS\times\calA\to\mathbb{R}$ as follows,
%\begin{equation}\label{eq.estimateQ}
%\begin{array}{rcl}
%Q_{r,h}^k(x,a) &\!\!=\!\!& \min\big(\,r_h(x,a) +\phi_{r,h}^k(x,a)^\top w_{r,h}^k+ \Gamma_{r,h}^k(x,a),\, H-h+1\,\big)^+ 
%\\[0.2cm]
%&\!\!=\!\!&\min\Big(\,r_h(x,a) +\sum_{x'\,\in\,\calS}V_{r,h+1}^k(x') [{w_{r,h}^k}]_{(x,a,x')}+ \Gamma_{r,h}^k(x,a),\, H-h+1\,\Big)^+ 
%\end{array}
%\end{equation}
%for any $(x,a)\in\calS\times\calA$. Thus, $V_{r,h}^k(x) = \langle{Q_{r,h}^k(x, \cdot)}, {\pi_h^k(\cdot\,|\,x)}\rangle_\calA$. Similarly, we estimate $Q_{g,h}^k(x,a)$ and $V_{g,h}^k(x)$. Using already estimated $\{Q_{r,h}^k(\cdot, \cdot),Q_{g,h}^k(\cdot, \cdot)\}_{h\,=\,1}^H$, we can execute the policy improvement and the dual update in Algorithm~\ref{alg:OPDOP} for the full-information setting.
%
%
On the other hand, since $\varphi(x_h^\tau,a_h^\tau)=\mathbf{e}_{(x_h^\tau,a_h^\tau)}$, the regularized least-squares problem~\eqref{eq.lsR} becomes
\[
%w_{r,h}^k 
%\;\leftarrow\;
%\argmin_{w\in\mathbb{R}^d} 
%\; 
\sum_{\tau\,=\,1}^{k-1} \big(r_h(x_h^\tau,a_h^\tau)\,-\,[u]_{(x_h^\tau,a_h^\tau)}\big)^2 \,+\, \lambda\,\|u\|_2^2
\] 
where $[u]_{(x,a)}$ is the $(x,a)$th entry of $u$, the solution $u_{r,h}^k$ gives an estimate of $r_h(x,a)$ as $\varphi(x,a)^\top u_{r,h}^k$. By adding similar UCB bonus terms $\Gamma_{h}^k$, $\Gamma_{r,h}^k$: $\calS\times\calA\to\mathbb{R}$ given in Algorithm~\ref{alg:LSVI}, we estimate the action-value function as follows,
\[
\begin{array}{rcl}
Q_{r,h}^k(x,a) &\!\!=\!\!& \displaystyle\min\big(\,[u_{r,h}^k]_{(x,a)}  +\phi_{r,h}^k(x,a)^\top w_{r,h}^k+ (\Gamma_{h}^k+\Gamma_{r,h}^k)(x,a),\, H-h+1\,\big)^+ 
\\[0.2cm]
&\!\!=\!\!&\displaystyle\min\bigg(\,[u_{r,h}^k]_{(x,a)} +\sum_{x'\,\in\,\calS}V_{r,h+1}^k(x') [{w_{r,h}^k}]_{(x,a,x')}+ (\Gamma_{h}^k+\Gamma_{r,h}^k)(x,a),\, H-h+1\,\bigg)^+ 
\end{array}
\]
for any $(x,a)\in\calS\times\calA$. Thus, $V_{r,h}^k(x) = \langle{Q_{r,h}^k(x, \cdot)}, {\pi_h^k(\cdot\,|\,x)}\rangle_\calA$. Similarly, we estimate $g_h(x,a)$ and thus $Q_{g,h}^k(x,a)$ and $V_{g,h}^k(x)$. Using already estimated $\{Q_{r,h}^k(\cdot, \cdot),Q_{g,h}^k(\cdot, \cdot), V_{r,h}^k(\cdot), Q_{g,h}^k(\cdot)\}_{h\,=\,1}^H$, we execute the policy improvement and the dual update in Algorithm~\ref{alg:OPDOP}.

We restate the result of Theorem~\ref{thm.main-full} for the tabular case as follows.
\begin{cor}[Regret and Constraint Violation]\label{cor.bandit}
	For the tabular CMDP with feature maps~\eqref{eq.tabular}, let Assumption~\ref{as.slater} hold.
	Fix $p\in(0,1)$. 
	In Algorithm~\ref{alg:OPDOP}, we set 
	$\alpha={\sqrt{\log\abr{\calA}}}/(H^2K)$, $\beta=C_1 \sqrt{|\calS|^2|\calA|H^2\log\rbr{{|\calS||\calA|T}/{p}}}$, $\eta=1/\sqrt{K}$, $\theta = {1}/K$, and $\lambda=1$ where $C_1$ is an absolute constant. Then, the regret and the constraint violation in~\eqref{eq.regret} satisfy 
	\[
	\text{\normalfont Regret}(K) \leq \displaystyle C |\calS|^2|\calA| H^{2.5} \sqrt{T} \log\!\rbr{\!\frac{|\calS||\calA|T}{p}\!}
	\text{ and } 
	\text{\normalfont Violation}(K) \leq \displaystyle C^{'} |\calS|^2|\calA| H^{2.5}\sqrt{T} {\log\!\rbr{\!\frac{|\calS||\calA|T}{p}}\!}
	\]
	%	\[
	%	\begin{array}{rcl}
	%	\frac{1}{K}\sum_{k=1}^{K}V_{g,1}^{\pi^k,k}\rbr{x_1} 
	%	&\!\!\geq\!\!& 
	%	b-C^{'''}\frac{H^4 d^3{\log^2\rbr{\frac{dT}{p}}}}{\sqrt{K}}.
	%	\end{array}
	%	\]
	with probability $1-p$ where $C$ and $C^{'}$ are absolute constants. 
\end{cor}
\begin{proof}
	It follows the proof of Theorem~\ref{thm.main-full} by noting that the tabular CMDP is a special linear MDP in Assumption~\ref{as.linearMDP}, with $d = |\calS|^2|\calA|$, and we have $\log\abr{\calA}\leq O\rbr{d^2\log\rbr{{dT}/{p}}}$ automatically.
\end{proof}

\subsection{Further Results: Proof of Theorem~\ref{thm.tabular}}\label{app.further}

As we see in the proof of Theorem~\ref{thm.main-full}, our final regret or constraint violation bounds are dominated by the accumulated bonus terms, which come from the design of `optimism in the face of uncertainty.' This framework provides a powerful flexibility for Algorithm~\ref{alg:OPDOP} to incorporate other optimistic policy evaluation methods.  In what follows, we introduce Algorithm~\ref{alg:OPDOP} with a variant of optimistic policy evaluation.

We repeat notation for readers' convenience. For any $(h,k)\in[H]\times[K]$, any $(x,a,x')\in\calS\times\calA\times\calS$, and any $(x,a)\in\calS\times\calA$, we define two visitation counters $n_h^k(x,a,x')$ and $n_h^k(x,a)$ at step $h$ in episode $k$,
\[
n_h^k(x,a,x') \;=\; \sum_{\tau\,=\,1}^{k-1} \one\{(x,a,x') = (x_h^\tau,a_h^\tau,a_{h+1}^\tau) \}
\;
\text{ and }
\;
n_h^k(x,a) \;=\; \sum_{\tau\,=\,1}^{k-1} \one\{(x,a) = (x_h^\tau,a_h^\tau) \}.
\]
This allows us to estimate transition kernel $\mathbb{P}_h$, reward function $r$, and utility function $g$ for episode $k$ by 
\[
\hat{\mathbb{P}}_h^k (x'\,|\,x,a) 
\;=\;
\frac{n_h^k(x,a,x')}{n_h^k(x,a) + \lambda},\; \text{ for all } (x,a,x')\,\in\,\calS\times\calA\times\calS 
\]
\[
\hat{r}_h^k(x,a) \;=\; \frac{1}{n_h^k(x,a) +\lambda} \sum_{\tau\,=\,1}^{k-1} \one\{ (x,a) = (x_h^\tau,a_h^\tau) \} r_h(x_h^\tau,a_h^\tau),\; \text{ for all } (x,a)\in\calS\times\calA.
\]
\[
\hat{g}_h^k(x,a) \;=\; \frac{1}{n_h^k(x,a) +\lambda} \sum_{\tau\,=\,1}^{k-1} \one\{ (x,a) = (x_h^\tau,a_h^\tau) \} g_h(x_h^\tau,a_h^\tau),\;
\text{ for all } (x,a)\in\calS\times\calA.
\]
where $\lambda>0$ is the regularization parameter. Moreover, we introduce the bonus term $\Gamma_{h}^k$: $\calS\times\calA\to\mathbb{R}$,
\[
\Gamma_h^k (x,a) \;=\; \beta \rbr{n_h^k(x,a)+\lambda}^{-1/2}
\]
which adapts the counter-based bonus terms in the literature~\cite{azar2017minimax,jin2018q}, where $\beta>0$ is to be determined later. Using the estimated transition kernels $\{\hat{\mathbb{P}}_h^k\}_{h\,=\,1}^H$, the estimated reward/utility functions $\{\hat{r}_h^k,\hat{g}_h^k\}_{h\,=\,1}^H$, and the bonus terms $\{\Gamma_h^k\}_{h\,=\,1}^H$, we now can estimate the action-value function via
\[
Q_{\diamond,h}^k(x,a) \;=\; \min\Big(\,\hat{\diamond}_h^k(x,a) +\sum_{x'\,\in\,\calS} \hat{\mathbb{P}}_h(x'\,|\,x,a) V_{\diamond,h+1}^k(x')+ 2\Gamma_{h}^k(x,a),\, H-h+1\,\Big)^+ 
\]
for any $(x,a)\in\calS\times\calA$, where $\diamond = r$ or $g$. Thus, $V_{\diamond,h}^k(x) = \langle{Q_{\diamond,h}^k(x, \cdot)}, {\pi_h^k(\cdot\,|\,x)}\rangle_\calA$. We summarize the above procedure in Algorithm~\ref{alg:tbandit}. Using already estimated $\{Q_{r,h}^k(\cdot, \cdot),Q_{g,h}^k(\cdot, \cdot)\}_{h\,=\,1}^H$, we can execute the policy improvement and the dual update in Algorithm~\ref{alg:OPDOP}.

Similar to Theorem~\ref{thm.main-full}, we prove the following regret and constraint violation bounds. 

\begin{theorem}[Regret and Constraint Violation]\label{thm.tfull}
	For the tabular CMDP with feature maps~\eqref{eq.tabular}, let Assumption~\ref{as.slater} hold.
	Fix $p\in\rbr{0,1}$. 
	In Algorithm~\ref{alg:OPDOP}, we set
	$\alpha={\sqrt{\log\abr{\calA}}}/(H^2K)$, $\beta = C_1 H\sqrt{|\calS| \log(|\calS||\calA|T/p)}$, $\eta=1/\sqrt{K}$, $\theta = {1}/K$, and $\lambda=1$ where $C_1$ is an absolute constant. Then, the regret and the constraint violation in~\eqref{eq.regret} satisfy 
	\[
	\text{\normalfont Regret}(K) \leq \displaystyle C |\calS|\sqrt{|\calA| H^{5}T} \log\!\rbr{\!\frac{|\calS||\calA|T}{p}\!}
	\text{ and }
	\text{\normalfont Violation}(K) \leq \displaystyle C'|\calS|\sqrt{|\calA| H^{5}T} \log\!\rbr{\!\frac{|\calS||\calA|T}{p}\!}
	\]
	%	\[
	%	\begin{array}{rcl}
	%	\frac{1}{K}\sum_{k=1}^{K}V_{g,1}^{\pi^k,k}\rbr{x_1} 
	%	&\!\!\geq\!\!& 
	%	b-C'\frac{H^4 d^3{\log^2\rbr{\frac{dT}{p}}}}{\sqrt{K}}.
	%	\end{array}
	%	\]
	with probability $1-p$ where $C$ and $C'$ are absolute constants. 
\end{theorem}
\begin{proof}
	The proof is similar to Theorem~\ref{thm.main-full}. Since we only change the policy evaluation, all previous policy improvement results still hold. By Lemma~\ref{lem.reger}, we have
	\[
	\begin{array}{rcl}
	\text{Regret}(K) 
	%	& \!\!=\!\! & C_2 H^{3.5} \sqrt{T \log|\calA|}
	%	\,+\,
	%	\displaystyle\sum_{k\,=\,1}^{K}\sum_{h\,=\,1}^H \mathbb{E}_{\pi^\star}\sbr{\iota_{r,h}^k(x_h,a_h) \,\vert\, x_1} 
	%	\,+\,
	%	\sum_{k\,=\,1}^{K}\rbr{V_{r,1}^{k}(x_1) -V_{r,1}^{\pi^k,k}(x_1) }
	%	\\[0.2cm]
	& \!\!=\!\! & C_3H^{2.5} \sqrt{ T \log|\calA|}
	\,+\,
	\displaystyle\sum_{k\,=\,1}^{K}\sum_{h\,=\,1}^H\rbr{ \mathbb{E}_{\pi^\star}\sbr{\iota_{r,h}^k(x_h,a_h)} 
		-\iota_{r,h}^k(x_h^k,a_h^k)}
	\,+\,
	M_{r,H,2}^K
	\end{array}
	\]
	where $\iota_{r,h}^k$ is the model prediction error given by~\eqref{eq.mper} and $\{M_{r,h,m}^k\}_{(k,h,m)\in[K]\times[H]\times[2]}$ is a martingale adapted to the filtration $\{\calF_{h,m}^k\}_{(k,h,m)\in[K]\times[H]\times[2]}$ in terms of time index $t$ defined in~\eqref{eq.t}. By Lemma~\ref{lem.M}, it holds with probability $1-p/3$ that $	|M_{r,H,2}^K| \leq 4\sqrt{H^2T \log({4}/{p})}$. The rest is to bound the double sum term. As shown in Section~\ref{subsec.ucb-full-t}, with probability $1-p/2$ it holds that for any $(k,h)\in[K]\times[H]$ and $(x,a)\in\calS\times\calA$,
	\begin{equation}\label{eq.ucb-full-t}
	-4 \Gamma_{h}^k(x,a) \;\leq\;\iota_{r,h}^k(x,a) \;\leq\;0.
	\end{equation}
	Together with the choice of $\Gamma_{h}^k$, we have 
	\[
	\begin{array}{rcl}
	\displaystyle\sum_{k\,=\,1}^{K}\sum_{h\,=\,1}^H\rbr{ \mathbb{E}_{\pi^\star}\sbr{\iota_{r,h}^k(x_h,a_h) \,\vert\, x_1} 
		-\iota_{r,h}^k(x_h^k,a_h^k)}
	&\!\!\leq\!\!& \displaystyle 4
	\sum_{k\,=\,1}^{K}\sum_{h\,=\,1}^H\Gamma_{h}^k(x_h^k,a_h^k)
	\\[0.2cm]
	&\!\!=\!\!&\displaystyle
	4\beta \sum_{k\,=\,1}^{K}\sum_{h\,=\,1}^H\rbr{n_h^k(x_h^k,a_h^k)+\lambda}^{-1/2}.
	\end{array}
	\]
	Define mapping $\bar\phi$: $\calS\times\calA\to\mathbb{R}^{|\calS||\calA|}$ as $\bar\phi(x,a) = \mathbf{e}_{(x,a)}$, we can utilize Lemma~\ref{lem.bdsums}. For any $(k,h)\in[K]\times[H]$, we have
	\[
	\bar\Lambda_{h}^k \;=\; \sum_{\tau\,=\,1}^{k-1} \bar{\phi}(x_h^\tau,a_h^\tau) \bar{\phi}(x_h^\tau,a_h^\tau)^\top + \lambda I \,\in\,\mathbb{R}^{|\calS||\calA|\times|\calS||\calA|}
	\]
	\[
	\Gamma_h^k(x,a) \;=\; \beta \rbr{n_h^k(x,a)+\lambda}^{-1/2}\;=\;\beta\sqrt{\bar{\phi}(x,a) (\bar\Lambda_h^k)^{-1}\bar{\phi}(x,a)^\top}
	\]
	where $\bar\Lambda_{h}^k$ is a diagonal matrix whose the $(x,a)$th diagonal entry is $n_h^k(x,a)+\lambda$. Therefore, we have
	\[
	\begin{array}{rcl}
	\displaystyle\sum_{k\,=\,1}^{K}\sum_{h\,=\,1}^H\rbr{ \mathbb{E}_{\pi^\star}\sbr{\iota_{r,h}^k(x_h,a_h)} 
		-\iota_{r,h}^k(x_h^k,a_h^k)}
	&\!\!\leq\!\!& \displaystyle
	4\beta \sum_{k\,=\,1}^{K}\sum_{h\,=\,1}^H\rbr{\bar{\phi}(x_h^k,a_h^k) (\bar\Lambda_h^k)^{-1}\bar{\phi}(x_h^k,a_h^k)^\top}^{1/2}
	\\[0.2cm]
	&\!\!\leq\!\!&\displaystyle
	4\beta \sum_{h\,=\,1}^H\rbr{K \sum_{k\,=\,1}^{K}\bar{\phi}(x_h^k,a_h^k) (\bar\Lambda_h^k)^{-1}\bar{\phi}(x_h^k,a_h^k)^\top}^{1/2}
	\\[0.2cm]
	&\!\!\leq\!\!&\displaystyle
	4\beta\sqrt{2K} \sum_{h\,=\,1}^H \log^{1/2}\rbr{\frac{\det\rbr{\bar\Lambda_{h}^{K+1}} }{\det\bar\Lambda_h^1}}
	\end{array}
	\]
	where we apply the Cauchy-Schwartz inequality for the second inequality and Lemma~\ref{lem.bdsums} for the third inequality.
	Notice that $(K+\lambda)I \succeq \bar\Lambda_{h}^K$ and $\bar\Lambda_{h}^1=\lambda I$. Hence, 
	\[
	\begin{array}{rcl}
	\text{Regret}(K) 
	%	& \!\!=\!\! & C_2 H^{3.5} \sqrt{T \log|\calA|}
	%	\,+\,
	%	\displaystyle\sum_{k\,=\,1}^{K}\sum_{h\,=\,1}^H \mathbb{E}_{\pi^\star}\sbr{\iota_{r,h}^k(x_h,a_h) \,\vert\, x_1} 
	%	\,+\,
	%	\sum_{k\,=\,1}^{K}\rbr{V_{r,1}^{k}(x_1) -V_{r,1}^{\pi^k,k}(x_1) }
	%	\\[0.2cm]
	& \!\!=\!\! & C_3H^{2.5} \sqrt{ T \log|\calA|}
	\,+\,
	4\beta \sqrt{2|\calS||\calA|HT} \sqrt{\log\rbr{\dfrac{K+\lambda}{\lambda}}}
	\,+\,
	4\sqrt{H^2T \log\rbr{\dfrac{6}{p}}}.
	\end{array}
	\]
	Notice that $\log|\calA| \leq O \big( |\calS|^2{|\calA|}\log^2(|\calS||\calA|T/p) \big)$. 
	By setting $\lambda=1$ and $\beta = C_1 H\sqrt{|\calS| \log(|\calS||\calA|T/p)}$, we conclude the desired regret bound.
	
	For the constraint violation analysis, Lemmas~\ref{lem.C1} still holds. Similar to~\eqref{eq.C0}, we have
	\[
	\begin{array}{rcl}
	&&\!\!\!\!\!\!\!\!\!\!  \displaystyle
	V_{r,1}^{\pi^\star}(x_1) -V_{r,1}^{\pi'}(x_1) \,+\,\chi \left[ b-V_{g,1}^{\pi'}(x_1)\right]_{+}
	\\[0.2cm]
	&\!\!\leq\!\!&
	\displaystyle \frac{C_4 H^{2.5} \sqrt{T \log|\calA|}}{K}
	\,+\,\frac{4}{K}\sum_{k\,=\,1}^{K}\sum_{h\,=\,1}^H \Gamma_{h}^k(x_h^k,a_h^k)\,+\,\frac{4\chi}{K}\sum_{k\,=\,1}^{K}\sum_{h\,=\,1}^H\Gamma_{h}^k(x_h^k,a_h^k)
	%	\\[0.2cm]
	%	&&\displaystyle
	\,+\,
	\frac{1}{K}M_{r,H,2}^K
	\,+\,
	\frac{\chi }{K}\big| M_{g,H,2}^K\big|
	\end{array}
	\]
	where $V_{r,1}^{\pi'}(x_1)=\frac{1}{K}\sum_{k\,=\,1}^{K}V_{r,1}^{\pi^k}(x_1)$ and $V_{g,1}^{\pi'}(x_1)=\frac{1}{K} \sum_{k\,=\,1}^{K}V_{g,1}^{\pi^k}(x_1)$. Similar to Lemma~\ref{lem.M}, it holds with probability $1-p/3$ that $	|M_{g,H,2}^K| \leq 4\sqrt{H^2T \log({6}/{p})}$ for $\diamond=r$ or $g$. As shown in Section~\ref{subsec.ucb-full-t}, with probability $1-p/3$ it holds that $-4 \Gamma_{h}^k(x,a) \leq \iota_{\diamond,h}^k(x,a) \leq 0$ for any $(k,h)\in[K]\times[H]$ and $(x,a)\in\calS\times\calA$. Therefore, we have
	\[
	\begin{array}{rcl}
	&& V_{r,1}^{\pi^\star}(x_1) -V_{r,1}^{\pi'}(x_1) \,+\,\chi \left[ b-V_{g,1}^{\pi'}(x_1)\right]_{+} 
	\\[0.2cm]
	&& \;\leq\;
	\dfrac{C_4 H^{2.5} \sqrt{T\log (\calA|)}}{K}
	\,+\,
	\dfrac{4(1+\chi) \beta \sqrt{2|\calS||\calA|HT} }{K}\sqrt{\log\rbr{ \dfrac{K+\lambda}{\lambda}}}
	\,+\,
	\dfrac{4(1+\chi)}{K}\sqrt{H^2T \log \rbr{\dfrac{6}{p}}}
	\end{array}
	\]
	which leads to the desired constraint violation bound due to Lemma~\ref{thm.violationgeneral} and we set $\lambda$ and $\beta$ as previously.
\end{proof}

\section{Other Verifications}

In this section, we collect some verifications for readers’ convenience.

\subsection{Proof of Formulas~\eqref{eq.RI} and~\eqref{eq.gRI}}\label{subsec.RCI}
For any $(k,h)\in[K]\times[H]$, we recall the definitions of $V_{r,h}^{\pi^\star}$ in the Bellman equations~\eqref{eq.bellman} and $V_{r,h}^{k}$ from line~12 in Algorithm~\ref{alg:LSVI},
\[
V_{r,h}^{\pi^\star}(x) \;=\; \big\langle{Q_{h}^{\pi^\star}(x,\cdot)},{\pi_h^\star(\,\cdot\,\vert\, x)}\big\rangle\;\text{ and }\; V_{r,h}^{k}(x) \;=\; \big\langle{Q_h^k(x,\,\cdot\,)},{\pi_h^k(\,\cdot\,\vert\, x)}\big\rangle.
\]
We can expand the difference $V_{r,h}^{\pi^\star}(x)-V_{r,h}^{k}(x)$ into 
\begin{equation}\label{eq.V1}
\begin{array}{rcl}
V_{r,h}^{\pi^\star}(x) \,-\, V_{r,h}^{k}(x) 
&\!\!=\!\!& 
\big\langle{Q_{h}^{\pi^\star}(x,\,\cdot\,)},{\pi_h^\star(\,\cdot\,\vert \,x)}\big\rangle
\,-\, 
\big\langle{Q_h^k(x,\,\cdot\,)},{\pi_h^k(\,\cdot\,\vert\, x)}\big\rangle
\\[0.2cm]
&\!\!=\!\!& 
\big\langle{Q_{h}^{\pi^\star}(x,\,\cdot\,)-Q_h^k(x,\,\cdot\,)},{\pi_h^\star(\,\cdot\,\vert\, x)}\big\rangle 
\,+\,
\big\langle{Q_h^k(x,\,\cdot\,)},{\pi_h^\star(\,\cdot\,\vert\, x)-\pi_h^k(\,\cdot\,\vert\, x)}\big\rangle
\\[0.2cm]
&\!\!=\!\!& 
\big\langle{Q_{h}^{\pi^\star}(x,\,\cdot\,)-Q_h^k(x,\,\cdot\,)},{\pi_h^\star(\,\cdot\,\vert \,x)}\big\rangle 
\,+\,
\xi_h^k(x),
\end{array}
\end{equation}
where $\xi_h^k(x) \DefinedAs \langle{Q_h^k(x,\,\cdot\,)},{\pi_h^\star(\,\cdot\,\vert\, x)-\pi_h^k(\,\cdot\,\vert\, x)}\rangle$.

Recall the equality in the Bellman equations~\eqref{eq.bellman} and the model prediction error,
\[
Q_{r,h}^{\pi^\star} \;=\; r_h^k \,+\, \mathbb{P}_h V_{r,h+1}^{\pi^\star}\; \text{ and }\; \iota_{r,h}^k \;=\; r_h \,+\, \mathbb{P}_h V_{r,h+1}^k \,-\, Q_{r,h}^k.
\]
As a result of the above two, it is easy to see that
\[
Q_{r,h}^{\pi^\star} \,-\, Q_{r,h}^k \;=\; \mathbb{P}_h \big({V_{r,h+1}^{\pi^\star} - V_{r,h+1}^k}\big) \,+\, \iota_{r,h}^k.
\] 
Substituting the above difference into the right-hand side of~\eqref{eq.V1} yields,
\[
V_{r,h}^{\pi^\star}(x) \,-\, V_{r,h}^{k}(x) 
\; = \;
\big\langle{\mathbb{P}_h \big({V_{r,h+1}^{\pi^\star} - V_{r,h+1}^k}\big)\rbr{x,\,\cdot\,}},{\pi_h^\star(\,\cdot\,\vert x)}\big\rangle 
\,+\,
\big\langle{\iota_{r,h}^k(x,\,\cdot\,)},{\pi_h^\star(\,\cdot\,\vert\, x)}\big\rangle 
\,+\,
\xi_h^k(x).
\]
which displays a recursive formula over $h$. 
Thus, we expand $V_{r,1}^{\pi^\star}(x_1)-V_{r,1}^{k}(x_1)$ recursively with $x=x_1$ as
\begin{equation}
\label{eq.V2}
\begin{array}{rcl}
V_{r,1}^{\pi^\star}(x_1)\,-\,V_{r,1}^{k}(x_1) 
&\!\! = \!\!&
\big\langle{\mathbb{P}_1 \big({V_{r,2}^{\pi^\star} - V_{r,2}^k}\big)\rbr{x_1,\cdot}},{\pi_1^\star(\cdot\vert x_1)}\big\rangle 
\,+\,
\big\langle{\iota_{r,1}^k(x_1,\cdot)},{\pi_1^\star(\cdot\vert x_1)}\rangle 
\,+\,
\xi_1^k(x_1)
\\[0.2cm]
&\!\! = \!\!&
\inner{\mathbb{P}_1 {\big\langle{\mathbb{P}_2 \big({V_{r,3}^{\pi^\star} - V_{r,3}^k}\big)\rbr{x_2,\cdot}},{\pi_2^\star(\cdot\vert x_2)}\big\rangle }\rbr{x_1,\cdot}}{\pi_1^\star(\cdot\vert x_1)} 
\\[0.3cm]
&& \,+\,\inner{\mathbb{P}_1 \big\langle{\iota_{r,2}^k(x_2,\cdot)},{\pi_2^\star(\cdot\vert x_2)}\big\rangle\rbr{x_1,\cdot}}{\pi_1^\star(\cdot\vert x_1)}
\,+\,\big\langle{\iota_{r,1}^k(x_1,\cdot)},{\pi_1^\star(\cdot\vert x_1)}\big\rangle
\\[0.3cm]
&&
\,+\,\big\langle{\mathbb{P}_1\xi_2^k\rbr{x_1,\cdot}},{\pi_1^\star(\cdot\vert x_1)}\big\rangle
\,+\,
\xi_1^k(x_1).
\end{array}
\end{equation}

For notational simplicity, for any $(k,h)\in[K]\times[H]$, we define an operator $\calI_h$ for function $f:\calS\times\calA\to \mathbb{R}$,
\[
\rbr{\calI_h f}\rbr{x} \;=\; \big\langle{f(x,\,\cdot\,)},{\pi_h^\star\rbr{\,\cdot\,\vert x}}\big\rangle.
%\text{ and } \rbr{\calI_{k,h} f}\rbr{x} \;=\; \inner{f(x,\,\cdot\,)}{\pi_h^k(\,\cdot\,\vert x)}.
\]
With this notation, repeating the above recursion~\eqref{eq.V2}  over $h\in[H]$ yields
\[
\begin{array}{rcl}
&& \!\!\!\!\!\!\!\!\!\! V_{r,1}^{\pi^\star}(x_1)-V_{r,1}^{k}(x_1) 
\\[0.2cm]
&\!\! = \!\!&
\calI_1\mathbb{P}_1\calI_2\mathbb{P}_2 \big({V_{r,3}^{\pi^\star} - V_{r,3}^k}\big)
+\calI_1 \mathbb{P}_1\calI_2 \iota_{r,2}^k
+\calI_1 \iota_{r,1}^k
+\calI_1\mathbb{P}_1\xi_2^k+\xi_1^k
\\[0.2cm]
&\!\! = \!\!&
\calI_1\mathbb{P}_1\calI_2\mathbb{P}_2\calI_3\mathbb{P}_3 \big({V_{r,4}^{\pi^\star} - V_{r,4}^k}\big)
+\calI_1 \mathbb{P}_1\calI_2 \mathbb{P}_2\calI_3\iota_{r,3}^k
+\calI_1 \mathbb{P}_1\calI_2 \iota_{r,2}^k
+\calI_1 \iota_{r,1}^k
+\calI_1\mathbb{P}_1\calI_2\mathbb{P}_2\xi_3^k
+\calI_1\mathbb{P}_1\xi_2^k+\xi_1^k
\\
&\!\! \vdots \!\!& 
\\
&\!\! = \!\!&  \rbr{\displaystyle\prod_{h\,=\,1}^H\calI_h\mathbb{P}_h }\big({V_{r,H+1}^{\pi^\star} - V_{r,H+1}^k}\big)
+\displaystyle\sum_{h\,=\,1}^{H}\rbr{\prod_{i\,=\,1}^{h-1} \calI_i \mathbb{P}_i} \calI_h \iota_{r,h}^k 
+\sum_{h\,=\,1}^H\rbr{\prod_{h\,=\,1}^{h-1} \calI_i\mathbb{P}_i}\xi_h^k.
\end{array}
\]
Finally, notice that $V_{r,H+1}^{\pi^\star} =V_{r,H+1}^k=0$, we use the definitions of $\mathbb{P}_h$ and $\calI_h$ to conclude~\eqref{eq.RI}. Similarly, we can also use the above argument to verify~\eqref{eq.gRI}.

\subsection{Proof of Formulas~\eqref{eq.RII} and~\eqref{eq.CII}}\label{subsec.RCII}

We recall the definition of $V_{r,h}^{\pi^k}$ and define an operator $\calI_h^k$ for function $f:\calS\times\calA\to \mathbb{R}$,
\[
V_{r,h}^{\pi^k}(x) \;=\; \big\langle{Q_h^{\pi^k}(x,\,\cdot\,)},{\pi_h^k(\,\cdot\,\vert \,x)}\big\rangle
\;\text{ and }\;
\rbr{\calI_{h}^k f}\rbr{x} \;=\; \big\langle{f(x,\,\cdot\,)},{\pi_h^k(\,\cdot\,\vert\, x)}\big\rangle.
\]
We expand the model prediction error $\iota_{r,h}^k$ into,
\[
\begin{array}{rcl}
\iota_{r,h}^k(x_h^k,a_h^k) 
&\!\!=\!\!& 
r_h(x_h^k,a_h^k) 
\,+\,
(\mathbb{P}_h V_{r,h+1}^k)(x_h^k,a_h^k)-Q_{r,h}^k(x_h^k,a_h^k)
\\[0.2cm]
&\!\!=\!\!&
\rbr{r_h(x_h^k,a_h^k)+(\mathbb{P}_h V_{r,h+1}^k)(x_h^k,a_h^k)-Q_{r,h}^{\pi^k}(x_h^k,a_h^k)} 
\,+\,
\rbr{Q_{r,h}^{\pi^k}(x_h^k,a_h^k)-Q_{r,h}^k(x_h^k,a_h^k)}
\\[0.2cm]
&\!\!=\!\!&
\rbr{\mathbb{P}_h V_{r,h+1}^k-\mathbb{P}_h V_{r,h+1}^{\pi^k}}(x_h^k,a_h^k)
\,+\,
\rbr{Q_{r,h}^{\pi^k}(x_h^k,a_h^k)-Q_{r,h}^k(x_h^k,a_h^k)},
\end{array}
\]
where we use the Bellman equation $Q_{r,h}^{\pi^k}(x_h^k,a_h^k)=r_h(x_h^k,a_h^k)+(\mathbb{P}_h V_{r,h+1}^{\pi^k})(x_h^k,a_h^k)$ in the last equality. With the above formula, we expand the difference $V_{r,1}^{k}(x_1) -V_{r,1}^{\pi^k}(x_1)$ into
\[
\begin{array}{rcl}
V_{r,h}^{k}(x_h^k) -V_{r,h}^{\pi^k}(x_h^k)  
&\!\!=\!\!& 
\rbr{\calI_h^k(Q_{r,h}^k-Q_{r,h}^{\pi^k}) }(x_h^k) 
\,-\,
\iota_{r,h}^k(x_h^k,a_h^k) 
\\[0.2cm]
&& \,+\,
\rbr{\mathbb{P}_h V_{r,h+1}^k-\mathbb{P}_h V_{r,h+1}^{\pi^k}}(x_h^k,a_h^k)
\,+\,
\rbr{Q_{r,h}^{\pi^k}-Q_{r,h}^k}(x_h^k,a_h^k).
\end{array}
\]

Let 
\[
\begin{array}{rcl}
D_{r,h,1}^k 
&\!\!\DefinedAs\!\!& 
\rbr{\calI_h^k(Q_{r,h}^k-Q_{r,h}^{\pi^k}) }(x_h^k) 
\,-\,
\rbr{Q_{r,h}^k-Q_{r,h}^{\pi^k}}(x_h^k,a_h^k),
\\[0.2cm]
D_{r,h,2}^k 
&\!\!\DefinedAs\!\!& 
\rbr{\mathbb{P}_h V_{r,h+1}^k-\mathbb{P}_h V_{r,h+1}^{\pi^k}}(x_h^k,a_h^k)
\,-\,
\rbr{V_{r,h+1}^k- V_{r,h+1}^{\pi^k}}(x_{h+1}^k).
\end{array}
\]
Therefore, we have the following recursive formula over $h$,
\[
V_{r,h}^{k}(x_h^k) -V_{r,h}^{\pi^k}(x_h^k)  
\;=\; 
D_{r,h,1}^k
\,+\,
D_{r,h,2}^k
\,+\,
\rbr{V_{r,h+1}^k-V_{r,h+1}^{\pi^k}}(x_{h+1}^k)
\,-\,
\iota_{r,h}^k(x_h^k,a_h^k).
\]
Notice that $V_{r,H+1}^{\pi^k} =V_{r,H+1}^k=0$. Summing the above equality over $h\in[H]$ yields
\begin{equation}\label{eq.V3}
V_{r,1}^{k}(x_1) \,-\, V_{r,1}^{\pi^k}(x_1)  
\;=\;
\sum_{h\,=\,1}^H\rbr{D_{r,h,1}^k+D_{r,h,2}^k} \,-\,\sum_{h\,=\,1}^H\iota_{r,h}^k(x_h^k,a_h^k).
\end{equation}
Following the definitions of $\calF_{h,1}^k$ and $\calF_{h,2}^k$, we know $D_{r,h,1}^k\in\calF_{h,1}^k$ and $D_{r,h,2}^k\in\calF_{h,2}^k$. Thus, for any $(k,h)\in[K]\times[H]$,
\[
\mathbb{E}\sbr{D_{r,h,1}^k\,\vert\, \calF_{h-1,2}^k}\;=\;0 
\;\text{ and }\;
\mathbb{E}\sbr{D_{r,h,2}^k\,\vert\, \calF_{h,1}^k} \;=\; 0.
\]
Notice that $t(k,0,2) = t(k-1,H,2) = 2H(k-1)$. Clearly, $\calF_{0,2}^k = \calF_{H,2}^{k-1}$ for any $k\geq 2$. Let $\calF_{0,2}^1$ be empty. We define a martingale sequence,
\[
\begin{array}{rcl}
M_{r,h,m}^k &\!\!= \!\!& \displaystyle\sum_{\tau\,=\,1}^{k-1}\sum_{i\,=\,1}^H \rbr{D_{r,i,1}^\tau+D_{r,i,2}^\tau} 
\,+\,
\sum_{i\,=\,1}^{h-1} \rbr{D_{r,i,1}^k+D_{r,i,2}^k} 
\,+\,
\sum_{\ell\,=\,1}^m D_{r,h,\ell}^k
\\[0.2cm]
&\!\!= \!\!& \displaystyle\sum_{\rbr{\tau,i,\ell}\,\in\,[K]\times[H]\times[2],\, t(\tau,i,\ell)\,\leq\, t(k,h,m) } D_{r,i,\ell}^\tau,
\end{array}
\]
where $t(k,h,m)\DefinedAs 2(k-1)H+2(h-1)+m$ is the time index. Clearly, this martingale is adapted to the filtration $\{\calF_{h,m}^k\}_{(k,h,m)\in[K]\times[H]\times[2]}$, and particularly, 
\[
\sum_{k\,=\,1}^K\sum_{h\,=\,1}^H(D_{r,h,1}^k+D_{r,h,2}^k)  \;=\; 
M_{r,H,2}^K.
\]

Finally, we combine the above martingale with~\eqref{eq.V3} to obtain~\eqref{eq.RII}. Similarly, we can show~\eqref{eq.CII}.

\subsection{Proof of Formula~\eqref{eq.ucb-bandit}}\label{subsec.ucb-bandit}

We recall the definition of the feature map $\phi_{r,h}^k$,
\[
\phi_{r,h}^k(x,a) \;=\;\int_{\calS} \psi(x,a,x') V_{r,h+1}^k(x') dx'
\]
for any $(k,h)\in[K]\times[H]$ and $(x,a)\in\calS\times\calA$. By Assumption~\ref{as.linearMDP}, we have 
\[
\begin{array}{rcl}
(\mathbb{P}_h V_{r,h+1}^k) \rbr{x,a} 
&\!\!=\!\!& \displaystyle
\int_{\calS}\psi\rbr{x,a,x'}^\top \theta_h \cdot V_{r,h+1}^k (x') dx'
\\[0.5cm]
&\!\!=\!\!&\displaystyle
\phi_{r,h}^k(x,a)^\top \theta_h
\\[0.2cm]
&\!\!=\!\!&\displaystyle
\phi_{r,h}^k(x,a)^\top (\Lambda_{r,h}^k)^{-1}\Lambda_{r,h}^k\theta_h
\\[0.2cm]
&\!\!=\!\!&\displaystyle
\phi_{r,h}^k(x,a)^\top (\Lambda_{r,h}^k)^{-1}\rbr{\sum_{\tau\,=\,1}^{k-1} \phi_{r,h}^\tau(x_h^\tau,a_h^\tau)\phi_{r,h}^\tau(x_h^\tau,a_h^\tau)^\top\theta_h + \lambda \theta_h}
\\[0.5cm]
&\!\!=\!\!&\displaystyle
\phi_{r,h}^k(x,a)^\top (\Lambda_{r,h}^k)^{-1} \rbr{\sum_{\tau\,=\,1}^{k-1} \phi_{r,h}^\tau(x_h^\tau,a_h^\tau) \cdot(\mathbb{P}_h V_{r,h+1}^\tau) \rbr{x_h^\tau,a_h^\tau} + \lambda \theta_h}
\end{array}
\]
where the second equality is due to the definition of $\phi_{r,h}^k$, we exploit $\Lambda_{r,h}^k = \sum_{\tau\,=\,1}^{k-1} \phi_{r,h}^\tau(x_h^\tau,a_h^\tau)\phi_{r,h}^\tau(x_h^\tau,a_h^\tau)^\top + \lambda I$ from line~4 of Algorithm~\ref{alg:LSVI} in the fourth equality, and we recursively replace $\phi_{r,h}^\tau(x_h^\tau,a_h^\tau)^\top\theta_h $ by $(\mathbb{P}_h V_{r,h+1}^\tau) \rbr{x_h^\tau,a_h^\tau}$ for all $\tau\in[k-1]$ in the last equality.

We recall the update $w_{r,h}^k = (\Lambda_{r,h}^k)^{-1} \sum_{\tau\,=\,1}^{k-1}\phi_{r,h}^\tau(x_h^\tau,a_h^\tau) V_{r,h+1}^\tau(x_{h+1}^\tau)$ from line~5 of Algorithm~\ref{alg:LSVI}. 
Therefore, 
\[
\begin{array}{rcl}
&& \!\!\!\! \!\!\!\! \abr{\phi_{r,h}^k(x,a)^\top w_{r,h}^k- (\mathbb{P}_h V_{r,h+1}^k) \rbr{x,a}}
%&\!\!=\!\!& \displaystyle
%\int_{\calS}\psi\rbr{x,a,x'}^\top \theta_h \cdot V_{r,h+1}^k (x') dx'
%\\[0.5cm]
%&\!\!=\!\!&\displaystyle
%\phi_{r,h}^k(x,a)^\top \theta_h
%\\[0.2cm]
%&\!\!=\!\!&\displaystyle
%\phi_{r,h}^k(x,a)^\top (\Lambda_{r,h}^k)^{-1}\Lambda_{r,h}^k\theta_h
%\\[0.2cm]
%&\!\!=\!\!&\displaystyle
%\phi_{r,h}^k(x,a)^\top (\Lambda_{r,h}^k)^{-1}\rbr{\sum_{\tau\,=\,1}^{k-1} \phi_{r,h}^\tau(x_h^\tau,a_h^\tau)\phi_{r,h}^\tau(x_h^\tau,a_h^\tau)^\top\theta_h + \lambda \theta_h}
\\[0.2cm]
&\!\!=\!\!&\displaystyle\abr{
	\phi_{r,h}^k(x,a)^\top (\Lambda_{r,h}^k)^{-1} \sum_{\tau\,=\,1}^{k-1} \phi_{r,h}^\tau(x_h^\tau,a_h^\tau) \cdot\rbr{ V_{r,h+1}^\tau(x_{h+1}^\tau)-(\mathbb{P}_h V_{r,h+1}^\tau) \rbr{x_h^\tau,a_h^\tau}} }
\\[0.5cm]
&&\,+\, \abr{\lambda \cdot\phi_{r,h}^k(x,a)^\top (\Lambda_{r,h}^k)^{-1}  \theta_h}
\\[0.2cm]
&\!\!\leq\!\!&\displaystyle
\rbr{\phi_{r,h}^k(x,a)^\top (\Lambda_{r,h}^k)^{-1} \phi_{r,h}^k(x,a)}^{1/2} \norm{\sum_{\tau\,=\,1}^{k-1} \phi_{r,h}^\tau(x_h^\tau,a_h^\tau) \cdot\rbr{ V_{r,h+1}^\tau(x_{h+1}^\tau)-(\mathbb{P}_h V_{r,h+1}^\tau) \rbr{x_h^\tau,a_h^\tau}} }_{ (\Lambda_{r,h}^k)^{-1} }
\\[0.5cm]
&&\,+\, \lambda \rbr{\phi_{r,h}^k(x,a)^\top (\Lambda_{r,h}^k)^{-1}\phi_{r,h}^k(x,a)}^{1/2} \norm{\theta_h}_{ (\Lambda_{r,h}^k)^{-1} }
\end{array}
\]
for any $(k,h)\in[K]\times[H]$ and $(x,a)\in\calS\times\calA$, where we apply the Cauchy-Schwarz inequality twice in the inequality. By Lemma~\ref{lem.SNP}, set $\lambda=1$, with probability $1-p/2$ it holds that 
\[
\norm{\sum_{\tau\,=\,1}^{k-1} \phi_{r,h}^\tau(x_h^\tau,a_h^\tau) \cdot\rbr{ V_{r,h+1}^\tau(x_{h+1}^\tau)-(\mathbb{P}_h V_{r,h+1}^\tau) \rbr{x_h^\tau,a_h^\tau}} }_{ (\Lambda_{r,h}^k)^{-1} }
\;\leq\;
C\sqrt{dH^2 \log\rbr{\frac{dT}{p}}}.
\]
Also notice that $\Lambda_{r,h}^k \succeq\lambda I$ and $\norm{\theta_h}\leq \sqrt{d}$, thus $\norm{\theta_h}_{ (\Lambda_{r,h}^k)^{-1} }\leq\sqrt{\lambda d}$. Thus, by taking an appropriate absolute constant $C$, we obtain that  
\[
\abr{\phi_{r,h}^k(x,a)^\top w_{r,h}^k- (\mathbb{P}_h V_{r,h+1}^k) \rbr{x,a}}
\;\leq\;
C \rbr{\phi_{r,h}^k(x,a)^\top (\Lambda_{r,h}^k)^{-1}\phi_{r,h}^k(x,a)}^{1/2}\sqrt{dH^2 \log\rbr{\frac{dT}{p}}}
\]
for any $(k,h)\in[K]\times[H]$ and $(x,a)\in\calS\times\calA$ under the event of Lemma~\ref{lem.SNP}.

We now set $C>1$ and $\beta = C\sqrt{dH^2 \log\rbr{{dT}/{p}}}$. By the exploration bonus $\Gamma_{r,h}^k$ in line~7 of Algorithm~\ref{alg:LSVI}, with probability $1-p/2$ it holds that 
\begin{equation}\label{eq.phiP}
\abr{\phi_{r,h}^k(x,a)^\top w_{r,h}^k- (\mathbb{P}_h V_{r,h+1}^k) \rbr{x,a}} 
\;\leq\;
\Gamma_{r,h}^k(x,a)
\end{equation}
for any $(k,h)\in[K]\times[H]$ and $(x,a)\in\calS\times\calA$. 

%The first step is similar as the verification of~\eqref{eq.phiP} in Section~\ref{subsec.ucb-full}. By the exploration bonus $\Gamma_{r,h}^k$ in line~10 of Algorithm~\ref{alg:LSVI.bandit}, with probability $1-p/3$ it holds that 
%\begin{equation}\label{eq.ucb1}
%\abr{\phi_{r,h}^k(x,a)^\top w_{r,h}^k- (\mathbb{P}_h V_{r,h+1}^k) \rbr{x,a}} 
%\;\leq\;
%\Gamma_{r,h}^k(x,a)
%\end{equation}
%for any $(k,h)\in[K]\times[H]$ and $(x,a)\in\calS\times\calA$. 

We note that reward/utility functions are fixed over episodes, $r_h(x_h^\tau,a_h^\tau) \DefinedAs \varphi(x_h^\tau,a_h^\tau)^\top\theta_{r,h}$
% and $r_h^\tau(x_h^\tau,a_h^\tau) \DefinedAs \varphi(x_h^\tau,a_h^\tau)^\top\theta_{r,h}^k$. 
For the difference $\varphi(x,a)^\top u_{r,h}^k-r_h(x,a)$, we have
\[
\begin{array}{rcl}
&& \!\!\!\! \!\!\!\! \!\!\!\! \abr{\varphi(x,a)^\top u_{r,h}^k- r_h(x,a)} 
\\[0.2cm]
&=& \displaystyle\abr{ \varphi(x,a)^\top u_{r,h}^k- \varphi(x,a)^\top\theta_{r,h}}
\\[0.2cm]
&=& \displaystyle\abr{ \varphi(x,a)^\top (\Lambda_h^k)^{-1}\rbr{\sum_{\tau\,=\,1}^{k-1}\varphi(x_h^\tau,a_h^\tau) r_h(x_h^\tau,a_h^\tau)- \Lambda_h^k\,\theta_{r,h}}}
\\[0.5cm]
&=&\displaystyle\abr{ \varphi(x,a)^\top (\Lambda_h^k)^{-1}\rbr{\sum_{\tau\,=\,1}^{k-1}\varphi(x_h^\tau,a_h^\tau) \rbr{r_h(x_h^\tau,a_h^\tau)- \varphi(x_h^\tau,a_h^\tau)^\top\theta_{r,h}} + \lambda \theta_{r,h}}}
\\[0.5cm]
&=& \lambda\displaystyle\abr{ \varphi(x,a)^\top (\Lambda_h^k)^{-1} \theta_{r,h}}
\\[0.2cm]
&\leq& \lambda\displaystyle\rbr{ \varphi(x,a)^\top (\Lambda_h^k)^{-1} \varphi(x,a)}^{1/2} \norm{\theta_{r,h}}_{ (\Lambda_h^k)^{-1}}
\end{array}
\]
where we apply the Cauchy-Schwartz inequality in the inequality. Notice that $\Lambda_{h}^k \succeq\lambda I$ and $\Vert{\theta_{r,h}}\Vert\leq \sqrt{d}$, thus $\Vert{\theta_{r,h}}\Vert_{ (\Lambda_{h}^k)^{-1} }\leq\sqrt{\lambda d}$. Hence, if we set $\lambda=1$ and $\beta = C\sqrt{dH^2 \log\rbr{{dT}/{p}}}$, then any $(k,h)\in[K]\times[H]$ and $(x,a)\in\calS\times\calA$, 
\begin{equation}\label{eq.ucb2}
\abr{\varphi(x,a)^\top u_{r,h}^k \,-\, r_h(x,a)} \;\leq\; \Gamma_{h}^k(x,a).
\end{equation}

We recall the model prediction error $\iota_{r,h}^k \DefinedAs r_h + \mathbb{P}_h V_{r,h+1}^k - Q_{r,h}^k$ and the estimated state-action value function $Q_{r,h}^k$ in line~11 of Algorithm~\ref{alg:LSVI},
\[
Q_{r,h}^k(x,a) \;=\; \min\big(\,\varphi(x,a)^\top u_{r,h}^k+\phi_{r,h}^k(x,a)^\top w_{r,h}^k+ (\Gamma_{h}^k+\Gamma_{r,h}^k)(x,a),\, H-h+1\,\big)^+
\]
for any $(k,h)\in[K]\times[H]$ and $(x,a)\in\calS\times\calA$. By~\eqref{eq.phiP} and~\eqref{eq.ucb2}, we first have
\[
\phi_{r,h}^k(x,a)^\top w_{r,h}^k \,+\, \Gamma_{r,h}^k(x,a) \;\geq\; 0\; \text{ and } \;\varphi(x,a)^\top u_{r,h}^k \,+\, \Gamma_{h}^k(x,a) \;\geq\; 0.
\]
Then, we can show that
\begin{equation}\label{eq.iota}
\begin{array}{rcl}
&& \!\!\!\! \!\!\!\! \!\!-\, \iota_{r,h}^k(x,a) 
\\[0.2cm]
&=&  Q_{r,h}^k(x,a) \,-\, (r_h + \mathbb{P}_h V_{r,h+1}^k)(x,a)
\\[0.2cm]
&\leq & \varphi(x,a)^\top u_{r,h}^k \,+\, \phi_{r,h}^k(x,a)^\top w_{r,h}^k \,+\, (\Gamma_{h}^k+\Gamma_{r,h}^k)(x,a) \,-\, (r_h^k + \mathbb{P}_h V_{r,h+1}^k)(x,a)
\\[0.2cm]
&\leq & (\varphi(x,a)^\top u_{r,h}^k-r_h(x,a))+\Gamma_{h}^k(x,a)+2\Gamma_{r,h}^k(x,a)
\end{array}
\end{equation}
for any $(k,h)\in[K]\times[H]$ and $(x,a)\in\calS\times\calA$. 

Therefore,~\eqref{eq.iota} reduces to
\[
-\, \iota_{r,h}^k(x,a) 
\;\leq\; 
2\Gamma_{h}^k(x,a)+2\Gamma_{r,h}^k(x,a) \;=\; 2(\Gamma_{h}^k+\Gamma_{r,h}^k)(x,a).
\]

On the other hand, notice that $(r_h^k + \mathbb{P}_h V_{r,h+1}^k)(x,a) \leq H-h+1$, thus
\[
\begin{array}{rcl}
&& \!\!\!\! \!\!\!\! \!\!\!\! \iota_{r,h}^k(x,a) 
\\[0.2cm]
&=&  (r_h + \mathbb{P}_h V_{r,h+1}^k)(x,a) \,-\, Q_{r,h}^k(x,a) 
\\[0.2cm]
&\leq & (r_h + \mathbb{P}_h V_{r,h+1}^k)(x,a) \,-\, \min\big(\,\varphi(x,a)^\top u_{r,h}^k +\phi_{r,h}^k(x,a)^\top w_{r,h}^k+ (\Gamma_{h}^k+\Gamma_{r,h}^k)(x,a),\, H-h+1\,\big)^+
\\[0.2cm]
&\leq & \max\big(\, r_h(x,a) - \varphi(x,a)^\top u_{r,h}^k - \Gamma_{h}^k(x,a)+(\mathbb{P}_h V_{r,h+1}^k)(x,a) -\phi_{r,h}^k(x,a)^\top w_{r,h}^k- \Gamma_{r,h}^k(x,a),\,0\,\big)^+
\\[0.2cm]
&\leq & 0
\end{array}
\]
for any $(k,h)\in[K]\times[H]$ and $(x,a)\in\calS\times\calA$.

Therefore, we have proved that with probability $1-p/2$ it holds that 
\[
-2 (\Gamma_{h}^k+\Gamma_{r,h}^k)(x,a) \;\leq\;\iota_{r,h}^k(x,a) \;\leq\;0
\]
for any $(k,h)\in[K]\times[H]$ and $(x,a)\in\calS\times\calA$. 

Similarly, we can show another inequality $-2 (\Gamma_{h}^k+\Gamma_{g,h}^k)(x,a)\leq\iota_{g,h}^k(x,a)\leq 0$. 

\subsection{Proof of Formula~\eqref{eq.ucb-full-t}}\label{subsec.ucb-full-t}

Let $\calV = \{V: \calS\to[0,H]\}$ be a set of bounded function on $\calS$. Fo any $V\in\calV$, we consider the difference between $\sum_{x'\,\in\,\calS}\hat{\mathbb{P}}_h^k (x' \,|\,\cdot,\cdot) V(x')$ and $\sum_{x'\,\in\,\calS}{\mathbb{P}}_h (x' \,|\,\cdot,\cdot) V(x')$ as follows,
\begin{equation}\label{eq.PP}
\begin{array}{rcl}
&& \!\!\!\! \!\!\!\!\displaystyle\rbr{n_h^k(x,a)+\lambda}^{1/2} \abr{ \sum_{x'\,\in\,\calS}\rbr{\hat{\mathbb{P}}_h^k (x' \,|\,x,a) V(x')-{\mathbb{P}}_h (x' \,|\,x,a) V(x') }}
\\[0.5cm]
& \!\!=\!\! & \displaystyle\rbr{n_h^k(x,a)+\lambda}^{-1/2} \abr{ {\sum_{x'\,\in\,\calS} n_h^k(x,a,x') V(x')-(n_h^k(x,a)+\lambda) (\mathbb{P}_hV)(x,a) }}
\\[0.5cm]
& \!\!\leq\!\! & \displaystyle\rbr{n_h^k(x,a)+\lambda}^{-1/2} \abr{ {\sum_{x'\,\in\,\calS} n_h^k(x,a,x') V(x')-n_h^k(x,a) (\mathbb{P}_hV)(x,a) }}
\\[0.5cm]
&&\,+\,\displaystyle\rbr{n_h^k(x,a)+\lambda}^{-1/2} \abr{ \lambda(\mathbb{P}_hV)(x,a) }
\\[0.5cm]
& \!\!=\!\! & \displaystyle\rbr{n_h^k(x,a)+\lambda}^{-1/2} \abr{ {\sum_{\tau\,=\,1}^{k-1} \one\{(x,a) = (x_h^\tau,a_h^\tau)\}  \rbr{V(x_{h+1}^\tau)- (\mathbb{P}_hV)(x,a)} }}
\\[0.5cm]
&&\,+\,\displaystyle\rbr{n_h^k(x,a)+\lambda}^{-1/2} \abr{ \lambda(\mathbb{P}_hV)(x,a) }
\end{array}
\end{equation}
for any $(k,h)\in[K]\times[H]$ and $(x,a)\in\calS\times\calA$, where we apply the triangle inequality for the inequality.

Let $\eta_h^\tau \DefinedAs V(x_{h+1}^\tau)- (\mathbb{P}_hV)(x_h^\tau,a_h^\tau)$. Conditioning on the filtration $\calF_{h,1}^k$, $\eta_h^\tau $ is a zero-mean and $H/2$-subGaussian random variable. By Lemma~\ref{lem.CSNP}, we use $Y =\lambda I$ and $X_\tau = \one\{(x,a) = (x_h^\tau,a_h^\tau)\} $ and thus with probability at least $1-\delta$ it holds that
\[
\begin{array}{rcl}
&& \!\!\!\! \!\!\!\! \!\! \displaystyle\rbr{n_h^k(x,a)+\lambda}^{-1/2} \abr{ {\sum_{\tau\,=\,1}^{k-1} \one\{(x,a) = (x_h^\tau,a_h^\tau)\}  \rbr{V(x_{h+1}^\tau)- (\mathbb{P}_hV)(x,a)} }}
\\[0.5cm]
&\!\!\leq\!\!& \displaystyle \sqrt{ \frac{H^2}{2}  \log \rbr{ \frac{ \rbr{n_h^k(x,a)+\lambda}^{1/2} \lambda^{-1/2}}{\delta/H}}}
\\[0.5cm]
&\!\!\leq\!\!& \displaystyle \sqrt{ \frac{H^2}{2} \log\rbr{\frac{T}{\delta}}}
\end{array}
\]
for any $(k,h)\in[K]\times[H]$. Also, since $0\leq V\leq H$, we have 
\[
\rbr{n_h^k(x,a)+\lambda}^{-1/2} \abr{ \lambda(\mathbb{P}_hV)(x,a) }
\;\leq\;
\sqrt{\lambda} H.
\]
By returning to~\eqref{eq.PP} and setting $\lambda=1$, with probability at least $1-\delta$ it holds that
\begin{equation}\label{eq.PPV}
\displaystyle\rbr{n_h^k(x,a)+\lambda} \abr{ \sum_{x'\,\in\,\calS}\rbr{\hat{\mathbb{P}}_h^k (x' \,|\,x,a) V(x')-{\mathbb{P}}_h (x' \,|\,x,a) V(x') }}^2
\;
\leq
\;
H^2 \rbr{\log\rbr{\frac{T}{\delta}}+2}
\end{equation}
for any $k\geq 1$.

Let $d(V,V') = \max_{x\,\in\,\calS} \abr{V(x)-V'(x)}$ be a distance on $\calV$. For any $\epsilon$, an $\epsilon$-covering $\calV_\epsilon$ of $\calV$ with respect to distance $d(\cdot,\cdot)$ satisfies 
\[
|\calV_\epsilon| \;\leq\; \rbr{1+\frac{2\sqrt{|\calS|}H}{\epsilon} }^{|\calS|}.
\]
Thus, for any $V\in\calV$, there exists $V'\in\calV_\epsilon$ such that $\max_{x\,\in\,\calS}|V(x)-V'(x)|\leq \epsilon$. By the triangle inequality, we have
\[
\begin{array}{rcl}
&& \!\!\!\! \!\!\!\!\displaystyle\rbr{n_h^k(x,a)+\lambda}^{1/2} \abr{ \sum_{x'\,\in\,\calS}\rbr{\hat{\mathbb{P}}_h^k (x' \,|\,x,a) V(x')-{\mathbb{P}}_h (x' \,|\,x,a) V(x') }}
\\[0.5cm]
& \!\!=\!\! & \displaystyle\rbr{n_h^k(x,a)+\lambda}^{1/2} \abr{ \sum_{x'\,\in\,\calS}\rbr{\hat{\mathbb{P}}_h^k (x' \,|\,x,a) V'(x')-{\mathbb{P}}_h (x' \,|\,x,a) V'(x') }}
\\[0.5cm]
&&\,+\,\displaystyle\rbr{n_h^k(x,a)+\lambda}^{1/2} \abr{ \sum_{x'\,\in\,\calS}\rbr{\hat{\mathbb{P}}_h^k (x' \,|\,x,a) (V(x') - V'(x'))-{\mathbb{P}}_h (x' \,|\,x,a) (V(x')-V'(x')) }}
\\[0.5cm]
& \!\!\leq\!\! & \displaystyle\rbr{n_h^k(x,a)+\lambda}^{1/2} \abr{ \sum_{x'\,\in\,\calS}\rbr{\hat{\mathbb{P}}_h^k (x' \,|\,x,a) V'(x')-{\mathbb{P}}_h (x' \,|\,x,a) V'(x') }}
\\[0.5cm]
&&\,+\,2\displaystyle\rbr{n_h^k(x,a)+\lambda}^{-1/2} \epsilon.
\end{array}
\]
Furthermore, we choose $\delta = \rbr{{p}/{3}}/\rbr{|\calV_\epsilon||\calS||\calA|}$ and take an union bound over $V\in\calV_\epsilon$ and $(x,a)\in\calS\times\calA$. By~\eqref{eq.PPV}, with probability at least $1-p/2$ it holds that
\[
\begin{array}{rcl}
&& \!\!\!\! \!\!\!\!\displaystyle \sup_{V\,\in\,\calV}\cbr{\rbr{n_h^k(x,a)+\lambda}^{1/2} \abr{ \sum_{x'\,\in\,\calS}\rbr{\hat{\mathbb{P}}_h^k (x' \,|\,x,a) V(x')-{\mathbb{P}}_h (x' \,|\,x,a) V(x') }}}
\\[0.5cm]
& \!\!\leq\!\! & \displaystyle\sqrt{H^2\rbr{\log\rbr{\frac{T}{\delta}}+2}}
\,+\,2\displaystyle\rbr{n_h^k(x,a)+\lambda}^{-1/2} \frac{H}{K}
\\[0.5cm]
& \!\!\leq\!\! & \displaystyle\sqrt{2H^2 \rbr{ \log|\calV_\epsilon|+\log\rbr{\frac{2|\calS||\calA|T}{p}}+2}}
\,+\,2\displaystyle\rbr{n_h^k(x,a)+\lambda}^{-1/2} \frac{H}{K}
\\[0.5cm]
& \!\!\leq\!\! & \displaystyle C_1 H\sqrt{|\calS| \log \rbr{\frac{|\calS||\calA|T}{p}}} \;\DefinedAs\;\beta
\end{array}
\]
for all $(k,h)$ and $(x,a)$, where $C_1$ is an absolute constant.
We recall our choice of $\Gamma_h^k$ and $\beta$. Hence, with probability at least $1-p/2$ it holds that 
\[
\abr{ \sum_{x'\,\in\,\calS}\rbr{\hat{\mathbb{P}}_h^k (x' \,|\,x,a) V(x')-{\mathbb{P}}_h (x' \,|\,x,a) V(x') }}
\;\leq\;
\beta
\rbr{n_h^k(x,a)+\lambda}^{-1/2}
\;\DefinedAs\;
\Gamma_h^k(x,a)
\]
for any $(k,h)\in[K]\times[H]$ and $(x,a)\in |\calS|\times|\calA|$, where $\beta \DefinedAs C_1 H\sqrt{|\calS| \log(|\calS||\calA|T/p)}$.

We recall the definition $r_h(x,a) = \mathbf{e}_{(x,a)}^\top \theta_{r,h}$.
By our estimation $\hat{r}_h^k(x,a)$ in Algorithm~\ref{alg:tbandit}, we have
\[
\hat{r}_h^k(x,a) \;=\; \frac{1}{n_h^k(x,a) +\lambda} \sum_{\tau\,=\,1}^{k-1} \one\{ (x,a) = (x_h^\tau,a_h^\tau) \} [\theta_{r,h}]_{(x_h^\tau,a_h^\tau)}
\]
and thus 
\[
\begin{array}{rcl}
&& \!\!\!\! \!\!\!\! \!\! \abr{\hat{r}_h^k(x,a) - {r}_h(x,a)} 
\\[0.2cm]
& \!\! = \!\! &\abr{\hat{r}_h^k(x,a) -  [\theta_{r,h} ]_{(x,a)}} 
\\[0.2cm]
& \!\! = \!\! &\displaystyle \rbr{n_h^k(x,a) +\lambda}^{-1}\abr{  \sum_{\tau\,=\,1}^{k-1} \one\{ (x,a) = (x_h^\tau,a_h^\tau) \} \rbr{[\theta_{r,h}]_{(x_h^\tau,a_h^\tau)} -  [\theta_{r,h} ]_{(x,a)}} -\lambda  [\theta_{r,h} ]_{(x,a)}}
\\[0.2cm]
& \!\! = \!\! & \displaystyle\rbr{n_h^k(x,a) +\lambda}^{-1}\abr{ \lambda  [\theta_{r,h} ]_{(x,a)}}
\\[0.2cm]
& \!\! \leq \!\! & \displaystyle\rbr{n_h^k(x,a) +\lambda}^{-1} \lambda
\\[0.2cm]
& \!\! \leq \!\! & \displaystyle\rbr{n_h^k(x,a) +\lambda}^{-1/2} \lambda
\\[0.2cm]
& \!\! \leq \!\! & \Gamma_{h}^k(x,a)
\end{array}
\]
where we utilize $\lambda=1$ and $\beta\geq1$ in the inequalities.

We now are ready to check the model prediction error $\iota_{r,h}^k$ defined by~\eqref{eq.mper},
\[
\begin{array}{rcl}
&& \!\!\!\! \!\!\!\! \!\! -\, \iota_{r,h}^k(x,a) 
\\[0.2cm]
&=&  Q_{r,h}^k(x,a) \,-\, (r_h + \mathbb{P}_h V_{r,h+1}^k)(x,a)
\\[0.2cm]
&\leq &  \hat r_h^k(x,a) \,+\, \sum_{x'\,\in\,\calS}\hat{\mathbb{P}}_h^k (x' \,|\,x,a) V_{r,h+1}^k(x') \,+\, 2\Gamma_{h}^k(x,a) \,-\, (r_h + \mathbb{P}_h V_{r,h+1}^k)(x,a)
\\[0.2cm]
&\leq & 4\Gamma_{h}^k(x,a)
\end{array}
\]
for any $(x,a)\in\calS\times\calA$. On the other hand, notice that $(r_h + \mathbb{P}_h V_{r,h+1}^k)(x,a) \leq H-h+1$, thus
\[
\begin{array}{rcl}
&& \!\!\!\! \!\!\!\! \!\!\!\! \iota_{r,h}^k(x,a) 
\\[0.2cm]
&=&  (r_h + \mathbb{P}_h V_{r,h+1}^k)(x,a) \,-\, Q_{r,h}^k(x,a) 
\\[0.2cm]
&\leq & (r_h + \mathbb{P}_h V_{r,h+1}^k)(x,a) \,-\, \min\big(\,\hat r_h^k(x,a) +\sum_{x'\,\in\,\calS}\hat{\mathbb{P}}_h^k (x' \,|\,x,a) V_{r,h+1}^k(x')+2\Gamma_{h}^k(x,a),\, H-h+1\,\big)^+
\\[0.2cm]
&\leq & \max\big(\, (r_h-\hat r_h)(x,a) - \Gamma_h^k(x,a)+ (\mathbb{P}_h V_{r,h+1}^k)(x,a) -\sum_{x'\,\in\,\calS}\hat{\mathbb{P}}_h^k (x' \,|\,x,a) V_{r,h+1}^k(x')- \Gamma_{h}^k(x,a),\, 0\,\big)^+
\\[0.2cm]
&\leq & 0
\end{array}
\]
for any $(k,h)\in[K]\times[H]$ and $(x,a)\in\calS\times\calA$. Hence, we complete the proof of~\eqref{eq.ucb-full-t}.

\section{Supporting Lemmas from Optimization}\label{app.sec.opt}

We collect some standard results from the literature for readers' convenience. We rephrase them for our constrained problem~\eqref{eq.hindsight},
\[
\begin{array}{c}
\!\!\!\!
\maximize\limits_{\pi \, \in \, \Delta(\calA\,\vert\, \calS, H)}
\;
V_{r,1}^{\pi}(x_1)
%\\[0.3cm]
\;\;
\subject 
\;\;
V_{g,1}^{\pi}(x_1) \;\geq\; b
\end{array}
\]
in which we maximize over all policies and $b\in (0,H]$. Let the optimal solution be $\pi^\star$ such that
\[
V_{r,1}^{\pi^\star}(x_1)
\;=\;
\maximize_{\pi \, \in \, \Delta(\calA\,\vert\, \calS, H)}\;\{ \,V_{r,1}^{\pi}(x_1) \,\vert\,V_{g,1}^{\pi}(x_1) \,\geq\, b\, \}.
\]
Let the Lagrangian be
$\mathcal{L}(\pi,Y) \DefinedAs V_{r,1}^{\pi}(x_1)+ Y(V_{g,1}^{\pi}(x_1)-b)$, where $Y\geq 0$ is the Lagrange multiplier or dual variable. The associated dual function is defined as 
\[
\mathcal{D}(Y) \;\DefinedAs\; \maximize_{\pi \, \in \, \Delta(\calA\,\vert\, \calS, H)} \; \mathcal{L}(\pi,Y) \DefinedAs V_{r,1}^{\pi}(x_1)\,+\, Y(V_{g,1}^{\pi}(x_1)-b)
\]
and the optimal dual is $Y^\star\DefinedAs\argmin_{Y\,\geq\,0}\mathcal{D}(Y)$,
\[
\mathcal{D}(Y^\star) \;\DefinedAs\; \minimize_{\lambda\,\geq\, 0} \; \mathcal{D}(Y)
\]

We recall that the problem~\eqref{eq.hindsight} enjoys strong duality under the standard Slater condition. The proof is a special case of~\cite[Proposition~1]{paternain2019safe} in finite-horizon.

\begin{assumption}[Slater Condition]
	There exists $\gamma>0$ and $\bar{\pi}$ such that $V_{g,1}^{\bar{\pi}}(x_1) -b \geq \gamma$.
\end{assumption}

%We use the shorthand notation $V_r^{\pi^\star}(\rho) = V_r^{\star}(\rho)$ and $V_D^{\lambda^\star} (\rho) = V_D^{\star} (\rho)$ whenever it is clear from the context. 

\begin{lemma}[Strong Duality]\cite[Proposition~1]{paternain2019safe}
	If the Slater condition holds, then the strong duality holds, 
	\[
	V_{r,1}^{\pi^\star}(x_1) \;= \; \mathcal{D}(Y^\star) .
	\]
\end{lemma}

It is implied by the strong duality that the optimal solution to the dual problem: $\minimize_{Y\,\geq\, 0} \; \mathcal{D}(Y)$ is obtained at $Y^\star$. Denote the set of all optimal dual variables as $\Lambda^\star$.

Under the Slater condition, an useful property of the dual variable is that the sublevel sets are bounded~\cite[Section~8.5]{beck2017first}. 
%It is standard in the convex optimization~\cite[Section~8.5]{beck2017first}. 
%Although our problem is nonconcave, we customize it as follows.

\begin{lemma}[Boundedness of Sublevel Sets of the Dual Function]
	Let the Slater condition hold.
	Fix $C\in\mathbb{R}$. For any $Y \in\{ Y\geq 0\,\vert\, \mathcal{D}(Y) \leq C \}$, it holds that 
	\[
	Y \;\leq\; \frac{1 }{\gamma} \rbr{C -V_{r,1}^{  \bar{\pi}} (x_1)}.
	\]
\end{lemma}
\begin{proof}
	By $Y \in\{ Y\geq 0\,\vert\, \mathcal{D}(Y)  \leq C \}$,
	\[
	C \;\geq\; \mathcal{D}(Y)  \;\geq\; V_{r,1}^{\bar{\pi}}(x_1) + Y\, (V_{g,1}^{\bar{\pi}}(x_1)-b) \;\geq\; V_{r}^{\bar{\pi}}(\rho) + Y\,\gamma
	\]
	where we utilize the Slater point $\bar{\pi}$ in the last inequality. We complete the proof by noting $\gamma>0$.
\end{proof}
\begin{cor}[Boundedness of $Y^\star$]
	\label{cor.boundeddual}
	If we take $C = V_{r,1}^{\pi^\star}(x_1) = \mathcal{D}(Y^\star)$, then $\Lambda^\star=\{ Y\geq 0\,\vert\, \mathcal{D}(Y)\leq C \}$. Thus, for any $Y\in\Lambda^\star$,
	\[
	Y \;\leq\; \frac{1 }{\gamma} \rbr{V_{r,1}^{\pi^\star}(x_1)-V_{r,1}^{  \bar{\pi}} (x_1)}.
	\]
\end{cor}

Another useful theorem from the optimization~\cite[Section~3.5]{beck2017first} is given as follows. It describes that the constraint violation $b-V_{g,1}^\pi(x_1)$ can be bounded similarly even if we have some weak bound. We next state and prove it for our problem, which is used in our constraint violation analysis in Section~\ref{ap.main1}.

\begin{lemma}[Constraint Violation]
	\label{thm.violationgeneral}
	Let the Slater condition hold and $Y^\star \in \Lambda^\star$. Let $C^\star \geq 2 Y^\star$. Assume that ${\tilde\pi\in\Delta(\calA\,\vert\, \calS, H)}$ satisfies 
	\[
	V_{r,1}^{\pi^\star}(x_1) \,-\, V_{r,1}^{\tilde\pi}(x_1) \,+\, C^\star \, \sbr{b - V_{g,1}^{\tilde\pi}(x_1)}_+ \;\leq \; \delta.
	\]
	
	Then, 
	\[
	\sbr{b - V_{g,1}^{\tilde\pi}(x_1)}_+ \;\leq \; \frac{2\delta}{C^\star}
	\]
	where $[x]_+=\max(x,0)$. 
\end{lemma}
\begin{proof}
	Let 
	\[
	v(\tau)
	\;=\;
	\maximize_{\pi\,\in\,\Delta(\calA\,\vert\, \calS, H)}\;\{ \,V_{r,1}^{\pi}(x_1) \,\vert\,V_{g,1}^{\pi}(x_1) \,\geq\, b + \tau\, \}.
	\]
	By definition, $v(0) = V_{r,1}^{\pi^\star}(x_1)$. 
	It has been shown as a special case of~\cite[Proposition~1]{paternain2019safe} that $v(\tau)$ is concave. First, we show that $-Y^\star\in\partial v(0)$. By the Lagrangian and the strong duality,
	\[
	\mathcal{L}(\pi,Y^\star) \;\leq\; \maximize_{\pi\,\in\,\Delta(\calA\,\vert\, \calS, H)}\; \mathcal{L}(\pi,Y^\star)  \;=\; \mathcal{D}(Y^\star) \;=\;V_{r,1}^{\pi^\star}(x_1) \;=\;v(0),\; \text{ for all } {\pi\,\in\,\Delta(\calA\,\vert\, \calS, H)}.
	\]
	For any $\pi\in\{ {\pi\in\Delta(\calA\,\vert\, \calS, H)} \,\vert\,V_{g,1}^{\pi}(x_1) \geq b + \tau \}$,
	\[
	\begin{array}{rcl}
	v(0) - \tau Y^\star &\geq& \mathcal{L}(\pi,Y^\star) - \tau Y^\star
	\\[0.2cm]
	&=& V_{r,1}^{\pi}(x_1) +Y^\star (V_{g,1}^\pi(x_1)-b) - \tau Y^\star
	\\[0.2cm]
	&=& V_{r,1}^{\pi}(x_1) +Y^\star (V_{g,1}^\pi(x_1)-b-\tau) 
	\\[0.2cm]
	&\geq& V_{r,1}^{\pi}(x_1).
	\end{array}
	\]
	If we maximize the right-hand side of above inequality over $\pi\in\{ {\pi\in\Delta(\calA\,\vert\, \calS, H)} \,\vert\,V_{g,1}^{\pi}(x_1) \geq b + \tau \}$, then
	\[
	v(0) -\tau Y^\star \;\geq\; v(\tau)
	\]
	which show that $-Y^\star\in\partial v(0)$. 
	%	It implies that for any $\tau$, 
	%	\[
	%	v(\tau) - v(0)\;\leq\;-{\tau}{\lambda^\star}.
	%	\]
	On the other hand, if we take $\tau = \tilde{\tau}\DefinedAs -(b - V_{g,1}^{\tilde{\pi}}(x_1))_+$, then
	\[
	V_{r,1}^{\tilde{\pi}}(x_1)\;\leq\;V_{r,1}^{\pi^\star}(x_1)\;=\;v(0) \;\leq\; v(\tilde{\tau}).
	\]
	Combing the above two yields 
	\[
	V_{r,1}^{\tilde{\pi}}(x_1)-V_{r,1}^\star(x_1)  \;\leq\;-\tilde{\tau}{Y^\star}.
	\]
	Thus,
	\[
	\begin{array}{rcl}
	\rbr{C^\star - Y^\star} \abr{\tilde{\tau}} &=&  - Y^\star\abr{\tilde{\tau}} + C^\star\abr{\tilde{\tau}}
	\\[0.2cm]
	&=&  \tilde{\tau} Y^\star + C^\star\abr{\tilde{\tau}}
	%	\\[0.2cm]
	%	&\leq&   V_r^\star(\rho) -V_r^{\tilde{\pi}}(\rho)+ C_{\lambda^\star} \abr{\tilde{\tau}}
	\\[0.2cm]
	&\leq&   V_{r,1}^{\pi^\star}(x_1) -V_{r,1}^{\tilde{\pi}}(x_1)+ C^\star \abr{\tilde{\tau}}.
	\end{array}
	\]
	
	By our assumption and $\tilde{\tau}= \sbr{b - V_g^{\tilde{\pi}}(\rho)}_+$, 
	\[
	\sbr{b - V_{g,1}^{\tilde{\pi}}(x_1)}_+ \;\leq\; \frac{\delta}{C^\star-Y^\star}
	\;\leq\;\frac{2\delta}{C^\star}.
	\]
\end{proof}

\section{Other Supporting Lemmas}\label{app.sec.support}

First, we state a lemma that is used throughout this paper. 
\begin{lemma}[Performance Difference Lemma]\label{lem.PDL}
	For any two policies $\pi,\pi'\in\Delta(\calA\,\vert\,\calS,H)$, it holds that 
	\[
	V_{\diamond,1}^{\pi'}(x_1^k) \,-\,V_{\diamond,1}^{\pi}(x_1^k) 
	\;=\;
	\mathbb{E}_{\pi'} \sbr{\sum_{h\,=\,1}^H \big\langle{Q_{\diamond,h}^{\pi} (x_h,\cdot)},{\pi_h'(\cdot\,\vert\,x_h) - \pi_h(\cdot\,\vert\,x_h)}\big\rangle\,\big\vert\,x_1  = x_1^k}
	\]
	where $\diamond = r$ or $g$. 
\end{lemma}
\begin{proof}
	See the proof of Lemma~3.2~in~\cite{cai2019provably}.
\end{proof}

Next, we state an useful concentration inequality for the standard self-normalized processes.
\begin{lemma}[Concentration of Self-normalized Processes]\label{lem.CSNP}
	Let $\{\calF_t\}_{t\,=\,0}^\infty$ be a filtration and $\{\eta_t\}_{t\,=\,1}^\infty$ be a $\mathbb{R}$-valued stochastic process such that $\eta_t$ is $\calF_t$-measurable for any $t\geq 0$. Assume that for any $t\geq 0$, conditioning on $\calF_t$, $\eta_t$ is a zero-mean and $\sigma$-subGaussian random variable with the variance proxy $\sigma^2>0$, i.e., $\mathbb{E}\sbr{e^{\lambda \eta_t} \,|\,\calF_t}\leq e^{\lambda^2\sigma^2/2}$ for any $\lambda\in\mathbb{R}$. Let $\{X_t\}_{t\,=\,1}^\infty$ be an $\mathbb{R}^d$-valued stochastic process such that $X_t$ is $\calF_t$-measurable for any $t\geq 0$. Let $Y\in\mathbb{R}^{d\times d}$ be a deterministic and positive-definite matrix. For any $t\geq 0$, we define 
	\[
	\bar Y_t \;\DefinedAs\; Y \,+\, \sum_{\tau\,=\,1}^{t} X_\tau X_\tau^\top \;\text{ and }\; S_t \;=\; \sum_{\tau\,=\,1}^t \eta_\tau X_\tau.
	\]
	Then, for any fixed $\delta\in(0,1)$, it holds with probability at least $1-\delta$ that 
	\[
	\norm{S_t}_{(\bar Y_t)^{-1}}^2 \;\leq\; 2\sigma^2 \log \rbr{ \dfrac{\det\rbr{ \bar{Y}_t}^{1/2} \det\rbr{Y}^{-1/2} }{\delta} }
	\]
	for any $t\geq 0$.
\end{lemma}
\begin{proof}
	See the proof of Theorem~1~in~\cite{abbasi2011improved}.
\end{proof}	

The above concentration inequality can be customized to our setting in the following form without using covering number arguments as in~\cite{jin2019provably}. 
\begin{lemma}\label{lem.SNP}
	Let $\lambda=1$ in Algorithm~\ref{alg:LSVI}. Fix $\delta\in(0,1)$. Then, for any $(k,h)\in[K]\times[H]$ it holds for $\diamond = r$ or $g$ that 
	\[
	\norm{ \sum_{\tau\,=\,1}^{k-1} \phi_{\diamond,h}^\tau(x_h^\tau,a_h^\tau)^\top \rbr{ V_{\diamond,h+1}^k (x_{h+1}^\tau) - (\mathbb{P}_hV_{\diamond,h+1}^k )(x_h^\tau,a_h^\tau)} }_{(\Lambda_{\diamond,h}^k)^{-1}}
	\;
	\leq
	\;
	C \sqrt{dH^2\log\rbr{\frac{dT}{\delta}}}
	\] 
	with probability at least $1-{\delta}/{2}$ where $C>0$ is an absolute constant.
\end{lemma}
\begin{proof}
	See the proof of Lemma~D.1~in~\cite{cai2019provably}.
\end{proof}

%\begin{lemma}[Concentration of Self-normalized Process]\label{lem.SNP}
%	Let $\lambda=1$ and $\beta = C_\beta \sqrt{dH^2\log(dT/p)}$ in Algorithm~\ref{alg:LSVI.full} or Algorithm~\ref{alg:LSVI.bandit} where $C_\beta$ is an absolute constant. Fix $p\in(0,1)$. Then, for any $(k,h)\in[K]\times[H]$ it holds for $\diamond = r$ or $g$ that 
%	\[
%	\norm{ \sum_{\tau\,=\,1}^{k-1} \phi_{\diamond,h}^\tau(x_h^\tau,a_h^\tau)^\top \rbr{ V_{\diamond,h+1}^k (x_{h+1}^\tau) - (\mathbb{P}_hV_{\diamond,h+1}^k )(x_h^\tau,a_h^\tau)} }_{(\Lambda_h^k)^{-1}}
%	\;
%	\leq
%	\;
%	C \sqrt{dH^2\log\rbr{\frac{dT}{p}}}
%	\] 
%	with probability at least $1-{p}/{3}$ where $C>0$ is an absolute constant.
%\end{lemma}
%\begin{proof}
%	See the proof of Lemma~C.4~in~\cite{jin2019provably}.
%\end{proof}

\begin{lemma}[Elliptical Potential Lemma]
	\label{lem.bdsums}
	Let $\{\phi_t\}_{t=1}^\infty$ be a sequence of functions in $\mathbb{R}^d$ and $\Lambda_0\in\mathbb{R}^{d\times d}$ be a positive definite matrix. Let $\Lambda_t = \Lambda_0+\sum_{i\,=\,1}^{t-1}\phi_i \phi_i^\top$. Assume $\norm{\phi_t}_2\leq 1$ and $\lambda_{\normalfont\text{min}} \rbr{\Lambda_0}\geq 1$. Then for any $t\geq 1$ it holds that
	\[
	\log\rbr{\frac{\det\rbr{\Lambda_{t+1}}}{\det\rbr{\Lambda_{1}}}} 
	\;\leq\;
	\sum_{i\,=\,1}^t \phi_i^\top \Lambda_i^{-1}\phi_i
	\;\leq\;
	2\log\rbr{\frac{\det\rbr{\Lambda_{t+1}}}{\det\rbr{\Lambda_{1}}}}.
	\]
\end{lemma}
\begin{proof}
	See the proof of Lemma~D.2~in~\cite{jin2019provably} or~\cite{cai2019provably}.
\end{proof}

\begin{lemma}[Pushback Property of KL-divergence]
	\label{lem.pushback}
	Let $f:\Delta\to\mathbb{R}$ be a concave function where $\Delta$ is a probability simplex in $\mathbb{R}^d$. Let $\Delta^o$ be the interior of $\Delta$. Let $x^\star = \argmax_{x\,\in\,\Delta} f(x)-\alpha^{-1} D(x,y)$ for a fixed $y\in\Delta^o$ and $\alpha>0$. Then, for any $z\in\Delta$,
	\[
	f(x^\star) 
	\,-\,
	\frac{1}{\alpha} D(x^\star,y) 
	\;\geq\;
	f(z) 
	\,-\,
	\frac{1}{\alpha} D(z,y) 
	\,+\,
	\frac{1}{\alpha} D(z,x^\star).
	\]
\end{lemma}
\begin{proof}
	See the proof of Lemma~14~in~\cite{wei2019online}.
\end{proof}

\begin{lemma}[Bounded KL-divergence Difference]
	\label{lem.mix}
	Let $\pi_1,\pi_2$ be two probability distributions in $\Delta(\calA)$. Let $\tilde{\pi}_2 = (1-\theta)\pi_2+\one{\theta}/{|\calA|}$ where $\theta\in(0,1]$. Then, 
	\[
	D\rbr{\pi_1\,\vert\, \tilde{\pi}_2} \,-\, D\rbr{\pi_1\,\vert\, \pi_2} 
	\;\leq\;
	\theta \log |\calA|.
	\]  
	Moreover, we have an uniform bound, $D\rbr{\pi_1\,\vert\, \tilde{\pi}_2}\leq \log({|\calA|}/{\theta})$.
\end{lemma}
\begin{proof}
	See the proof of Lemma~31~in~\cite{wei2019online}.
\end{proof}
\end{document}